\PassOptionsToPackage{numbers,sort&compress}{natbib}
\documentclass{article}

\usepackage[main, final]{neurips_2026}

\usepackage[utf8]{inputenc} 
\usepackage[T1]{fontenc}    
\usepackage{hyperref}       
\usepackage{url}            
\usepackage{booktabs}       
\usepackage{amsfonts}       
\usepackage{nicefrac}       
\usepackage{microtype}      
\usepackage{xcolor}         
\usepackage{amsmath}
\usepackage{amsthm}
\usepackage{graphicx}
\usepackage{graphicx}
\usepackage{subcaption}
\usepackage{amsthm}
\usepackage{thmtools}
\usepackage{thm-restate}
\usepackage{algorithm}
\usepackage{algpseudocode}
\usepackage{titletoc}
\declaretheorem[name=Proposition]{proposition}

\title{Specificity-Aware Diffusion Steering via Variance-Reduced Sequential Monte Carlo}

\author{%
  Luran Wang\thanks{Correspondence to \texttt{luran@mit.edu}}, Linrui Ma, Hannes St{\"a}rk, Regina Barzilay \\
  CSAIL, Massachusetts Institute of Technology\\
}

\begin{document}

\maketitle
\begin{abstract}
Inference-time steering enables pretrained diffusion models to satisfy new constraints without full retraining. However, specificity-aware generation is difficult: repelling samples from a negative reference distribution can also erode the positive distribution where the two overlap. The key challenge is to suppress negative mass while minimally distorting the positive distribution.
We address this problem by formulating specificity-aware steering as a target-design problem and deriving a target distribution from an overlap-based objective. The resulting target keeps the desired reference distribution only in regions where it is sufficiently preferred over the undesired reference distribution, giving a likelihood-ratio interpretation of specificity. To sample from the corresponding time-dependent target path, we develop a Sequential Monte Carlo sampler with a variance-minimized local proposal. We further introduce a practical fixed-noise optimization procedure with the Jacobian--vector products with the desired and undesired score fields. Experiments on synthetic task, class-contrastive generation, text-to-image tasks and peptide-MHC (p-MHC) binder show that the proposed method suppresses undesired regions more effectively, reduces mode shift, and improves sampling stability by decreasing the SMC weight collapse compared with negative-guidance baselines. Code is available at: \url{https://github.com/WangLuran/Specificity-Aware-Diffusion-Steering}
\end{abstract}
\section{Introduction}

Diffusion and score-based models have become standard tools for
high-dimensional generation
\citep{ho2020denoising,song2021score}. One important advantage of these
models is that their sampling process can be steered at inference time,
without retraining the underlying models \citep{uehara2025inference}\citep{liu2023more}. This makes them suitable
for settings where generation must satisfy both positive and negative
constraints: samples should remain close to a desired distribution
while avoiding regions associated with an undesired reference distribution.\citep{kim2025comprehensive}

Existing steering methods, including classifier guidance, classifier-free
guidance, negative prompting, and safe-generation guidance, modify the
denoising direction or score estimate to promote or suppress selected
attributes
\citep{dhariwal2021diffusion}\citep{ho2022classifierfree,schramowski2023safe}\citep{
armandpour2023perpneg}\citep{koulischer2024dynamic}. However, these methods usually
define a steering rule rather than an explicit target distribution. As a
result, when the desired and undesired densities overlap, repulsive guidance
can suppress valid modes and induce mode shift.\citep{armandpour2023reimagine}\citep{dinh2023rethinking}\citep{lee2024contrastivecfg}\citep{chung2025cfgpp}

To address this issue, we formulate specificity-aware diffusion steering. We
introduce a specificity-aware target distribution that preserves the desired
density only in regions where it is sufficiently preferred over the undesired
density.  To sample from the resulting time-dependent target path, we use Sequential
Monte Carlo (SMC). SMC provides a distributionally principled mechanism for
diffusion steering through propagation and reweighting
\citep{doucet2001sequential}\citep{delmoral2006sequential}\citep{skreta2025feynman}.
However, its stability depends on the variance of incremental weights\citep{naesseth2019elements}. We
therefore derive a local proposal rule using forward--backward Gaussian
kernel ratios, following the density-ratio view of guided diffusion
\citep{he2025rne}. The backward drift is chosen to cancel the leading fluctuation in the log-weight, and the remaining leading-order terms define
a tractable variance objective for adapting the proposal.

Our contributions are threefold. First, we formulate negative guidance as
sampling from a specificity-aware target distribution derived from an
overlap-based objective, yielding a principled likelihood-ratio gate that
preserves positive mass while suppressing negative overlap. Second, we develop
a proposal-aware SMC sampler with a gate-augmented forward kernel and locally
optimized backward kernel to reduce diffusion-path weight degeneracy. Third,
we introduce a fixed-noise optimization procedure using retained-order 
weight surrogates and Jacobian--vector products to optimize ESS
without retraining the base model. We validate the method on synthetic dataset, class-contrastive generation, and Stable Diffusion text-to-image
benchmarks with both related and unrelated negative prompts as well as the peptide-MHC (p-MHC) binder design. Across these
settings, our method reduces mode shift, improves sampling stability, and
better removes negative concepts than the baselines while preserving positive-prompt quality.

\section{Related Work}

\paragraph{Diffusion guidance and negative guidance.}
Diffusion and score-based generative models can be steered at inference time by
modifying their reverse-time score, denoising direction, or velocity field
without retraining the base model
\citep{sohl2015deep,ho2020denoising,song2021score}.
Such guidance can use classifiers, rewards, or compositions of learned
score and flow fields
\citep{dhariwal2021diffusion,ho2022classifierfree,liu2022compositional,
lipman2022flow,song2020denoising,karras2022elucidating}.
Negative guidance further suppresses undesired concepts during sampling
\citep{schramowski2023safe,armandpour2023perpneg,kim2025training,
koulischer2024dynamic}.
However, these methods typically define the resulting distribution only
implicitly, and repulsive updates can remove valid modes when desired and
undesired concepts overlap.
In contrast, we define an explicit specificity-aware target that preserves
probability mass favored by the desired reference while penalizing regions
more characteristic of the undesired reference, and derive a sampler for its
diffusion path.

\paragraph{Feynman--Kac correction and density-ratio steering.}
Feynman--Kac and Sequential Monte Carlo (SMC) methods approximate target
distributions through propagation, reweighting, and resampling
\citep{doucet2001sequential,delmoral2006sequential}.
Recent work applies these ideas to diffusion guidance and model composition
\citep{skreta2025feynman}, as well as test-time sampling from
reward-modified targets \citep{kim2025testtime}.
Density-ratio approaches instead characterize diffusion control through local
transition-kernel ratios, avoiding explicit estimation of higher-order PDE
quantities \citep{he2025rne}.
We adopt this forward--backward kernel-ratio view for proposal design and use
local kernel ratios to reduce the variance of incremental SMC weights.

\paragraph{SMC stability and proposal optimization.}
A central difficulty in SMC is weight degeneracy, which reduces effective
sample size and can destabilize guided sampling
\citep{doucet2001sequential,delmoral2006sequential}.
Prior work has improved SMC proposals through proposal tuning, importance
weight optimization, path-space objectives, optimal-control formulations,
annealing schedules, prior shifts, and timestep adaptation
\citep{kim2025tuning,ou2026inferencetime,baradat2020minimizing,
berner2024optimal,neal2001annealed,huang2026art,yoon2026psi,
ren2026driftlite}.
These methods often fix part of the underlying dynamics, optimize within a
restricted proposal family, or focus primarily on schedule adaptation.
Our approach instead jointly analyzes the local forward and backward Gaussian
kernels to construct low-variance proposals.

\paragraph{pMHC-specific binder design.}
Peptide--MHC (pMHC) targeting is a challenging specificity problem because
different targets largely share the same MHC scaffold, while specificity is
often determined by only a few exposed peptide residues.
Earlier TCR-mimic antibody design preserved MHC interactions while redesigning
peptide-contacting residues \citep{yang2023facile}.
More recent generative approaches have designed peptide-focused pMHC binders
using explicit off-target filtering, in silico cross-panning, and large-scale
off-target screening
\citep{liu2025design,johansen2025novo,householder2025novo}.
These approaches are closely related to classical negative design, which
favors the desired interaction while penalizing competing states
\citep{havranek2003automated}.
Our work instead studies pMHC specificity through inference-time guidance of a
pretrained binder generative model.

\begin{figure}[ht]
    \centering

    \begin{subfigure}{0.8\linewidth}
        \centering
        \includegraphics[width=\linewidth]{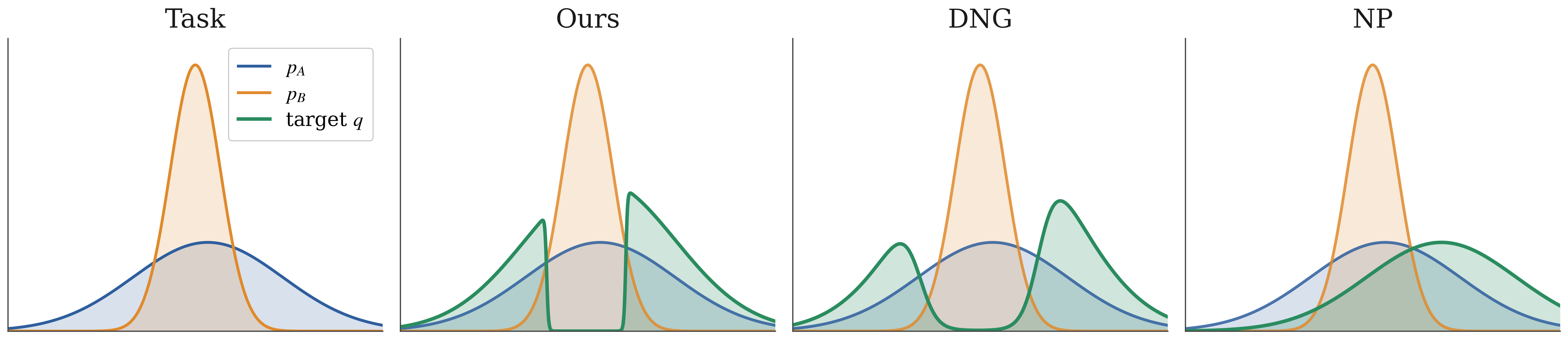}
        \label{fig:first}
    \end{subfigure}

    \vspace{-0.4cm}

    \begin{subfigure}{0.8\linewidth}
        \centering
        \includegraphics[width=\linewidth]{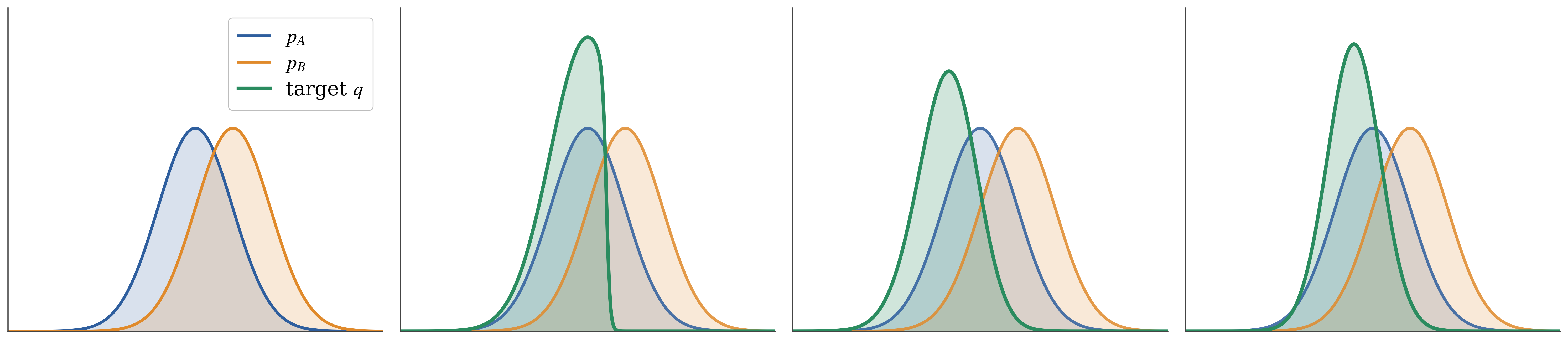}
        \label{fig:second}
    \end{subfigure}

    \caption{
Comparison of target distributions in two one-dimensional examples. 
Blue and orange denote the desired density \(p^A\) and undesired density \(p^B\), respectively; green denotes the induced target. 
The proposed target preserves regions where \(p^A\) dominates \(p^B\) and removes regions with high undesired density. 
By contrast, DNG can shift mass away from desired modes, while the negative-prompt (NP) construction does not consistently suppress the undesired region.
}
    \label{fig:two_vertical_subfigures}
\end{figure}

\section{Target Distribution Design}

Let \(p^A\) and \(p^B\) denote the desired and undesired reference
densities, respectively. We seek a target that retains regions likely under
\(p^A\) while suppressing regions supported by \(p^B\). For two densities
\(q\) and \(p\), define
\begin{equation}
    \mathrm{OVL}(q,p)
    :=
    \int \min\{q(x),p(x)\}\,dx .
\end{equation}
We consider the specificity objective
\begin{equation}
    J(q)
    :=
    \mathrm{OVL}(q,p^A)
    -
    \gamma\,\mathrm{OVL}(q,p^B),
    \qquad
    \gamma > 1.
\end{equation}
When \(J(q)\) has multiple maximizers, we select the one closest to \(p^A\)
in KL divergence.

\begin{restatable}[Optimal target and its soft relaxation]{proposition}{propone}
\label{prop:1}
Let
\[
    \mathcal A_\gamma
    :=
    \left\{
        x:p^A(x)>\gamma p^B(x)
    \right\}.
\]
Suppose that \(\gamma>1\), \(p^A(\mathcal A_\gamma)>0\), and the boundary
\(\{x:p^A(x)=\gamma p^B(x)\}\) has zero \(p^A\)-mass. Among all normalized
densities maximizing \(J(q)\), the unique minimizer of
\(D_{\mathrm{KL}}(q\|p^A)\) is
\begin{equation}
\label{target_1}
    q^*(x)
    =
    \frac{
        p^A(x)\mathbf 1\{x\in\mathcal A_\gamma\}
    }{
        \displaystyle
        \int_{\mathcal A_\gamma}p^A(u)\,du
    }.
\end{equation}

For the time-dependent marginals \(p_t^A\) and \(p_t^B\), define
\begin{equation}
    s_t(x)
    :=
    \log\frac{p_t^A(x)}{p_t^B(x)}
    -
    \log\gamma .
\end{equation}
Under pointwise Bernoulli-entropy regularization with inverse temperature
\(\eta_t>0\), the hard gate \(\mathbf 1\{s_t(x)>0\}\) is relaxed to
\begin{equation}
\label{eq:soft-gate}
\begin{aligned}
    1-\theta_t(x)
    &:=
    \sigma\!\left(\eta_t s_t(x)\right) \\
    &=
    \frac{
        \bigl(p_t^A(x)\bigr)^{\eta_t}
    }{
        \gamma^{\eta_t}\bigl(p_t^B(x)\bigr)^{\eta_t}
        +
        \bigl(p_t^A(x)\bigr)^{\eta_t}
    }.
\end{aligned}
\end{equation}
For any \(m_t>0\), the corresponding powered soft target is
\begin{equation}
\label{eq:powered-soft-target}
    q_t^{\eta,m}(x)
    =
    \frac{
        p_t^A(x)\bigl(1-\theta_t(x)\bigr)^{m_t}
    }{
        \displaystyle
        \int
        p_t^A(u)\bigl(1-\theta_t(u)\bigr)^{m_t}\,du
    }.
\end{equation}
For every fixed \(m>0\), \(q_t^{\eta,m}\) converges to the corresponding
hard target as \(\eta_t\to\infty\).
\end{restatable}

The parameter \(\eta_t\) controls the sharpness of the likelihood-ratio
boundary, while \(m_t\) controls the strength of the soft reweighting. Thus,
the powered target preserves regions specific to \(p_t^A\) and smoothly
downweights regions supported by \(p_t^B\), without changing the limiting
hard target. Proofs and further discussion are provided in
Appendices~\ref{sec:target} and~\ref{soften}.


\section{Local SMC Proposal Design}
\label{sec:sampler}

\subsection{Local Forward--Backward Proposal Design}
\label{sec:local-proposal}

Let $\tau=t+h$, with $h>0$. Suppose the current weighted particles
approximate $q_\tau$. To propagate them backward to $q_t$, we introduce
an auxiliary forward kernel $K_t(y\mid x)$ and sample from a backward
kernel $L_\tau(x\mid y)$. The ideal incremental Radon--Nikodym weight is
\begin{equation}
    \omega(x,y)
    =
    \frac{q_t(x)K_t(y\mid x)}
         {q_\tau(y)L_\tau(x\mid y)}
    \propto
    \frac{\gamma_t(x)K_t(y\mid x)}
         {\gamma_\tau(y)L_\tau(x\mid y)}.
    \label{eq:reweight}
\end{equation}
Thus, the proposal kernels can be adapted while the target remains fixed,
provided that the actual transition ratio is included in the weight
\citep{he2025rne,delmoral2004feynman}. We use this construction to motivate
an auxiliary forward proposal and a matched backward drift that removes
the leading stochastic variation of the Gaussian surrogate below.

\paragraph{Two-field forward proposal.}

Recall the powered soft target
\begin{equation}
\begin{aligned}
    q_t(z)\propto\gamma_t(z)
    &:=p_t^A(z)(1-\theta_t(z))^m
    =p_t^A(z)e^{\mathcal R_t(z)},\\
    \mathcal R_t(z)&:=m\log(1-\theta_t(z)),
\end{aligned}
\label{eq:sampler-powered-target}
\end{equation}
The gate schedules are independent of the spatial variable. Define
\begin{equation}
\begin{aligned}
    d_t(z)
    &:=\nabla_z\log p_t^A(z)-\nabla_z\log p_t^B(z),\\
    g_t(z)
    &:=\nabla_z\mathcal R_t(z)
    =\zeta_t(z)d_t(z),
    \qquad \zeta_t(z):=m\eta_t\theta_t(z).
\end{aligned}
\label{eq:two-forward-fields}
\end{equation}
Here $d_t$ is the desired--undesired score contrast, while $g_t$ is the
corresponding correction induced by the soft target. Although these fields
are pointwise collinear, their relative multiplier $\zeta_t(z)$ varies
across particles.

The two references share the forward drift
$\mu_r^A=\mu_r^B$ and the space-independent scalar diffusion
variance rate $s_r>0$. We adopt the decreasing-time reverse-drift
convention
\[
    \nu_r^j(z)
    :=
    \mu_r^j(z)-s_r\nabla_z\log p_r^j(z),
    \qquad j\in\{A,B\}.
\]
For each candidate $(\rho_1,\rho_2)$, the coefficients are held fixed
over the current interval $[t,\tau]$. We define
\[
    u_r^{(\rho_1,\rho_2)}(z)
    :=
    \rho_1 d_r(z)+\rho_2 g_r(z),
    \qquad r\in[t,\tau],
\]
where $\rho_1$ controls a direct correction along the
desired--undesired score contrast, while $\rho_2$ modulates the
reward-gradient component induced by the soft target. The resulting
auxiliary forward drift is
\begin{equation}
    b_t^{(\rho_1,\rho_2)}(z)
    =
    \mu_t^A(z)+s_tu_t^{(\rho_1,\rho_2)}(z).
    \label{eq:forward-proposal}
\end{equation}

For arbitrary forward and backward drifts $b$ and $a$, define
\begin{equation}
\begin{aligned}
    K_t^b(y\mid x)
    :=\mathcal N(y;x+hb_t(x),hs_tI),\qquad L_\tau^a(x\mid y)
    :=\mathcal N(x;y-ha_\tau(y),hs_\tau I).
\end{aligned}
\label{eq:local-gaussian-kernels}
\end{equation}
Conditioning on the current gate value, we define the operational Gaussian
log-weight surrogate
\begin{equation}
    \log\omega_{\rho_1,\rho_2}^{\mathrm G}(x,y)
    :=\Delta\mathcal R^{\mathrm G}(x,y)
    +\log
    \frac{L_\tau^{\nu^A}(x\mid y)}
         {K_t^{\mu^A}(y\mid x)}-\log
    \frac{L_\tau^{a_{\rho_1,\rho_2}}(x\mid y)}
         {K_t^{b^{(\rho_1,\rho_2)}}(y\mid x)}.
    \label{eq:two-field-local-weight}
\end{equation}
Here $\Delta\mathcal R^{\mathrm G}(x,y)$ is the one-step gate contribution
obtained from the same A/B reference Gaussian kernel-ratio approximation.
Its construction and score-based local spatial derivatives are specified
in Appendix~\ref{prop2proof}. In numerical evaluation, we use the same
notation for estimated gate values; their computation and any associated
estimator-state update follow the experimental protocol. This defines a
Gaussian weight surrogate, rather than an exact marginal RN weight.

\begin{restatable}[Two-Field Local Reweighting Expansion]
{proposition}{proptwo}
\label{prop2}
Assume the Gaussian kernels in \eqref{eq:local-gaussian-kernels}, the
Gaussian gate construction used in \eqref{eq:two-field-local-weight}, and
the regularity conditions in Appendix~\ref{prop2proof}. For the forward
proposal \eqref{eq:forward-proposal}, choose the matched backward drift
\begin{equation}
    a_{\rho_1,\rho_2}^\star(y)
    =\nu_\tau^A(y)
    +s_\tau\left(
        u_\tau^{(\rho_1,\rho_2)}(y)-g_\tau(y)
    \right).
    \label{eq:matched-backward}
\end{equation}
This cancels the $O_p(h^{1/2})$ term in the Gaussian log-weight surrogate.
Let
\[
    \Delta y:=x-y
    =-ha_{\rho_1,\rho_2}^\star(y)+\sqrt{hs_\tau}\,\xi,
    \qquad \xi\sim\mathcal N(0,I).
\]
Then
\begin{equation}
    \log\omega_{\rho_1,\rho_2}^{\mathrm G}
    =h\alpha_{\rho_1,\rho_2}(y)
    +\Delta y^\top\Gamma_{\rho_1,\rho_2}(y)\Delta y
    +O_p(h^{3/2}),
    \label{eq:local-weight-expansion}
\end{equation}
where
\begin{equation}
\begin{aligned}
    \Gamma_{\rho_1,\rho_2}(y)
    =-\left(\rho_1+\rho_2\zeta_\tau(y)\right)Jd_\tau(y)+\chi_\tau(y)\left(\rho_2-\frac12\right)
    d_\tau(y)d_\tau(y)^\top,
\end{aligned}
\label{eq:two-field-gamma}
\end{equation}
with
\[
    Jd_\tau(y):=\nabla_y d_\tau(y),
    \qquad
    \chi_\tau(y):=m\eta_\tau^2\theta_\tau(y)(1-\theta_\tau(y)).
\]
The scalar coefficient $\alpha_{\rho_1,\rho_2}(y)$ collects
the gate and kernel contributions. Its explicit expression
and the complete derivation are given in Appendix~\ref{prop2proof}.The remainder is uniform
over the fixed finite candidate grid used below.
\end{restatable}

We present the fixed-$m$ case; the time-varying-$m_t$ extension is given
in Appendix~\ref{prop2proof}. Conditioned on the current state, the scalar
coefficient is independent of the mutation noise, while the quadratic term
contributes both a conditional mean and transition-level fluctuations.
The $Jd$ term captures curvature of the score contrast, and the
$dd^\top$ term arises from the nonlinear gate.

\subsubsection{Fixed-Displacement JVP Approximation}
\label{sec:fixed-displacement-jvp}

A naive candidate-wise evaluation of \eqref{eq:local-weight-expansion}
would require candidate-dependent JVPs. We instead fix $\rho_1=\rho_{1,t}$
and reuse the displacement corresponding to $\rho_2=0$. Draw one
$\xi_i\sim\mathcal N(0,I)$ per particle, independently of the current
particle system, and define
\begin{equation}
\begin{aligned}
    a_{0,i}^\star
    &:=\nu_\tau^A(y_i)
    +s_\tau\left(\rho_{1,t}d_\tau(y_i)-g_\tau(y_i)\right),\\
    \Delta y_{0,i}
    &:=-ha_{0,i}^\star+\sqrt{hs_\tau}\,\xi_i.
\end{aligned}
\label{eq:fixed-displacement}
\end{equation}
For any candidate $\rho_2$,
\[
    \Delta y_i(\rho_2)
    =\Delta y_{0,i}-hs_\tau\rho_2g_\tau(y_i).
\]
The difference is $O(h)$, while $\Delta y_{0,i}=O_p(h^{1/2})$. Replacing
the candidate-dependent displacement inside the quadratic term therefore
changes the expansion only by $O_p(h^{3/2})$.

Define the fixed-displacement log-weight approximation by
\begin{align}
    \log\widehat\omega_i(\rho_2)
    &:=h\alpha_{\rho_{1,t},\rho_2}(y_i)
    -\left(\rho_{1,t}+\rho_2\zeta_\tau(y_i)\right)
    \Delta y_{0,i}^\top Jd_\tau(y_i)\Delta y_{0,i}
    \nonumber\\
    &\quad+\chi_\tau(y_i)\left(\rho_2-\frac12\right)
    \left(d_\tau(y_i)^\top\Delta y_{0,i}\right)^2.
    \label{eq:fixed-seed-local-weight}
\end{align}
For $x_i(\rho_2)=y_i+\Delta y_i(\rho_2)$, this gives
\[
    \log\omega_{\rho_{1,t},\rho_2}^{\mathrm G}
    \bigl(x_i(\rho_2),y_i\bigr)
    =\log\widehat\omega_i(\rho_2)+O_p(h^{3/2}).
\]
All candidates reuse $Jd_\tau(y_i)[\Delta y_{0,i}]$, which requires one
score JVP per reference model. The remaining terms require only scalar
coefficients and inner products.

\paragraph{Relative weight-variance minimization.}
For incoming weights $W_i$, let
$V_i(\rho_2):=W_i\widehat\omega_i^{\mathrm{use}}(\rho_2)$ and
$\bar V(\rho_2):=N^{-1}\sum_iV_i(\rho_2)$.
Here $\widehat\omega_i^{\mathrm{use}}$ includes any prescribed
stabilization, applied consistently in selection and reweighting.
We minimize the empirical relative variance which is the same as normalized cumulative ESS maximization:
\begin{align}
    \widehat{\mathcal V}_t^{\mathrm{rel}}(\rho_2)
    &:=\frac1N\sum_i
    \left(\frac{V_i(\rho_2)}{\bar V(\rho_2)}-1\right)^2,
    \label{eq:relative-cumulative-weight-variance}\\
    \rho_{2,t}^\star
    &\in\operatorname*{arg\,min}_{\rho_2\in\mathcal R_2}
    \widehat{\mathcal V}_t^{\mathrm{rel}}(\rho_2)
    =\operatorname*{arg\,max}_{\rho_2\in\mathcal R_2}
    \widehat{\mathrm{nESS}}_t(\rho_2).
    \label{eq:rho2-local-objective}
\end{align}
This equivalence holds for arbitrary incoming weights. For equal weights and small increments, incremental log-weight variance is its leading-order approximation; see Appendix~\ref{app:forward-kernel-matching}.

\subsection{Continuous-Time Interpretation}
\label{sec:continuous_interpretation}

The local Gaussian surrogate in Proposition~\ref{prop2} also admits a
natural continuous-time interpretation. Consider the ideal reference
setting with exact scores and gate values, and fix the proposal
coefficients over the current step. As shown in
Appendix~\ref{sec:continuous_connection}, the finite-step surrogate
satisfies
\begin{equation}
\begin{aligned}
    \log\omega_{\rho_1,\rho_2}^{\mathrm G}
    &=
    \log\omega_{\rho_1,\rho_2}^{\mathrm{cont}}
    +
    hs_\tau
    \left[
        \xi^\top\Gamma_{\rho_1,\rho_2}(y)\xi
        -
        \operatorname{tr}\Gamma_{\rho_1,\rho_2}(y)
    \right]
    +
    O_p(h^{3/2}),
\end{aligned}
\label{eq:finite-continuous-link}
\end{equation}
where $\omega_{\rho_1,\rho_2}^{\mathrm{cont}}$ is the corresponding
continuous Feynman--Kac incremental weight along the matched backward
process.

For a fixed candidate and $\xi\sim\mathcal N(0,I)$,
\[
    \mathbb E_\xi
    \left[
        \xi^\top\Gamma_{\rho_1,\rho_2}(y)\xi
        \mid y
    \right]
    =
    \operatorname{tr}\Gamma_{\rho_1,\rho_2}(y).
\]
Thus, the finite Gaussian surrogate and the continuous formulation share
the same conditional $O(h)$ contribution. The difference is that the
continuous formulation contains the trace correction, whereas the
finite-step surrogate retains its realized Gaussian quadratic form
$\xi^\top\Gamma\xi$. In this sense, the same mutation noise provides a
one-sample estimate of the continuous trace term while preserving
transition-level information from the finite step. The full derivation
and regularity conditions are given in
Appendix~\ref{sec:continuous_connection}.

\subsection{Relation to Variance-Minimizing Proposals}
\label{sec:relation-variance-methods}

For a fixed candidate $\boldsymbol\rho=(\rho_1,\rho_2)$, the variance of
the finite Gaussian log-weight surrogate decomposes as
\begin{equation}
\begin{aligned}
    \operatorname{Var}_{Y,\xi}
    \bigl(\log\widehat\omega_{\boldsymbol\rho}\bigr)
    &=
    \mathbb E_Y\!\left[
        \operatorname{Var}_\xi\!\left(
            \log\widehat\omega_{\boldsymbol\rho}\mid Y
        \right)
    \right] +
    \operatorname{Var}_Y\!\left(
        \mathbb E_\xi[
            \log\widehat\omega_{\boldsymbol\rho}\mid Y
        ]
    \right).
\end{aligned}
\label{eq:main-logweight-total-variance}
\end{equation}
The first term captures transition-level fluctuations, while the second
captures variation of the conditional mean across particles.

At the continuous level, after matching the target path and control family,
our matched drift induces the same Feynman--Kac potential as the
corresponding DriftLite-VCG control
\citep{ren2026driftlite}. Its potential-variance objective is the
leading-order coefficient of the second term above. Our finite Gaussian
construction instead evaluates the realized one-step weight, retaining the
quadratic trace estimator and the additional finite-displacement
corrections. The benefit is therefore in the finite-step weight
construction and optimization objective, rather than in separate
continuous control principles.

This also connects to DAS~\citep{kim2025testtime}: for a fixed forward
kernel, its ideal backward proposal minimizes the conditional transition
variance. Our coefficients additionally modify the forward kernel together
with its matched backward proposal. In practice, we select $\rho_2$ by
minimizing the empirical relative variance of the updated cumulative
surrogate weights, equivalently maximizing cumulative ESS in
\eqref{eq:rho2-local-objective}. Detailed comparisons are given in
Appendices~\ref{app:comparison-das} and
\ref{app:comparison-driftlite}.

\begin{algorithm}[t]
\caption{Common-Noise Two-Field SMC: One Stochastic Step}
\label{alg:smc-two-field}
\small
\begin{algorithmic}[1]
\Require Weighted particles $\{(y_i,W_i)\}_{i=1}^N$ and gate state;
step $h=\tau-t$; fixed $\rho_{1,t}$; grid $\mathcal R_2$;
previous coefficient $\rho_2^{\mathrm{prev}}$

\State Evaluate the reference fields and local gate coefficients.
\State Draw $\xi_i\sim\mathcal N(0,I)$ independently and form
$\Delta y_{0,i}$ by \eqref{eq:fixed-displacement}.
\State Compute the shared JVPs $Jd_\tau(y_i)[\Delta y_{0,i}]$.

\State Set $\mathcal C_t=\mathcal R_2$ on optimization steps;
otherwise $\mathcal C_t=\{\rho_2^{\mathrm{prev}}\}$.
\State Form $\widehat\omega_i^{\mathrm{use}}(\rho_2)$ for
$\rho_2\in\mathcal C_t$ using \eqref{eq:fixed-seed-local-weight}
and the prescribed stabilization.
\State Select
$\rho_{2,t}^\star\in
\operatorname*{arg\,min}_{\rho_2\in\mathcal C_t}
\widehat{\mathcal V}_t^{\mathrm{rel}}(\rho_2)$
by \eqref{eq:rho2-local-objective}.

\State $x_i\leftarrow y_i+\Delta y_{0,i}
-hs_\tau\rho_{2,t}^\star g_\tau(y_i)$.
\Comment{Reuse the same noise}
\State Update the gate state using the selected transition.
\State $\log W_i\leftarrow\log W_i+
\log\widehat\omega_i^{\mathrm{use}}(\rho_{2,t}^\star)$;
normalize the weights.
\State Resample particles and gate state jointly if triggered,
retaining residual importance weights.

\State \Return $\{(x_i,W_i)\}_{i=1}^N$, gate state,
and $\rho_{2,t}^\star$.
\end{algorithmic}
\end{algorithm}

\section{Experiment}
\label{sec:experiments}

\subsection{Synthetic Specificity on a 40-Component Gaussian Mixture}
\label{sec:mog}
\begin{figure}[h]
    \centering
    \includegraphics[width=1.0\linewidth]{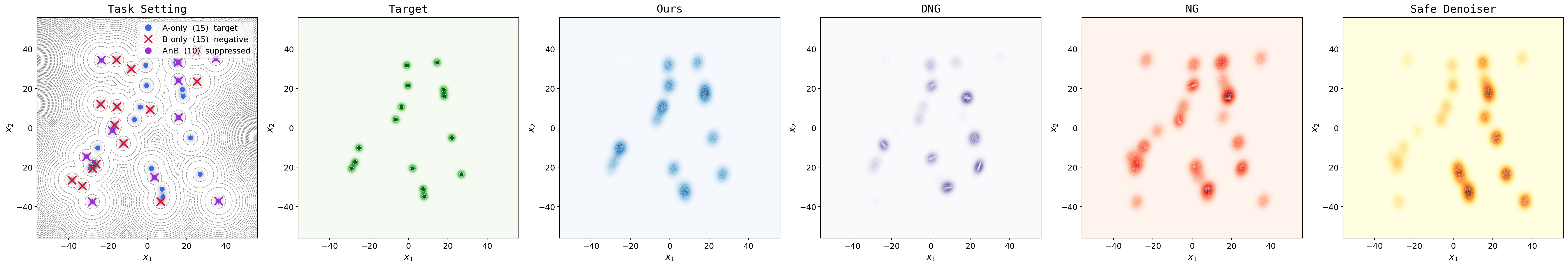}
    \caption{Visualization on the 40-MoG synthetic task. The task setting contains
\(A\)-only target modes, \(B\)-only negative modes, and \(A\cap B\)
suppressed modes. The target distribution retains only the \(A\)-only modes.}
    \label{fig:mog}
\end{figure}
In the first experiment, the 40-MoG synthetic dataset \citep{skreta2025feynman}\citep{he2025rne} is used to evaluate class-contrastive
sampling. As shown in Fig.\ref{fig:mog}, the 40 Gaussian modes are randomly divided into
three groups: \(A\)-only modes, \(B\)-only modes, and overlapping \(A\cap B\)
modes. The desired target distribution retains the \(A\)-only modes while
suppressing both the \(B\)-only and \(A\cap B\) modes.

The proposed sampler produces samples that are visually closest to the
target distribution. Dynamic Negative Guidance
(DNG) \citep{koulischer2024dynamic} suppresses several negative regions, but it also
induces mode shifts, especially for the modes in the lower-right region.
Negative Guidance fails to represent the target distribution accurately and
retains substantial mass near negative regions. SSafe Denoiser (SD)
\citep{kim2025training} removes the
negative modes more effectively than Negative Guidance, but its repulsion is
driven by the absolute \(B\)-density. Consequently, it can over-repel samples
from regions where \(B\) is dense, causing missing or weakened positive modes
that are surrounded by or close to negative modes, such as those in the
lower-left region. This supports the use of a ratio-based gate, which
activates repulsion according to the relative evidence \(p_B/p_A\) rather
than the absolute negative density alone.
\subsection{CIFAR-10 Specificity}
\label{sec:cifar10}

\begin{figure}[ht]
    \centering

    \begin{subfigure}{0.24\linewidth}
        \centering
        \includegraphics[width=\linewidth]{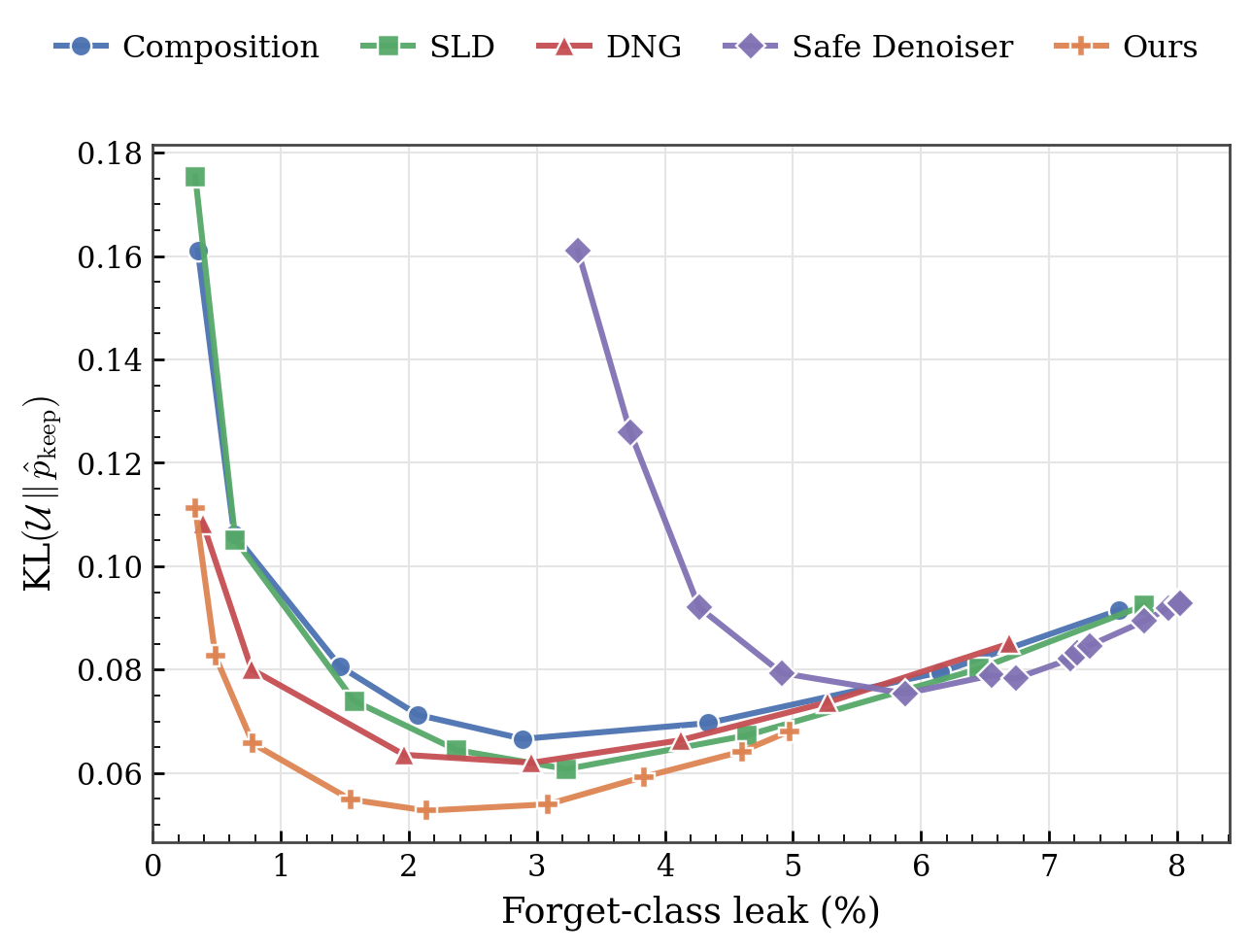}
        \label{fig:sub1}
    \end{subfigure}
    \hfill
    \begin{subfigure}{0.24\linewidth}
        \centering
        \includegraphics[width=\linewidth]{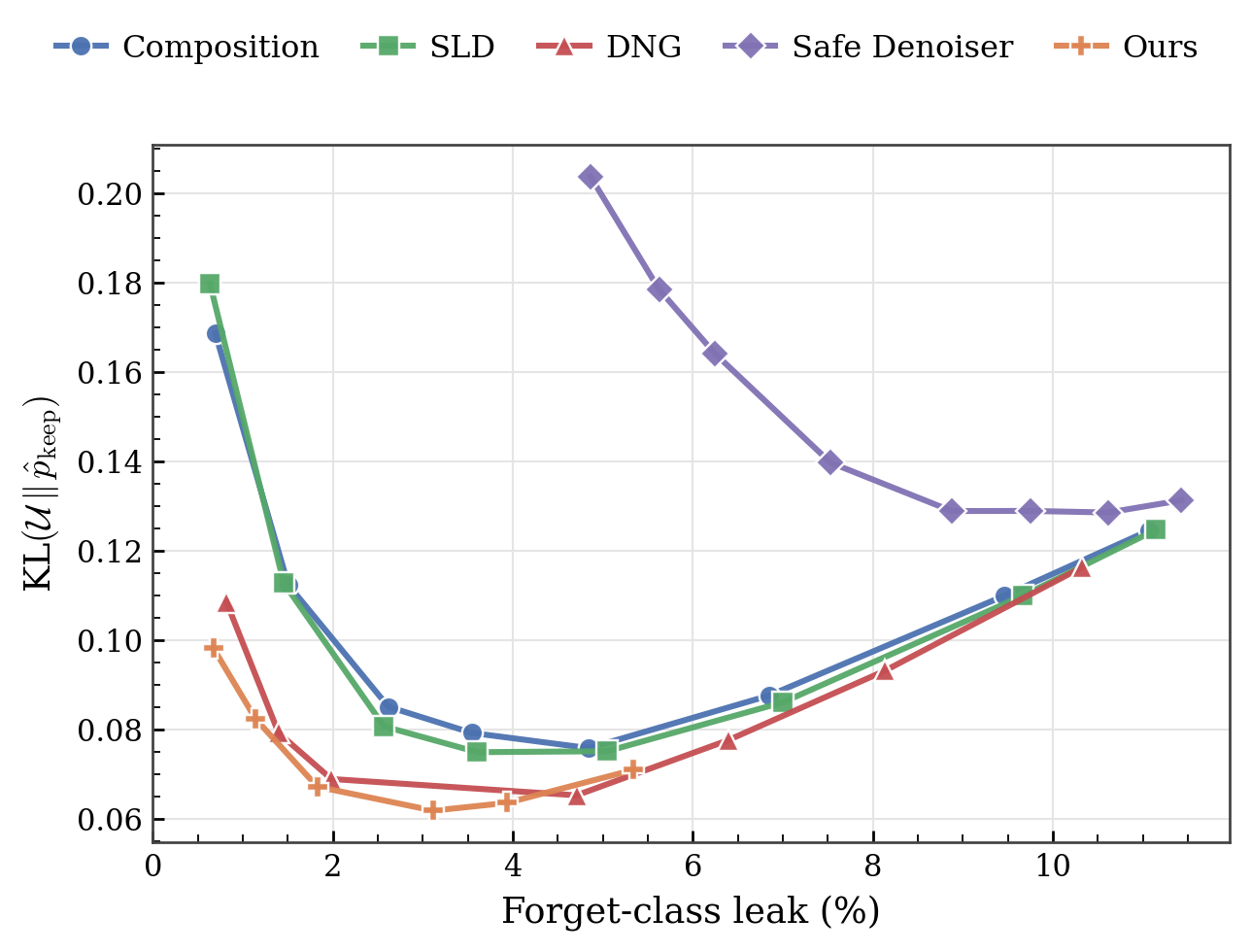}
        \label{fig:sub2}
    \end{subfigure}
    \hfill
    \begin{subfigure}{0.24\linewidth}
        \centering
        \includegraphics[width=\linewidth]{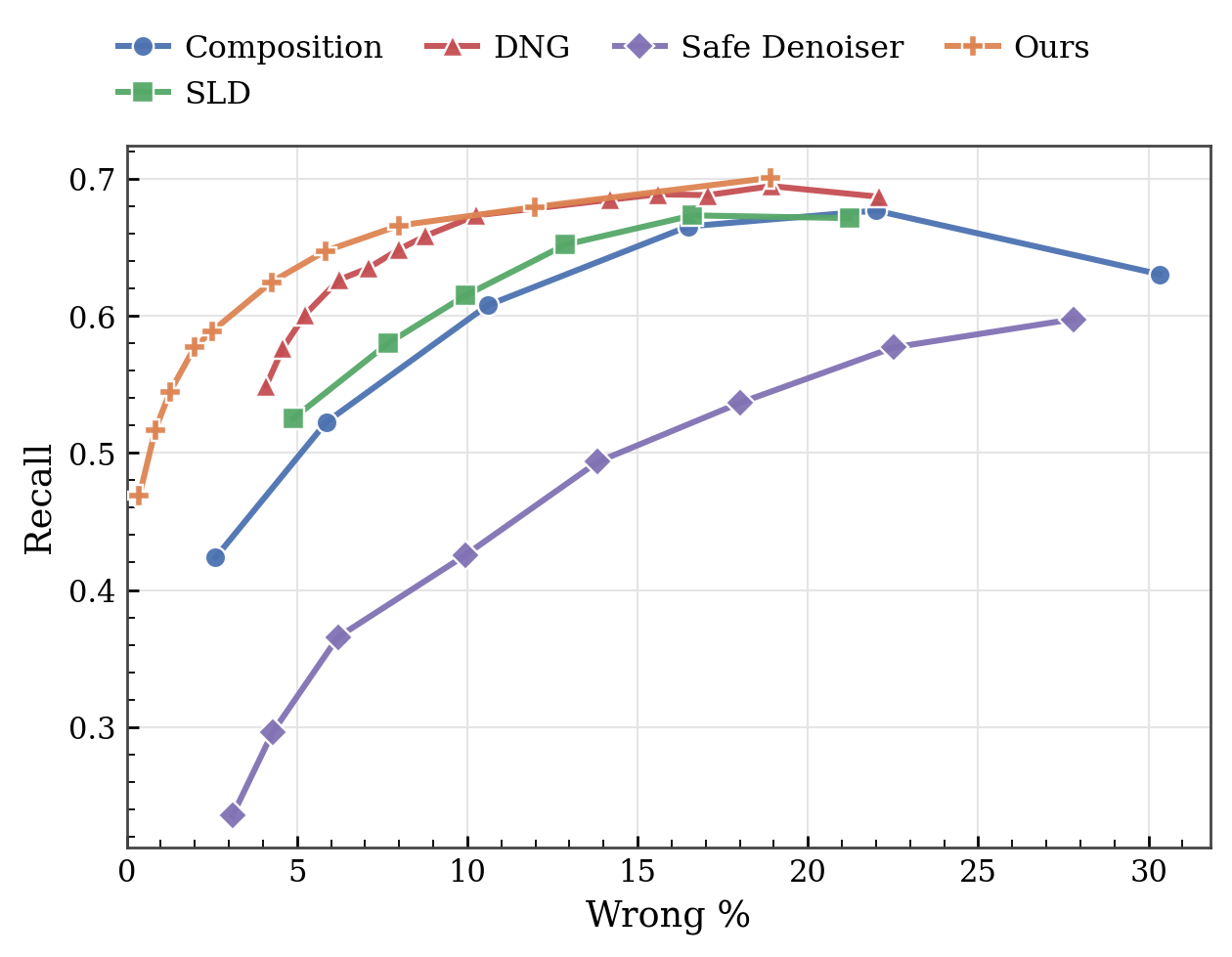}
        \label{fig:sub3}
    \end{subfigure}
    \hfill
    \begin{subfigure}{0.24\linewidth}
        \centering
        \includegraphics[width=\linewidth]{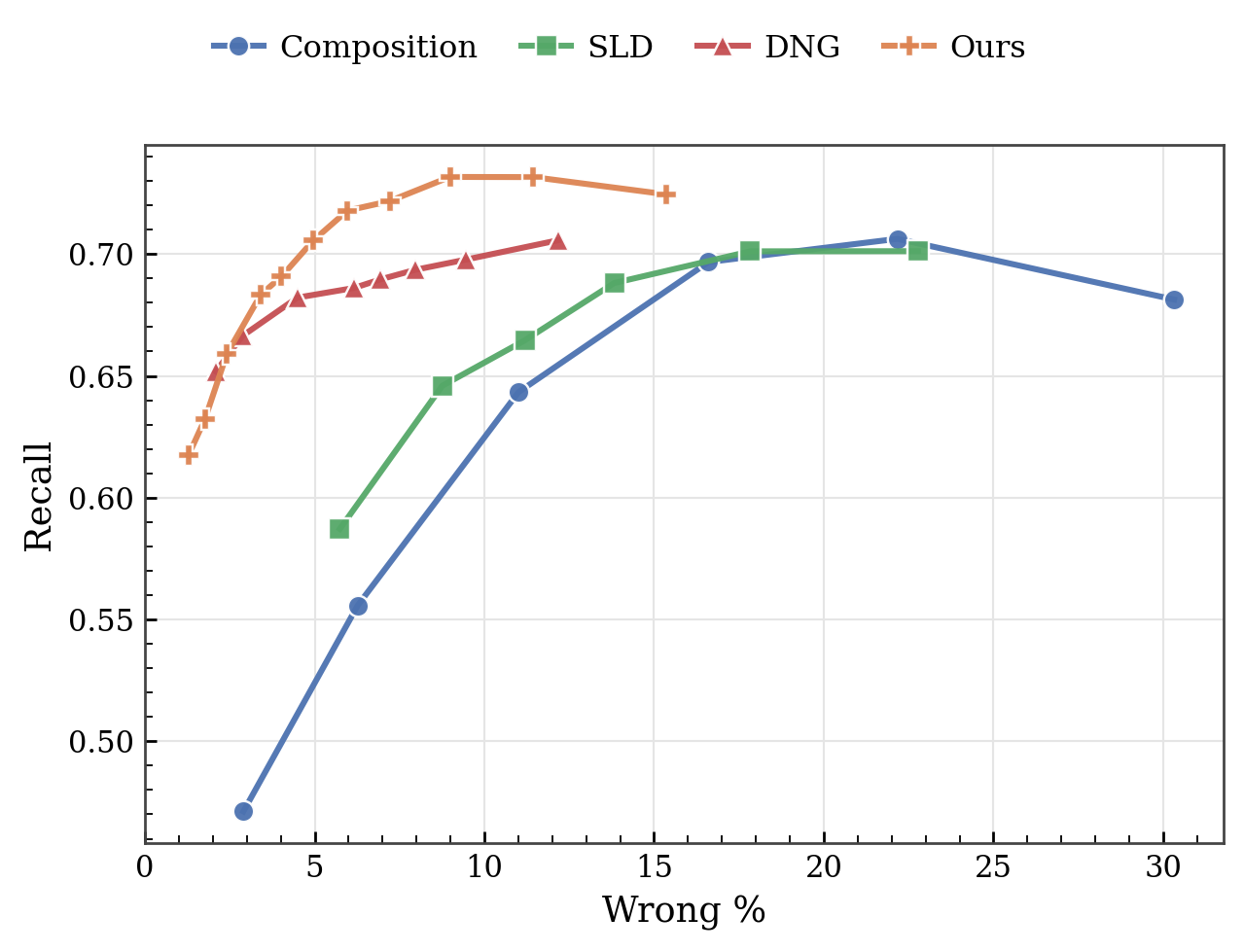}
        \label{fig:sub4}
    \end{subfigure}

    \vspace{0.5em}

    \begin{subfigure}{0.24\linewidth}
        \centering
        \includegraphics[width=\linewidth]{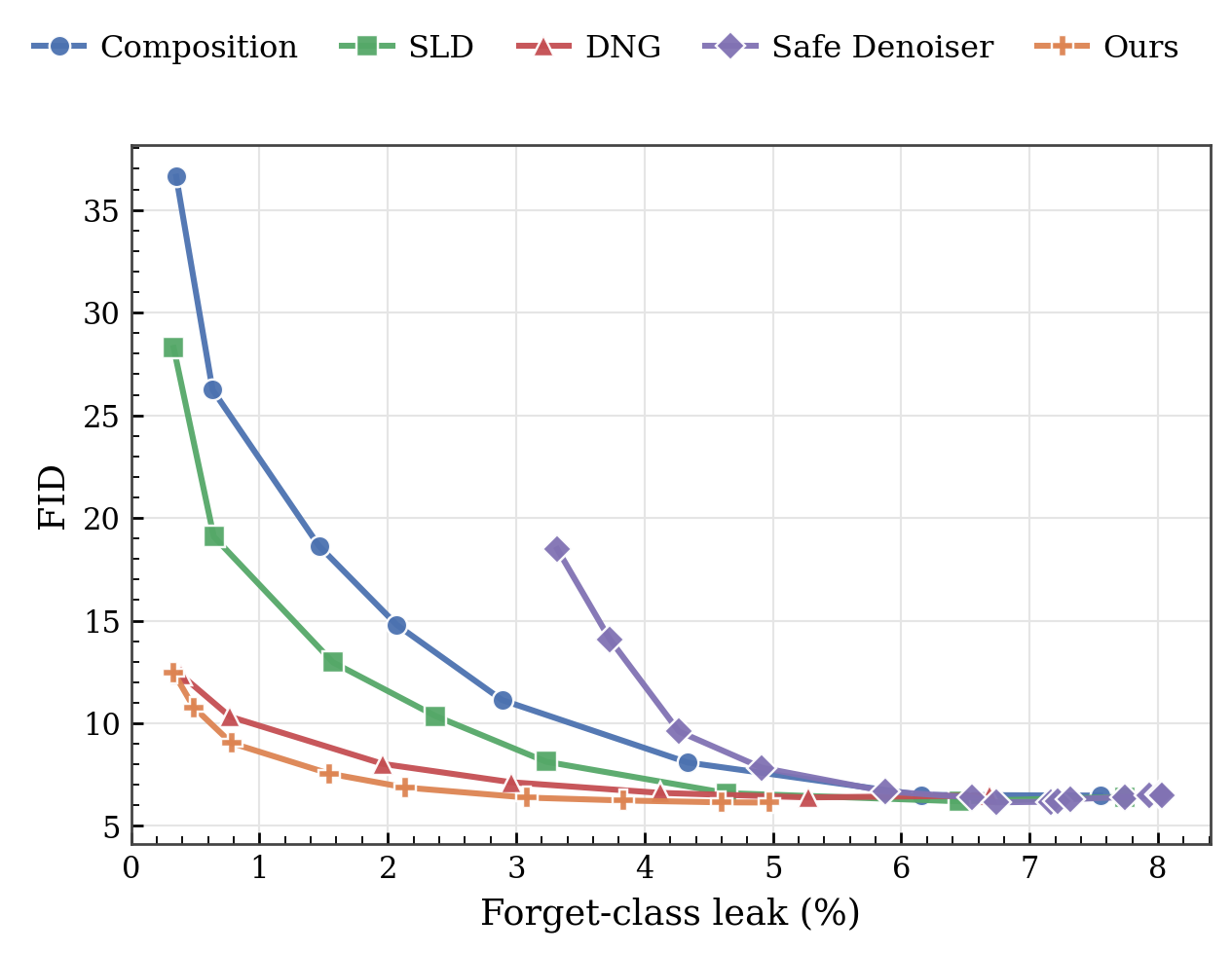}
        \caption{Removing Class 0}
        \label{fig:sub5}
    \end{subfigure}
    \hfill
    \begin{subfigure}{0.24\linewidth}
        \centering
        \includegraphics[width=\linewidth]{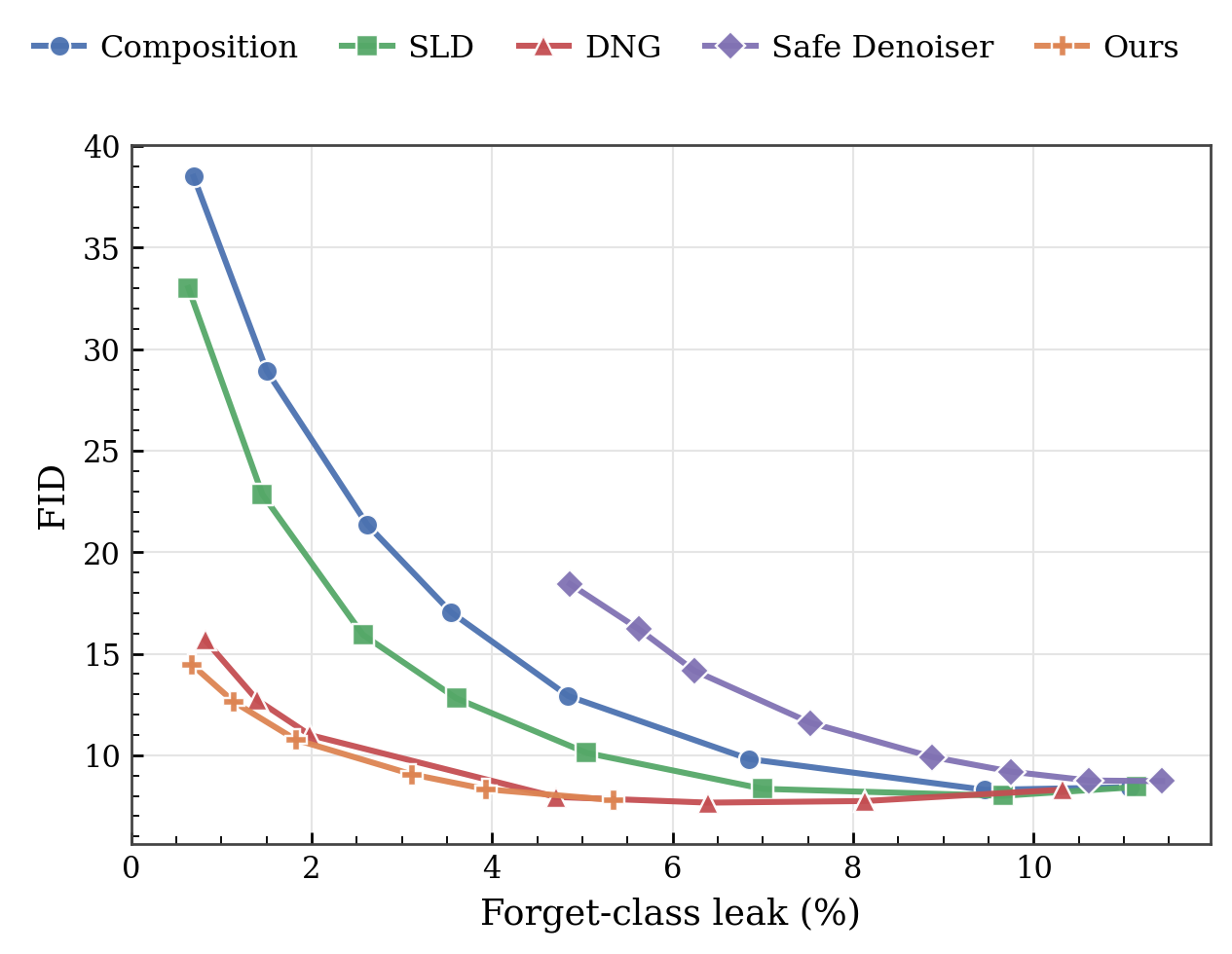}
        \caption{Removing Class 1}
        \label{fig:sub6}
    \end{subfigure}
    \hfill
    \begin{subfigure}{0.24\linewidth}
        \centering
        \includegraphics[width=\linewidth]{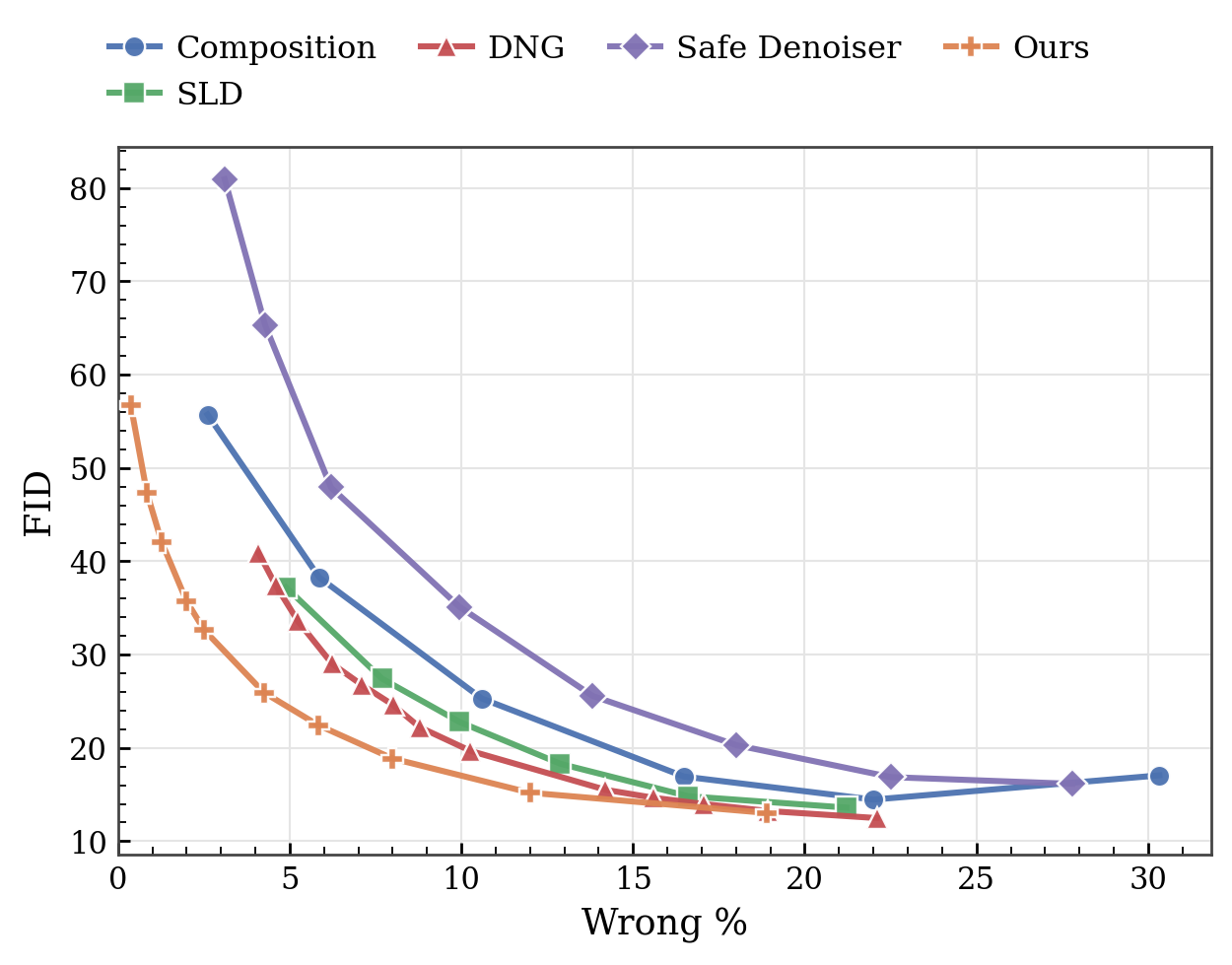}
        \caption{Keeping Classes 0,1,2}
        \label{fig:sub7}
    \end{subfigure}
    \hfill
    \begin{subfigure}{0.24\linewidth}
        \centering
        \includegraphics[width=\linewidth]{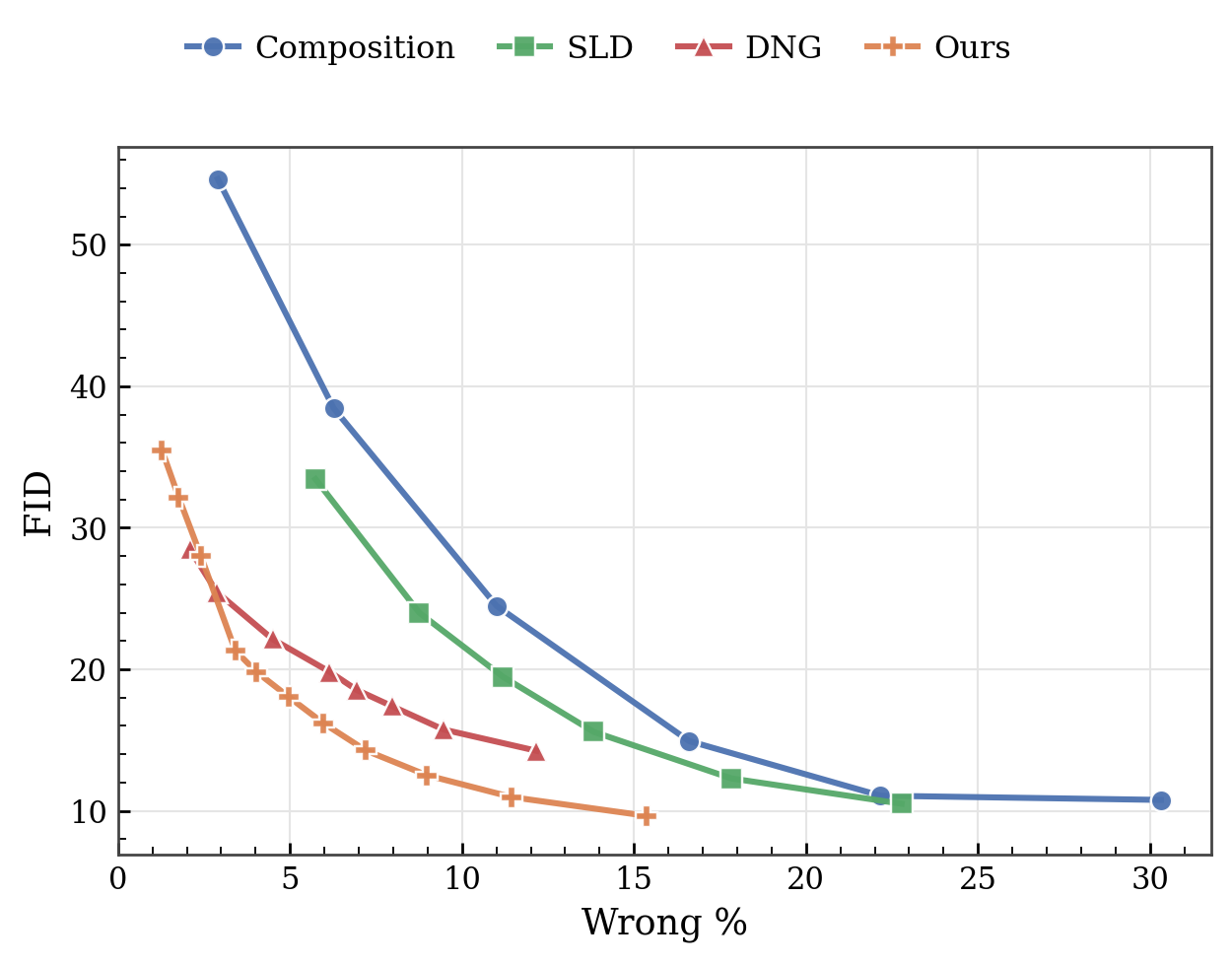}
        \caption{Keeping Classes 5,6,7\(^\dagger\)}
        \label{fig:sub8}
    \end{subfigure}

    \caption{CIFAR-10 specificity Pareto frontiers with \(N=10{,}240\)
    samples. Cols.~1--2 show class removal
    (\S\ref{sec:cifar10:cr}); Cols.~3--4 show multiclass selection
    (\S\ref{sec:cifar10:mc}). Bottom row: FID versus Wrong\%. Top row:
    KL versus Wrong\% for class removal and Recall versus Wrong\% for
    multiclass selection. \(^{\dagger}\)Safe Denoiser is omitted because
    its Wrong\% does not reach the plotted region across multiple kernel
    bandwidths; see App.~\ref{app:cifar-mc-details}.}
    \label{fig:two_by_four}
\end{figure}

We evaluate two CIFAR-10 specificity tasks. In class removal
(\S\ref{sec:cifar10:cr}), \(p^A\) is an unconditional generator and \(p^B\)
is a class-only generator trained on one forget class. In multiclass selection
(\S\ref{sec:cifar10:mc}), \(p^A\) and \(p^B\) are overlapping multiclass
generators, and the target is the \(p^A\)-only modes.

\subsubsection{Class Removal}
\label{sec:cifar10:cr}

We remove planes (c0) and automobiles (c1). For each method, we sweep the
guidance scale using \(N=10{,}240\) samples and report forget-class leak
(Wrong\%, ViT classifier), keep-class histogram KL to uniform, and FID to
the keep-class real subset. Baselines include Composition, Safe Latent
Diffusion (SLD) \citep{schramowski2023safe,koulischer2024dynamic}, DNG,
and SD. Full protocols and per-setting results are provided in
App.~\ref{app:cifar-cr-details}.

Our sampler gives the best Pareto trade-off for both removed classes
(Fig.~\ref{fig:two_by_four}). The gain is consistent with the target
design in \S\ref{sec:target}: DNG penalizes directly through
\(r(x)=p^B/p^A\), while our target uses \(r(x)^\lambda\), reducing
over-suppression near the positive--negative boundary. In contrast, SD
uses an effectively always-active repulsive correction, which leads to a
Wrong\% floor and larger keep-class distortion under stronger guidance.

\subsubsection{Multiclass Selection}
\label{sec:cifar10:mc}

We next consider overlapping five-class generators. Model \(p^A\) covers
\(\{0,1,2,3,4\}\), model \(p^B\) covers \(\{3,4,5,6,7\}\), and the goal is
to sample from \(A\setminus B\). The forward task keeps \(\{0,1,2\}\),
while the reverse task keeps \(\{5,6,7\}\). Because only three classes are
retained, we report Recall \citep{kynkaanniemi2019improved} together with
FID to measure both sample quality and mode coverage. Further details are
given in App.~\ref{app:cifar-mc-details}.

Our sampler achieves the strongest FID--Recall Pareto frontier in both
directions, with the largest gain on the harder forward task
(Fig.~\ref{fig:two_by_four}). The overlap between the two generators
widens the transition region where the class densities are comparable,
making excessive negative suppression more damaging. The powered
likelihood-ratio gate better preserves valid modes near this boundary.

\subsection{Text-to-Image}
\label{sec:t2i}

We evaluate text-to-image specificity with Stable Diffusion v1.4
\citep{rombach2022high} on the five-prompt DNG suite
\citep{koulischer2024dynamic}. Each positive prompt is paired with a
\emph{related} negative, testing suppression of semantically overlapping
undesired content, and an \emph{unrelated} negative, testing unnecessary
interference with valid generation. For each prompt and method, the same
hyperparameter setting is used for both negative types. Full prompts and
settings are given in Appendix~\ref{app:t2i-details}.

All methods use $100$ reverse steps. We compare against CFG with negative
prompting, DNG, Feynman--Kac Corrector (FKC), Feynman--Kac Steering (FKS),
DriftLite, and DAS. FKC, FKS, DriftLite, and DAS use our target distribution
to isolate the effect of proposal design. Each method is tuned for the
overall margin.

\begin{table}[ht]
    \centering
    \caption{Text-to-image specificity with SD~v1.4 and $100$ reverse
    steps, averaged over the 10 DNG prompt pairs. For each prompt and
    method, a single hyperparameter setting is selected and used for
    both related and unrelated negatives. CLIP$_+$ is similarity to the
    positive prompt ($\uparrow$), CLIP$_-$ is similarity to the active
    negative prompt ($\downarrow$), and margin is
    CLIP$_+ -$ CLIP$_-$ ($\uparrow$).}
    \label{tab:t2i-main}

    \resizebox{0.8\linewidth}{!}{
    \begin{tabular}{lccccccc}
        \toprule
        & \multicolumn{3}{c}{Related Negative}
        & \multicolumn{3}{c}{Unrelated Negative} & Overall\\
        \cmidrule(lr){2-4}
        \cmidrule(lr){5-7}
        \cmidrule(lr){8-8}
        Method
        & CLIP$_+$ & CLIP$_-$ & Margin
        & CLIP$_+$ & CLIP$_-$ & Margin
        & Margin\\
        \midrule
        Ours (Proposition~2)
        & 0.2722 & \textbf{0.1413} & \textbf{0.1309}
        & 0.2679 & 0.0655 & 0.2024
        & \textbf{0.1666} \\
        CFG
        & 0.2475 & 0.1473 & 0.1002
        & 0.2561 & \textbf{0.0632} & 0.1928 & 0.1465\\
        DNG
        & 0.2758 & 0.1766 & 0.0992
        & 0.2748 & 0.0812 & 0.1937 & 0.1465\\
        FKC
        & \textbf{0.2788} & 0.1652 & 0.1136
        & 0.2796 & 0.0822 & 0.1974 & 0.1555\\
        FK steering
        & 0.2695 & 0.1758 & 0.0937
        & 0.2727 & 0.0843 & 0.1884 & 0.1411\\
        DriftLite
        & 0.2700 & 0.1472 & 0.1228
        & 0.2652 & 0.0737 & 0.1915 & 0.1572\\
        DAS
        & 0.2721 & 0.1607 & 0.1114
        & \textbf{0.2800} & 0.0694 & \textbf{0.2107} & 0.1611\\
        \bottomrule
    \end{tabular}}
\end{table}

Table~\ref{tab:t2i-main} shows that our method achieves the largest
overall margin while maintaining competitive positive-prompt alignment.
For related negatives, it obtains both the lowest CLIP$_-$ and the largest
margin, improving over DriftLite, the second-best method, by $6.6\%$.
For unrelated negatives, it again gives the lowest CLIP$_-$ and the
second-highest margin, close to DAS. These results indicate stronger
selective suppression without uniformly weakening the positive prompt.

Figure~\ref{fig:t2i-collapse-compare} further shows that proposal
optimization prevents the severe visual collapse observed with direct
SMC while preserving image quality. Additional results are provided in
Appendix~\ref{app:t2i-details}.

\begin{figure}[H]
    \centering
    \includegraphics[width=0.5\linewidth]
    {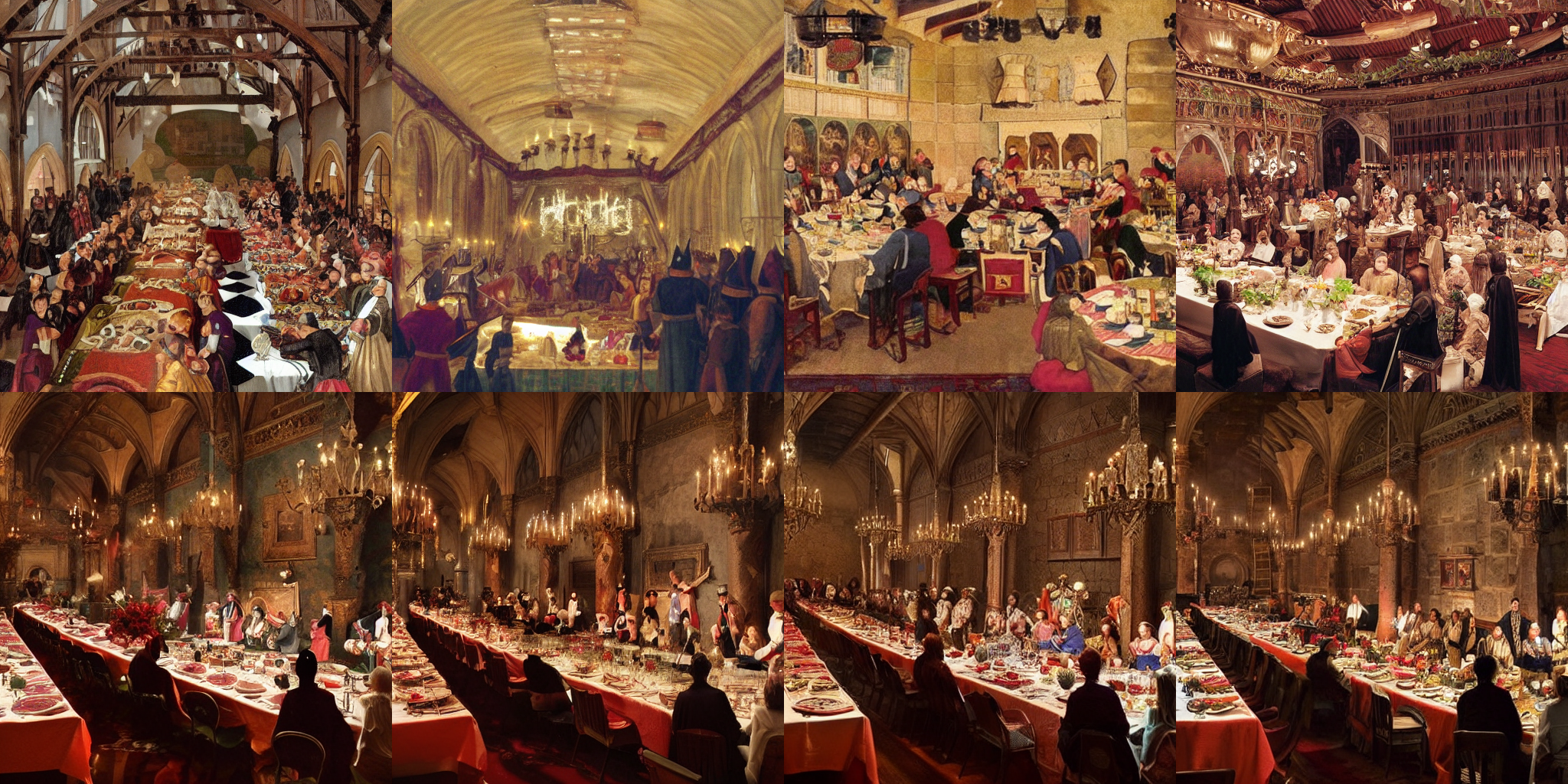}
    \caption{Samples for ``medieval feast.'' The first row uses proposal
    optimization and the second uses direct SMC. Proposal optimization
    substantially reduces visual collapse.}
    \label{fig:t2i-collapse-compare}
\end{figure}

\subsection{Specificity-Aware pMHC Binder Design}
\label{sec:binder-experiment}

We evaluate our method on a protein-design task that requires specificity
between two closely related targets. The desired target is an
HLA-A*02:01--$\beta_2$m complex presenting \texttt{NLVPMVATV}. The undesired
target has the same HLA-A*02:01--$\beta_2$m scaffold but presents
\texttt{NLVFMVATV}. The peptides differ only at position~4, where Pro in the
desired target is replaced by Phe in the undesired target. Thus, successful
discrimination must depend on the presented peptide rather than on a change
in the MHC scaffold.

We design an $80$-residue binder using the desired pMHC as the positive
condition and the undesired pMHC as the negative condition. The candidate and
propagated weights use the same retained Proposition~2 expansion. The
complete target construction, weight terms, clipping rules, and
hyperparameters are given in Appendix~\ref{app:binder-details}.

We use eight matched to generate 64 candidate binders. We then retain one inverse-folded sequence per backbone. Generation uses $200$
reverse-diffusion steps. Each sequence is modeled with both the desired and
undesired pMHC using Boltz-2. 

\paragraph{Evaluation metric.}
For each Boltz-2 prediction, we symmetrize the two entries of the
\texttt{pair\_chains\_iptm} matrix that correspond to the binder and peptide.
Let $a_{i,c,k}$ be this binder--peptide iPTM for binder $i$, condition $c$,
and prediction $k$. We average over the $K=3$ predictions for each condition and define the specificity margin
\begin{equation}
    A_{i,c}
    =
    \frac{1}{K}\sum_{k=1}^{K}a_{i,c,k}, \qquad \Delta_i
    =
    A_{i,\mathrm{desired}}
    -
    A_{i,\mathrm{undesired}}.
\end{equation}
It is a structural confidence measure and should not be
interpreted as an experimental binding-affinity measurement.

\paragraph{Specificity results.}
We compare our method with target-A-only BoltzGen, fixed contrastive guidance,
and Dynamic Negative Guidance (DNG). Every method is evaluated on $N=64$
binders generated from the same seed blocks and with the same
desired-minus-undesired specificity metric. All methods use three Boltz-2
predictions per binder and condition.

\begin{table}[h]
    \centering
    \caption{
    Specificity on the pMHC binder-design task. Here, $N$ is the number of
    evaluated binders, Q75 is the $75$th percentile, and $\Delta>0$ is the
    fraction of binders with a positive specificity margin. Higher values
    are better.
    }
    \label{tab:binder-specificity}
    \small
    \resizebox{0.7\linewidth}{!}{
    \begin{tabular}{lccccc}
        \toprule
        \textbf{Method}
        & $\boldsymbol{N}$
        & \textbf{Mean} $\uparrow$
        & \textbf{Median} $\uparrow$
        & \textbf{Q75} $\uparrow$
        & $\boldsymbol{\Delta>0}$ $\uparrow$ \\
        \midrule
        BoltzGen (target A only)
                         & 64 & 0.1786 & 0.1587 & 0.2862 & 87.5\% \\
        Fixed guidance   & 64 & 0.1654 & 0.1453 & 0.2833
                         & \textbf{93.8\%} \\
        DNG              & 64 & 0.1483 & 0.1581 & 0.2623 & 89.1\% \\
        \textbf{Ours (Proposition~2)}
                         & 64
                         & \textbf{0.1856}
                         & \textbf{0.1822}
                         & \textbf{0.3337}
                         & 84.4\% \\
        \bottomrule
    \end{tabular}
    }
\end{table}

Our method achieves the largest mean, median, and Q75 specificity margins.
Relative to the strongest baseline on these statistics, target-A-only
BoltzGen, it improves the mean by $3.9\%$, the median by $14.8\%$, and Q75 by
$16.6\%$. Fixed guidance has the largest fraction of positive margins. Thus,
our method improves the central tendency and upper quartile, although it
does not maximize the fraction of positive designs.

\paragraph{Structural example.}
Figure~\ref{fig:pmhc-design} shows predicted structures for a representative
binder with peptide-dependent recognition. Its mean binder--peptide iPTM is
\begin{equation}
    A_{\mathrm{desired}}=0.8673,
    \qquad
    A_{\mathrm{undesired}}=0.5174,
    \qquad
    \Delta=0.3499.
\end{equation}
The displayed desired-target prediction has placed the binder directly against the peptide, including the mutated
position. The matched undesired-target prediction places the binder away from the peptide. Across the three
predictions, the binder contacts the desired peptide in two cases and the
undesired peptide in none. The binder sequence and additional
details are provided in Appendix~\ref{app:binder-details}.

\begin{figure*}[t]
    \centering

    \begin{minipage}{0.62\textwidth}
        \centering

        \begin{subfigure}[c]{0.55\linewidth}
            \centering
            \includegraphics[width=\linewidth]{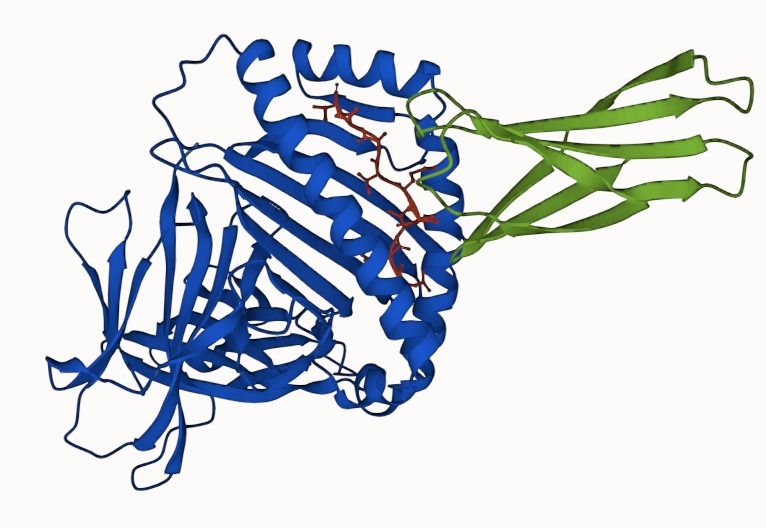}
            \caption{Desired pMHC presenting \texttt{NLVPMVATV}.}
            \label{fig:wanted}
        \end{subfigure}
        \hfill
        \begin{subfigure}[c]{0.41\linewidth}
            \centering
            \includegraphics[width=0.90\linewidth]{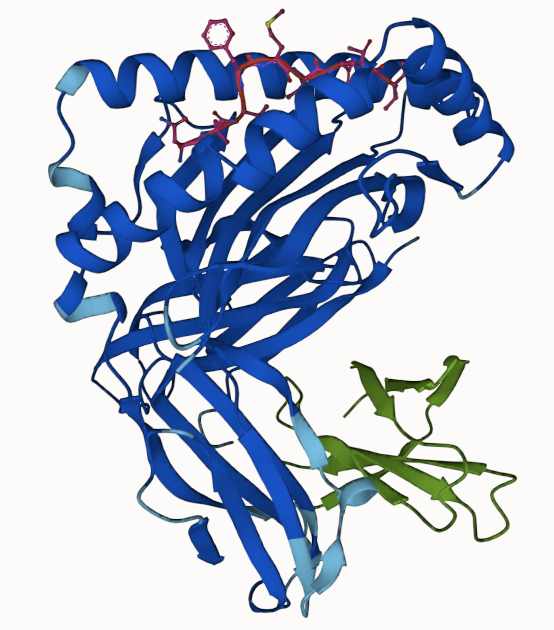}
            \caption{Undesired pMHC presenting \texttt{NLVFMVATV}.}
            \label{fig:unwanted}
        \end{subfigure}

    \end{minipage}
    \caption{
    Representative predicted complexes for a specificity-selective binder.
    The binder is shown in green, and the peptide is shown in red.
    The remaining chains form the shared HLA-A*02:01--$\beta_2$m scaffold.
    The desired target supports a peptide-adjacent binder pose, whereas the
    displayed prediction for the undesired P4F target places the binder away
    from the peptide.
    }
    \label{fig:pmhc-design}
\end{figure*}

\section{Conclusion}

We presented a specificity-aware diffusion steering framework that preserves a
desired distribution while suppressing an undesired one. We first derive an
explicit target from an overlap-based objective. Its softened likelihood-ratio
gate retains regions specific to the desired distribution without inducing the
global mode shift of standard negative guidance. We then construct an SMC sampler that locally minimizes the empirical relative variance of cumulative surrogate weights using the full displacement and one score Jacobian--vector product
per reference model.

Experiments on synthetic 40-MoG, CIFAR-10, text-to-image generation, and
pMHC binder design show improved target fidelity, specificity, and particle
diversity. The framework is
distributionally grounded, training-free, and applicable when desired and
undesired reference models are available.

\section{Limitations}

Our method assumes stochastic SDE inference and a local one-step
approximation; extending it to deterministic solvers, EDM
parameterizations, and finite-time flow maps is left for future work.
It also introduces application-specific hyperparameters and additional
compute from particle sampling with two reference models. Evaluation is
limited by the lack of a large-scale text-to-image specificity benchmark.
For binder design, we study only one peptide pair with a modest number of
candidates, and assess specificity using predicted iPTM rather than
experimental affinity. Larger-scale studies with experimental on- and
off-target validation are needed.

\section*{Acknowledgments}

We sincerely thank Prof. Tommi Jaakkola for valuable discussions and feedback on the formulation of the target distribution, the choice and interpretation of baselines, and the design about the sampling stability. His insights on variance reduction and target distribution design were particularly helpful in shaping the theoretical motivation and positioning of this work.

\newpage
\bibliographystyle{unsrtnat}
\bibliography{reference}

@book{delmoral2004feynman,
  title={Feynman-Kac Formulae: Genealogical and Interacting Particle Systems with Applications},
  author={Del Moral, Pierre},
  year={2004},
  publisher={Springer}
}

@article{liu2025design,
  title={Design of high-specificity binders for peptide--MHC-I complexes},
  author={Liu, Bingxu and Greenwood, Nathan F and Bonzanini, Julia E and Motmaen, Amir and Meyerberg, Jeremy and Dao, Tao and Xiang, Xinyu and Ault, Russell and Sharp, Jazmin and Wang, Chunyu and others},
  journal={Science},
  volume={389},
  number={6758},
  pages={386--391},
  year={2025},
  publisher={American Association for the Advancement of Science}
}

@article{householder2025novo,
  title={De novo design and structure of a peptide--centric TCR mimic binding module},
  author={Householder, Karsten D and Xiang, Xinyu and Jude, Kevin M and Deng, Arthur and Obenaus, Matthias and Zhao, Yang and Wilson, Steven C and Chen, Xiaojing and Wang, Nan and Garcia, K Christopher},
  journal={Science},
  volume={389},
  number={6758},
  pages={375--379},
  year={2025},
  publisher={American Association for the Advancement of Science}
}

@article{blondel2020learning,
  title={Learning with fenchel-young losses},
  author={Blondel, Mathieu and Martins, Andr{\'e} FT and Niculae, Vlad},
  journal={Journal of Machine Learning Research},
  volume={21},
  number={35},
  pages={1--69},
  year={2020}
}

@article{havranek2003automated,
  title={Automated design of specificity in molecular recognition},
  author={Havranek, James J and Harbury, Pehr B},
  journal={nature structural biology},
  volume={10},
  number={1},
  pages={45--52},
  year={2003},
  publisher={Nature Publishing Group US New York}
}

@article{yang2023facile,
  title={Facile repurposing of peptide--MHC-restricted antibodies for cancer immunotherapy},
  author={Yang, Xinbo and Nishimiya, Daisuke and L{\"o}chte, Sara and Jude, Kevin M and Borowska, Marta and Savvides, Christina S and Dougan, Michael and Su, Leon and Zhao, Xiang and Piehler, Jacob and others},
  journal={Nature biotechnology},
  volume={41},
  number={7},
  pages={932--943},
  year={2023},
  publisher={Nature Publishing Group US New York}
}

@article{johansen2025novo,
  title={De novo-designed pMHC binders facilitate T cell--mediated cytotoxicity toward cancer cells},
  author={Johansen, Kristoffer Haurum and Wolff, Darian Stephan and Scapolo, Beatrice and Fern{\'a}ndez-Quintero, Monica L and Risager Christensen, Charlotte and Loeffler, Johannes R and Rivera-de-Torre, Esperanza and Overath, Max D and Kj{\ae}rgaard Munk, Kamilla and Morell, Oliver and others},
  journal={Science},
  volume={389},
  number={6758},
  pages={380--385},
  year={2025},
  publisher={American Association for the Advancement of Science}
}

@article{yoon2026psi,
  title={{$\Psi$-Sampler}: Initial Particle Sampling for {SMC}-Based Inference-Time Reward Alignment in Score Models},
  author={Yoon, Taehoon and Min, Yunhong and Yeo, Kyeongmin and Sung, Minhyuk},
  journal={Advances in Neural Information Processing Systems},
  volume={38},
  pages={104745--104781},
  year={2026}
}

@inproceedings{ren2026driftlite,
  title={Driftlite: Lightweight drift control for inference-time scaling of diffusion models},
  author={Ren, Yinuo and Gao, Wenhao and Ying, Lexing and Rotskoff, Grant and Han, Jiequn},
  booktitle={International Conference on Learning Representations},
  volume={2026},
  pages={76664--76702},
  year={2026}
}

@article{baradat2020minimizing,
  title={Minimizing relative entropy of path measures under marginal constraints},
  author={Baradat, Aymeric and L{\'e}onard, Christian},
  journal={arXiv preprint arXiv:2001.10920},
  year={2020}
}

@article{huang2026art,
  title={ART for Diffusion Sampling: A Reinforcement Learning Approach to Timestep Schedule},
  author={Huang, Yilie and Tang, Wenpin and Zhou, Xunyu},
  journal={arXiv preprint arXiv:2601.18681},
  year={2026}
}

@article{he2025rne,
  title={RNE: plug-and-play diffusion inference-time control and energy-based training},
  author={He, Jiajun and Hern{\'a}ndez-Lobato, Jos{\'e} Miguel and Du, Yuanqi and Vargas, Francisco},
  journal={arXiv preprint arXiv:2506.05668},
  year={2025}
}

@article{koulischer2024dynamic,
  title={Dynamic negative guidance of diffusion models},
  author={Koulischer, Felix and Deleu, Johannes and Raya, Gabriel and Demeester, Thomas and Ambrogioni, Luca},
  journal={arXiv preprint arXiv:2410.14398},
  year={2024}
}

@article{kim2025training,
  title={Training-free safe denoisers for safe use of diffusion models},
  author={Kim, Mingyu and Kim, Dongjun and Yusuf, Amman and Ermon, Stefano and Park, Mijung},
  journal={arXiv preprint arXiv:2502.08011},
  year={2025}
}

@inproceedings{schramowski2023safe,
  title={Safe latent diffusion: Mitigating inappropriate degeneration in diffusion models},
  author={Schramowski, Patrick and Brack, Manuel and Deiseroth, Bj{\"o}rn and Kersting, Kristian},
  booktitle={Proceedings of the IEEE/CVF conference on computer vision and pattern recognition},
  pages={22522--22531},
  year={2023}
}

@article{delmoral2006sequential,
  title   = {Sequential Monte Carlo Samplers},
  author  = {Del Moral, Pierre and Doucet, Arnaud and Jasra, Ajay},
  journal = {Journal of the Royal Statistical Society: Series B (Statistical Methodology)},
  volume  = {68},
  number  = {3},
  pages   = {411--436},
  year    = {2006},
  doi     = {10.1111/j.1467-9868.2006.00553.x}
}

@inproceedings{skreta2025feynman,
  title     = {Feynman-Kac Correctors in Diffusion: Annealing, Guidance, and Product of Experts},
  author    = {Skreta, Marta and Akhound-Sadegh, Tara and Ohanesian, Viktor and Bondesan, Roberto and Aspuru-Guzik, Alan and Doucet, Arnaud and Brekelmans, Rob and Tong, Alexander and Neklyudov, Kirill},
  booktitle = {Proceedings of the 42nd International Conference on Machine Learning},
  series    = {Proceedings of Machine Learning Research},
  volume    = {267},
  pages     = {55906--55949},
  publisher = {PMLR},
  year      = {2025},
  url       = {https://proceedings.mlr.press/v267/skreta25a.html}
}

@inproceedings{sohl2015deep,
  title     = {Deep Unsupervised Learning using Nonequilibrium Thermodynamics},
  author    = {Sohl-Dickstein, Jascha and Weiss, Eric A. and Maheswaranathan, Niru and Ganguli, Surya},
  booktitle = {Proceedings of the 32nd International Conference on Machine Learning},
  pages     = {2256--2265},
  year      = {2015},
  series    = {Proceedings of Machine Learning Research},
  volume    = {37},
  publisher = {PMLR},
  url       = {https://proceedings.mlr.press/v37/sohl-dickstein15.html}
}

@inproceedings{ho2020denoising,
  title     = {Denoising Diffusion Probabilistic Models},
  author    = {Ho, Jonathan and Jain, Ajay and Abbeel, Pieter},
  booktitle = {Advances in Neural Information Processing Systems},
  volume    = {33},
  pages     = {6840--6851},
  year      = {2020}
}

@inproceedings{song2021score,
  title     = {Score-Based Generative Modeling through Stochastic Differential Equations},
  author    = {Song, Yang and Sohl-Dickstein, Jascha and Kingma, Diederik P. and Kumar, Abhishek and Ermon, Stefano and Poole, Ben},
  booktitle = {International Conference on Learning Representations},
  year      = {2021},
  url       = {https://openreview.net/forum?id=PxTIG12RRHS}
}

@inproceedings{dhariwal2021diffusion,
  title     = {Diffusion Models Beat {GANs} on Image Synthesis},
  author    = {Dhariwal, Prafulla and Nichol, Alexander Quinn},
  booktitle = {Advances in Neural Information Processing Systems},
  volume    = {34},
  pages     = {8780--8794},
  year      = {2021}
}

@inproceedings{ho2022classifierfree,
  title     = {Classifier-Free Diffusion Guidance},
  author    = {Ho, Jonathan and Salimans, Tim},
  booktitle = {NeurIPS Workshop on Deep Generative Models and Downstream Applications},
  year      = {2021},
  url       = {https://openreview.net/forum?id=qw8AKxfYbI}
}

@inproceedings{song2020denoising,
  title     = {Denoising Diffusion Implicit Models},
  author    = {Song, Jiaming and Meng, Chenlin and Ermon, Stefano},
  booktitle = {International Conference on Learning Representations},
  year      = {2021},
  url       = {https://openreview.net/forum?id=St1giarCHLP}
}

@inproceedings{karras2022elucidating,
  title     = {Elucidating the Design Space of Diffusion-Based Generative Models},
  author    = {Karras, Tero and Aittala, Miika and Aila, Timo and Laine, Samuli},
  booktitle = {Advances in Neural Information Processing Systems},
  volume    = {35},
  year      = {2022}
}

@article{armandpour2023perpneg,
  title   = {Re-Imagine the Negative Prompt Algorithm: Transform 2D Diffusion into 3D, Alleviate Janus Problem and Beyond},
  author  = {Armandpour, Mohammadreza and Sadeghian, Ali and Zheng, Huangjie and Sadeghian, Amir and Zhou, Mingyuan},
  journal = {arXiv preprint arXiv:2304.04968},
  year    = {2023}
}

@book{doucet2001sequential,
  title     = {Sequential Monte Carlo Methods in Practice},
  editor    = {Doucet, Arnaud and de Freitas, Nando and Gordon, Neil},
  publisher = {Springer},
  year      = {2001},
  doi       = {10.1007/978-1-4757-3437-9}
}

@inproceedings{kim2025testtime,
  title     = {Test-Time Alignment of Diffusion Models without Reward Over-Optimization},
  author    = {Kim, Sunwoo and Kim, Minkyu and Park, Dongmin},
  booktitle = {International Conference on Learning Representations},
  year      = {2025},
  url       = {https://openreview.net/forum?id=vi3DjUhFVm}
}

@inproceedings{kim2025tuning,
  title     = {Tuning Sequential {M}onte {C}arlo Samplers via Greedy Incremental Divergence Minimization},
  author    = {Kim, Kyurae and Xu, Zuheng and Gardner, Jacob R. and Campbell, Trevor},
  booktitle = {Proceedings of the 42nd International Conference on Machine Learning},
  pages     = {30826--30878},
  year      = {2025},
  editor    = {Singh, Aarti and Fazel, Maryam and Hsu, Daniel and Lacoste-Julien, Simon and Berkenkamp, Felix and Maharaj, Tegan and Wagstaff, Kiri and Zhu, Jerry},
  volume    = {267},
  series    = {Proceedings of Machine Learning Research},
  publisher = {PMLR},
  url       = {https://proceedings.mlr.press/v267/kim25ak.html}
}

@article{ou2026inferencetime,
  title   = {Inference-Time Scaling of Discrete Diffusion Models via Importance Weighting and Optimal Proposal Design},
  author  = {Ou, Zijing and Pani, Chinmay and Li, Yingzhen},
  journal = {arXiv preprint arXiv:2505.22524},
  year    = {2025}
}

@article{berner2024optimal,
  title   = {An Optimal Control Perspective on Diffusion-Based Generative Modeling},
  author  = {Berner, Julius and Richter, Lorenz and Ullrich, Karen},
  journal = {Transactions on Machine Learning Research},
  year    = {2024},
  url     = {https://openreview.net/forum?id=oYIjw37pTP}
}

@article{neal2001annealed,
  title   = {Annealed Importance Sampling},
  author  = {Neal, Radford M.},
  journal = {Statistics and Computing},
  volume  = {11},
  pages   = {125--139},
  year    = {2001}
}

@misc{kynkaanniemi2019improved,
      title={Improved Precision and Recall Metric for Assessing Generative Models}, 
      author={Tuomas Kynkäänniemi and Tero Karras and Samuli Laine and Jaakko Lehtinen and Timo Aila},
      year={2019},
      eprint={1904.06991},
      archivePrefix={arXiv},
      primaryClass={stat.ML},
      url={https://arxiv.org/abs/1904.06991}, 
}

@misc{radford2021learning,
      title={Learning Transferable Visual Models From Natural Language Supervision}, 
      author={Alec Radford and Jong Wook Kim and Chris Hallacy and Aditya Ramesh and Gabriel Goh and Sandhini Agarwal and Girish Sastry and Amanda Askell and Pamela Mishkin and Jack Clark and Gretchen Krueger and Ilya Sutskever},
      year={2021},
      eprint={2103.00020},
      archivePrefix={arXiv},
      primaryClass={cs.CV},
      url={https://arxiv.org/abs/2103.00020}, 
}

@inproceedings{rombach2022high,
  title={High-resolution image synthesis with latent diffusion models},
  author={Rombach, Robin and Blattmann, Andreas and Lorenz, Dominik and Esser, Patrick and Ommer, Bj{\"o}rn},
  booktitle={Proceedings of the IEEE/CVF conference on computer vision and pattern recognition},
  pages={10684--10695},
  year={2022}
}

@article{naesseth2019elements,
  title={Elements of sequential monte carlo},
  author={Naesseth, Christian A and Lindsten, Fredrik and Sch{\"o}n, Thomas B},
  journal={Foundations and Trends{\textregistered} in Machine Learning},
  volume={12},
  number={3},
  pages={187--306},
  year={2019},
  publisher={Emerald Publishing Limited}
}

@article{uehara2025inference,
  title={Inference-time alignment in diffusion models with reward-guided generation: Tutorial and review},
  author={Uehara, Masatoshi and Zhao, Yulai and Wang, Chenyu and Li, Xiner and Regev, Aviv and Levine, Sergey and Biancalani, Tommaso},
  journal={arXiv preprint arXiv:2501.09685},
  year={2025}
}

@inproceedings{liu2023more,
  title={More control for free! image synthesis with semantic diffusion guidance},
  author={Liu, Xihui and Park, Dong Huk and Azadi, Samaneh and Zhang, Gong and Chopikyan, Arman and Hu, Yuxiao and Shi, Humphrey and Rohrbach, Anna and Darrell, Trevor},
  booktitle={Proceedings of the IEEE/CVF winter conference on applications of computer vision},
  pages={289--299},
  year={2023}
}

@article{kim2025comprehensive,
  title={A comprehensive survey on concept erasure in text-to-image diffusion models},
  author={Kim, Changhoon and Qi, Yanjun},
  journal={arXiv preprint arXiv:2502.14896},
  year={2025}
}

@article{armandpour2023reimagine,
  title={Re-imagine the Negative Prompt Algorithm: Transform 2D Diffusion into 3D, Alleviate Janus Problem and Beyond},
  author={Armandpour, Mohammadreza and Sadeghian, Ali and Zheng, Huangjie and Sadeghian, Amir and Zhou, Mingyuan},
  journal={arXiv preprint arXiv:2304.04968},
  year={2023}
}

@article{lee2024contrastivecfg,
  title={Contrastive CFG: Improving CFG in Diffusion Models by Contrasting Positive and Negative Concepts},
  author={Lee, Minwoo and others},
  journal={arXiv preprint arXiv:2411.17077},
  year={2024}
}

@inproceedings{dinh2023rethinking,
  title={Rethinking Conditional Diffusion Sampling with Progressive Guidance},
  author={Dinh, Anh-Dung and Liu, Daochang and Xu, Chang},
  booktitle={Advances in Neural Information Processing Systems},
  year={2023}
}

@inproceedings{chung2025cfgpp,
  title={{CFG++}: Manifold-constrained Classifier Free Guidance for Diffusion Models},
  author={Chung, Hyungjin and Kim, Jeongsol and Park, Geon Yeong and Nam, Hyelin and Ye, Jong Chul},
  booktitle={International Conference on Learning Representations},
  year={2025}
}

@inproceedings{liu2022compositional,
  title={Compositional visual generation with composable diffusion models},
  author={Liu, Nan and Li, Shuang and Du, Yilun and Torralba, Antonio and Tenenbaum, Joshua B},
  booktitle={European conference on computer vision},
  pages={423--439},
  year={2022},
  organization={Springer}
}

@article{lipman2022flow,
  title={Flow matching for generative modeling},
  author={Lipman, Yaron and Chen, Ricky TQ and Ben-Hamu, Heli and Nickel, Maximilian and Le, Matt},
  journal={arXiv preprint arXiv:2210.02747},
  year={2022}
}

\newpage
\clearpage
\appendix

\startcontents[appendices]
\printcontents[appendices]{}{1}{
    \section*{Appendix Contents}
    \setcounter{tocdepth}{2}
}

\clearpage

\section{Target Distribution Design}
\label{sec:target}

\subsection{Proof and Interpretation of the Target Distribution}
\label{soften}

\propone*

\begin{proof}
We first prove the hard-target result. For fixed $x$, write
\[
    a=p^A(x),\qquad
    b=p^B(x),\qquad
    t=q(x),
\]
and define
\[
    F_x(t)
    =
    \min\{t,a\}
    -
    \gamma\min\{t,b\}.
\]
Then
\[
    J(q)=\int F_x(q(x))\,dx.
\]

For $\gamma>1$, direct inspection gives
\[
    \sup_{t\geq0}F_x(t)
    =
    (a-\gamma b)_+.
\]
When $a>\gamma b$, the supremum is attained for $t\geq a$; when
$a<\gamma b$, it is attained only at $t=0$. Hence
\[
    J(q)
    \leq
    \int_{\mathcal A_\gamma}
    \bigl(p^A(x)-\gamma p^B(x)\bigr)\,dx,
\]
where
\[
    \mathcal A_\gamma
    =
    \{x:p^A(x)>\gamma p^B(x)\}.
\]

Let
\[
    Z_A
    =
    \int_{\mathcal A_\gamma}p^A(x)\,dx.
\]
Since $Z_A>0$,
\[
    q^*(x)
    =
    \frac{
        p^A(x)\mathbf 1_{\mathcal A_\gamma}(x)
    }{
        Z_A
    }
\]
satisfies $q^*(x)\geq p^A(x)$ on $\mathcal A_\gamma$ and vanishes
elsewhere, so it attains the upper bound and maximizes $J$.

Now consider any maximizer $q$ with
$D_{\mathrm{KL}}(q\|p^A)<\infty$. Pointwise optimality and the
zero-$p^A$-mass boundary assumption imply that
\[
    q(x)
    =
    p^A(x)r(x)\mathbf 1_{\mathcal A_\gamma}(x),
\]
where
\[
    r(x)\geq1,
    \qquad
    \int_{\mathcal A_\gamma}p^A(x)r(x)\,dx=1.
\]
Therefore,
\[
    D_{\mathrm{KL}}(q\|p^A)
    =
    \int_{\mathcal A_\gamma}
    p^A(x)r(x)\log r(x)\,dx.
\]

Define
\[
    d\nu(x)
    =
    \frac{
        p^A(x)\mathbf 1_{\mathcal A_\gamma}(x)
    }{
        Z_A
    }\,dx.
\]
Then
\[
    \mathbb E_\nu[r]=\frac{1}{Z_A}.
\]
Since $u\log u$ is strictly convex, Jensen's inequality yields
\[
    D_{\mathrm{KL}}(q\|p^A)
    =
    Z_A\mathbb E_\nu[r\log r]
    \geq
    \log\frac{1}{Z_A},
\]
with equality if and only if
\[
    r(x)=\frac{1}{Z_A}
\]
almost everywhere. Thus $q^*$ is the unique KL-minimizing maximizer.

We next derive the soft relaxation. Define
\[
    s(x)
    =
    \log\frac{p^A(x)}{p^B(x)}
    -
    \log\gamma.
\]
The hard retention decision is
\[
    \mathbf 1\{s(x)>0\}.
\]
Let $\pi(x)\in[0,1]$ denote a soft retention probability and consider
\begin{equation}
    \max_{\pi\in[0,1]}
    \left\{
        \pi s(x)
        +
        \frac{1}{\eta}H_{\mathrm B}(\pi)
    \right\},
    \qquad
    \eta>0,
    \label{eq:pointwise_entropy_gate}
\end{equation}
where
\[
    H_{\mathrm B}(\pi)
    =
    -\pi\log\pi-(1-\pi)\log(1-\pi).
\]
The objective is strictly concave, and its first-order condition gives
\[
    \log\frac{\pi}{1-\pi}
    =
    \eta s(x).
\]
Hence the unique optimizer is
\begin{equation}
    r_\eta(x)
    =
    \sigma(\eta s(x))
    =
    \frac{
        p^A(x)^\eta
    }{
        p^A(x)^\eta+\gamma^\eta p^B(x)^\eta
    }.
    \label{eq:entropy_optimal_gate}
\end{equation}
Thus $\eta$ controls the sharpness of the likelihood-ratio boundary
\citep{blondel2020learning}.

Our powered target is
\begin{equation}
    q_{\eta,m}(x)
    =
    \frac{
        p^A(x)r_\eta(x)^m
    }{
        Z_{\eta,m}
    },
    \qquad
    Z_{\eta,m}
    =
    \int p^A(u)r_\eta(u)^m\,du.
    \label{eq:powered_soft_target}
\end{equation}
The exponent $m$ controls the strength of the reweighting without
changing the likelihood-ratio boundary.

To characterize this target variationally, observe that
\[
    \log\frac{q_{\eta,m}(x)}{p^A(x)}
    =
    m\log r_\eta(x)
    -
    \log Z_{\eta,m}.
\]
Therefore, for any $q\ll p^A$,
\[
\begin{aligned}
    D_{\mathrm{KL}}(q\|q_{\eta,m})
    &=
    D_{\mathrm{KL}}(q\|p^A)
    -
    m\mathbb E_q[\log r_\eta(X)]
    +
    \log Z_{\eta,m}.
\end{aligned}
\]
Equivalently,
\[
    \mathbb E_q[\log r_\eta(X)]
    -
    \frac{1}{m}D_{\mathrm{KL}}(q\|p^A)
    =
    \frac{1}{m}\log Z_{\eta,m}
    -
    \frac{1}{m}D_{\mathrm{KL}}(q\|q_{\eta,m}).
\]
Since KL divergence is nonnegative, $q_{\eta,m}$ is the unique solution
of
\begin{equation}
    \max_{q\ll p^A}
    \left\{
        \mathbb E_q[\log r_\eta(X)]
        -
        \frac{1}{m}D_{\mathrm{KL}}(q\|p^A)
    \right\}.
    \label{eq:powered_gate_variational}
\end{equation}
Thus $m$ controls the trade-off between retention confidence and
proximity to $p^A$.

Finally, for every fixed $m>0$ and $s(x)\neq0$,
\[
    r_\eta(x)^m
    =
    \sigma(\eta s(x))^m
    \longrightarrow
    \mathbf 1\{s(x)>0\}
\]
as $\eta\to\infty$. Since $0\leq r_\eta(x)^m\leq1$, dominated
convergence gives
\[
    Z_{\eta,m}
    \longrightarrow
    \int p^A(x)\mathbf 1\{s(x)>0\}\,dx.
\]
Under the assumptions of Proposition~\ref{prop:1},
\[
    q_{\eta,m}
    \longrightarrow
    q^*
\]
in total variation. Hence every fixed $m>0$ preserves the same hard
target.
\end{proof}

\subsection{Relation to Existing Negative Guidance}
\label{app:related-targets}

\paragraph{Dynamic Negative Guidance.}

Dynamic Negative Guidance (DNG) \citep{koulischer2024dynamic} uses the
negative-class posterior. With
\[
    R(x)
    =
    \frac{p^B(x)}{p^A(x)},
    \qquad
    c=\frac{p(B)}{p(A)},
\]
and DNG exponent $\rho>0$, its reweighting is
\[
    g_{\mathrm{DNG}}(x)
    =
    \bigl(1+cR(x)\bigr)^{-\rho}.
\]
Our powered gate is
\begin{equation}
    g_{\mathrm{ours}}(x)
    =
    r_\eta(x)^m
    =
    \left[
        1+\bigl(\gamma R(x)\bigr)^\eta
    \right]^{-m}.
    \label{eq:ours_ratio_weight}
\end{equation}

After matching the threshold and power,
\[
    c=\gamma,
    \qquad
    \rho=m,
\]
DNG is exactly the special case $\eta=1$:
\[
    g_{\mathrm{DNG}}(x)
    =
    g_{\mathrm{ours}}(x)
    \qquad
    \text{when }\eta=1.
\]
Our formulation therefore separates two effects that are coupled in
DNG: $\eta$ controls the sharpness of the likelihood-ratio transition,
whereas $m$ controls the strength of the reweighting.

For $\eta>1$, writing $u=\gamma R(x)$ gives
\[
    u<1
    \;\Longrightarrow\;
    g_{\mathrm{ours}}>g_{\mathrm{DNG}},
    \qquad
    u>1
    \;\Longrightarrow\;
    g_{\mathrm{ours}}<g_{\mathrm{DNG}}.
\]
Thus the proposed gate preserves the same boundary $u=1$ while producing
a sharper distinction between positive- and negative-dominant regions.

\paragraph{Safe Denoiser.}

Safe Denoiser \citep{kim2025training} also modifies the base score along
the contrastive direction
\[
    s_t^A(x)-s_t^B(x),
\]
but uses an unsafe-density-driven coefficient.

For our target,
\[
    q_t^{\eta,m}(x)
    \propto
    p_t^A(x)
    \bigl(1-\theta_t(x)\bigr)^m,
\]
where
\[
    \theta_t(x)
    =
    \frac{
        \gamma^{\eta_t}p_t^B(x)^{\eta_t}
    }{
        p_t^A(x)^{\eta_t}
        +
        \gamma^{\eta_t}p_t^B(x)^{\eta_t}
    }.
\]
Differentiating gives
\begin{equation}
    \nabla_x\log q_t^{\eta,m}(x)
    =
    s_t^A(x)
    +
    m\eta_t\theta_t(x)
    \bigl(
        s_t^A(x)-s_t^B(x)
    \bigr).
    \label{eq:powered_target_score}
\end{equation}

Hence both methods use the same contrastive score direction, but weight
it differently. Safe Denoiser uses a weight based primarily on the
absolute unsafe density, whereas our coefficient
\[
    m\eta_t\theta_t(x)
\]
is determined by the relative likelihood of the positive and negative
distributions. The parameter $\eta_t$ controls where the correction
becomes active, while $m$ controls its overall strength.

\section{Main Algorithm Derivation}

We use the notation of Section~\ref{sec:sampler} and write
$\boldsymbol\rho=(\rho_1,\rho_2)$. The operational Gaussian weight is
$\omega_{\boldsymbol\rho}^{\mathrm G}$, its fixed-displacement approximation
is $\widehat\omega_{\boldsymbol\rho}$, and the ideal normalized RN weight is
denoted by $\omega_{\boldsymbol\rho}^{\mathrm{RN}}$ when needed. These
objects are kept separate throughout.

\subsection{Local Gaussian Weight and Proposal Expansion}
\label{prop2proof}

\paragraph{Setup and regularity.}
Let $\tau=t+h$, $\Delta y=x-y$, and $\boldsymbol\rho$ range over a fixed
finite candidate set. Each candidate is held fixed on $[t,\tau]$.
The reference and proposal fields are smooth in time and space, with
polynomially bounded derivatives through the orders used below
(two in time and three in space suffice). The schedules $s,c,\eta,m$
and their reciprocals are bounded on this interval, and the gate logits
have the required moments. We assume sufficiently high moments of the
current state and of the local fields so that Taylor remainders are
$O(h^{3/2})$ in $L^2$, uniformly over candidates. For fixed states, the
corresponding statements hold conditionally on those states. These
assumptions concern stochastic steps with $s>0$. Here
$O_{L^2}(h^a)$ denotes a remainder with $L^2$ norm $O(h^a)$.

Unless shown otherwise, quantities are evaluated at $(\tau,y)$. Write
\[
    d=d_\tau(y),\quad g=\zeta d,\quad
    \zeta=m_\tau\eta_\tau\theta,\quad
    \chi=m_\tau\eta_\tau^2\theta(1-\theta),\quad
    u=\rho_1d+\rho_2g,\quad v=u-g.
\]
The local score-based derivatives are
\begin{equation}
    \nabla\theta=-\eta_\tau\theta(1-\theta)d,
    \qquad
    Jg=\zeta Jd-\chi dd^\top,
    \qquad
    Ju=(\rho_1+\rho_2\zeta)Jd-\rho_2\chi dd^\top.
    \label{eq:app-local-field-derivatives}
\end{equation}
For the exact gate, these follow by differentiation. With an estimated
current gate, they specify the smooth local extension used in the proposal
expansion; they do not differentiate an estimator through its sampling
history. A different derivative convention requires a corresponding
change to the local expansion.

\paragraph{Gaussian gate contribution.}
The current gate value $\theta\in(0,1)$ is supplied by the gate evaluation.
Because the reference forward kernels coincide, their contributions
cancel in the A/B contrast. For the Gaussian kernels in
\eqref{eq:local-gaussian-kernels}, define
\begin{equation}
\begin{aligned}
    D^{\mathrm G}(x,y)
    &:=\log\frac{L_\tau^{\nu^B}(x\mid y)}
                    {L_\tau^{\nu^A}(x\mid y)}
      =h\alpha_{AB}-d^\top\Delta y,\\
    \alpha_{AB}
    &:=\frac{\|\nu_\tau^A(y)\|^2-\|\nu_\tau^B(y)\|^2}{2s_\tau}.
\end{aligned}
\label{eq:app-gaussian-contrast}
\end{equation}
The equality is exact for these equal-covariance Gaussian kernels: it
follows by expanding the difference of their squared residuals.

The resulting change in the gate log odds is
\begin{equation}
    S^{\mathrm G}
    :=\log\frac{c_t}{c_\tau}
    +\left(\frac{\eta_t}{\eta_\tau}-1\right)
       \left(\log\frac{\theta}{1-\theta}-\log c_\tau\right)
    +\eta_tD^{\mathrm G}.
    \label{eq:app-gate-logit-increment}
\end{equation}
Multiplying the current gate odds by $e^{S^{\mathrm G}}$ gives the
one-step gate contribution used in the main text:
\begin{equation}
    \Delta\mathcal R^{\mathrm G}(x,y)
    :=-m_t\log\left[1+\theta(e^{S^{\mathrm G}}-1)\right]
      +(m_t-m_\tau)\log(1-\theta).
    \label{eq:app-gaussian-gate-contribution}
\end{equation}
For fixed $m$, the last term vanishes. This defines a local weight
contribution conditioned on the current gate; it does not prescribe a
recursive estimator for later gate values.

Expanding the schedules at $\tau$ gives
\begin{equation}
\begin{aligned}
    S^{\mathrm G}
    &=h\alpha_S-\eta_\tau d^\top\Delta y+O_{L^2}(h^{3/2}),\\
    \alpha_S
    &:=\eta_\tau\alpha_{AB}-\partial_\tau\log c_\tau
    -\frac{\dot\eta_\tau}{\eta_\tau}
       \left(\log\frac{\theta}{1-\theta}-\log c_\tau\right).
\end{aligned}
\label{eq:app-gate-scalar-coefficients}
\end{equation}
Using
$\log[1+\theta(e^S-1)]
=\theta S+\tfrac12\theta(1-\theta)S^2+O(S^3)$, we obtain
\begin{equation}
\begin{aligned}
    \Delta\mathcal R^{\mathrm G}
    &=h\alpha_{\mathcal R}+g^\top\Delta y
      -\frac{\chi}{2}(d^\top\Delta y)^2
      +O_{L^2}(h^{3/2}),\\
    \alpha_{\mathcal R}
    &:=-m_\tau\theta\alpha_S-\dot m_\tau\log(1-\theta).
\end{aligned}
\label{eq:app-gate-local-expansion}
\end{equation}
Setting $\dot m_\tau=0$ recovers the fixed-$m$ coefficient.

\begin{proof}[Proof of Proposition~\ref{prop2}]
Let $a=a_{\boldsymbol\rho}(y)$ and define the kernel contribution
\[
    \mathcal K_{\boldsymbol\rho}
    :=\log\frac{L_\tau^{\nu^A}}{K_t^{\mu^A}}
      -\log\frac{L_\tau^a}{K_t^{b^{\boldsymbol\rho}}},
\]
with all kernels evaluated at $(x,y)$. Their common covariance terms
cancel. Since
$b_t^{\boldsymbol\rho}(x)=\mu_t^A(x)+s_tu_t^{\boldsymbol\rho}(x)$,
exact Gaussian algebra gives
\begin{align}
    \mathcal K_{\boldsymbol\rho}
    &=\left[\frac{a-\nu_\tau^A(y)}{s_\tau}
             -u_t^{\boldsymbol\rho}(x)\right]^\top\Delta y
    \nonumber\\
    &\quad+h\left[
       -\mu_t^A(x)^\top u_t^{\boldsymbol\rho}(x)
       -\frac{s_t}{2}\|u_t^{\boldsymbol\rho}(x)\|^2
       +\frac{\|a\|^2-\|\nu_\tau^A(y)\|^2}{2s_\tau}
    \right].
    \label{eq:app-exact-kernel-difference}
\end{align}
Under the backward Gaussian proposal, $\Delta y=O_p(h^{1/2})$. Expanding
$u_t^{\boldsymbol\rho}(y+\Delta y)
=u+Ju\,\Delta y+O(h+\|\Delta y\|^2)$ therefore yields
\begin{equation}
    \mathcal K_{\boldsymbol\rho}
    =h\alpha_{\mathrm{ker}}
     +\left[-u+\frac{a-\nu_\tau^A}{s_\tau}\right]^\top\Delta y
     -\Delta y^\top Ju\,\Delta y
     +O_{L^2}(h^{3/2}),
    \label{eq:app-kernel-expansion}
\end{equation}
where
\[
    \alpha_{\mathrm{ker}}
    =-(\mu_t^A)^\top u-\frac{s_t}{2}\|u\|^2
     +\frac{\|a\|^2-\|\nu_\tau^A\|^2}{2s_\tau}.
\]
Adding \eqref{eq:app-gate-local-expansion}, the linear coefficient becomes
\[
    g-u+\frac{a-\nu_\tau^A}{s_\tau}.
\]
It vanishes for $a=a_{\boldsymbol\rho}^\star=\nu_\tau^A+s_\tau v$.
Moreover,
\[
    \frac{\|a_{\boldsymbol\rho}^\star\|^2-\|\nu_\tau^A\|^2}{2s_\tau}
    =(\nu_\tau^A)^\top v+\frac{s_\tau}{2}\|v\|^2.
\]
Thus the remaining coefficients are
\begin{equation}
\begin{aligned}
    \alpha_{\boldsymbol\rho}
    &=\alpha_{\mathcal R}-(\mu_t^A)^\top u
      -\frac{s_t}{2}\|u\|^2
      +(\nu_\tau^A)^\top v+\frac{s_\tau}{2}\|v\|^2,\\
    \Gamma_{\boldsymbol\rho}
    &=-Ju-\frac{\chi}{2}dd^\top
      =-(\rho_1+\rho_2\zeta)Jd
       +\chi\left(\rho_2-\frac12\right)dd^\top.
\end{aligned}
\label{eq:app-matched-coefficients}
\end{equation}
These are \eqref{eq:app-matched-coefficients} and \eqref{eq:two-field-gamma}.
Consequently,
\[
    \log\omega_{\boldsymbol\rho}^{\mathrm G}
    =h\alpha_{\boldsymbol\rho}
     +\Delta y^\top\Gamma_{\boldsymbol\rho}\Delta y
     +O_{L^2}(h^{3/2}).
\]
The retained terms are $O_p(h)$, and the remainder is uniform over the
fixed finite candidate set.
\end{proof}

\paragraph{Fixed-displacement evaluation and same-noise selection.}
For fixed $\rho_1$, let $\Delta y_0$ be defined by
\eqref{eq:fixed-displacement}. Then
$\Delta y_{\rho_2}=\Delta y_0+\delta_{\rho_2}$, where
$\delta_{\rho_2}=-hs_\tau\rho_2g=O(h)$. Hence
\begin{align}
    &\Delta y_{\rho_2}^\top\Gamma_{\boldsymbol\rho}\Delta y_{\rho_2}
     -\Delta y_0^\top\Gamma_{\boldsymbol\rho}\Delta y_0
    \nonumber\\
    &\quad=\Delta y_0^\top\Gamma_{\boldsymbol\rho}\delta_{\rho_2}
          +\delta_{\rho_2}^\top\Gamma_{\boldsymbol\rho}\Delta y_0
          +\delta_{\rho_2}^\top\Gamma_{\boldsymbol\rho}\delta_{\rho_2}
     =O_{L^2}(h^{3/2}).
    \label{eq:app-fixed-displacement-error}
\end{align}
All candidates therefore reuse $Jd[\Delta y_0]$, computed by subtracting
one score JVP per reference model. In particular,
\begin{equation}
    \left\|\max_{\rho_2\in\mathcal R_2}
    \left|
       \log\omega_{\rho_1,\rho_2}^{\mathrm G}(x(\rho_2),y)
       -\log\widehat\omega(\rho_2)
    \right|\right\|_{L^2}
    =O(h^{3/2}).
    \label{eq:app-uniform-candidate-error}
\end{equation}
The selected error is bounded by this maximum, so the same expansion
holds after choosing $\rho_2$ using the current noises.

This samplewise statement does not imply conditional trace unbiasedness.
If $\boldsymbol\rho^\star$ depends on $\xi$, then generally
\[
    \mathbb E\!\left[
      \xi^\top\Gamma_{\boldsymbol\rho^\star}\xi
       -\operatorname{tr}\Gamma_{\boldsymbol\rho^\star}
      \mid\text{current particle system}\right]\ne0.
\]
Likewise, conditioning on the selected coefficient changes the noise
law, so the actual selected mutation need not have the nominal Gaussian
candidate density. The algorithm uses adaptive surrogate weights, not an
exact RN correction for that selected law.

\paragraph{Distinction from an oracle endpoint gate.}
Under the exact-reference assumptions of
Section~\ref{sec:continuous_connection}, set
$H=\nabla^2\log(p_\tau^B/p_\tau^A)(y)=-Jd$. The log-density
Fokker--Planck equations and a Taylor expansion give
\begin{equation}
    \log\frac{p_t^B(x)}{p_t^A(x)}
    -\log\frac{p_\tau^B(y)}{p_\tau^A(y)}
    =D^{\mathrm G}(x,y)
     +\frac12\left[\Delta y^\top H\Delta y
                    -hs_\tau\operatorname{tr}H\right]
     +O_{L^2}(h^{3/2}).
    \label{eq:app-oracle-contrast-comparison}
\end{equation}
For an exact current gate, the corresponding reward difference is
\begin{equation}
    \mathcal R_t(x)-\mathcal R_\tau(y)
    -\Delta\mathcal R^{\mathrm G}(x,y)
    =\frac{\zeta}{2}\left[
        \Delta y^\top Jd\,\Delta y-hs_\tau\operatorname{tr}Jd
      \right]+O_{L^2}(h^{3/2}).
    \label{eq:app-oracle-gate-comparison}
\end{equation}
This centered curvature term distinguishes the oracle endpoint weight
from the defined Gaussian surrogate; it is not an omitted term in
Proposition~\ref{prop2}.

\subsection{Variance Objective and Forward Kernel Matching}
\label{app:forward-kernel-matching}
\subsubsection{Variance Objective}
Here we show that normalized cumulative ESS maximization is our variance minimization used in section \ref{sec:fixed-displacement-jvp}
\[
    \widehat{\mathrm{nESS}}_t(\rho_2)
    :=
    \frac{\bigl(\sum_iV_i(\rho_2)\bigr)^2}
         {N\sum_iV_i(\rho_2)^2}
    =
    \left[1+\widehat{\mathcal V}_t^{\mathrm{rel}}(\rho_2)\right]^{-1}.
\]
\subsubsection{Forward Kernel Matching}
This subsection concerns fixed candidates, exact reference scores and
gates, and $Y\sim q_\tau$ independent of the current Gaussian noise.
Write $q_r=\gamma_r/Z_r$, where $Z_r=\int\gamma_r(z)\,dz$.
The forward-propagated marginal is
\begin{equation}
    \bar q_\tau^{\,\boldsymbol\rho}(y)
    :=\int q_t(x)K_t^{b^{\boldsymbol\rho}}(y\mid x)\,dx.
    \label{eq:proposal-forward-marginal}
\end{equation}
Define the adjoint generator and normalized residual by
\begin{equation}
\begin{aligned}
    \mathcal L_{r,\boldsymbol\rho}^*f
    &:=-\nabla\cdot(b_r^{\boldsymbol\rho}f)
       +\frac{s_r}{2}\nabla\cdot\nabla f,\\
    R_{r,\boldsymbol\rho}
    &:=\frac{\partial_rq_r-\mathcal L_{r,\boldsymbol\rho}^*q_r}{q_r}.
\end{aligned}
\label{eq:forward-fp-residual}
\end{equation}
The corresponding unnormalized log-weight potential is
\begin{equation}
    A_{\boldsymbol\rho}(r,\cdot)
    :=-\frac{\partial_r\gamma_r-
                 \mathcal L_{r,\boldsymbol\rho}^*\gamma_r}{\gamma_r}
     =-R_{r,\boldsymbol\rho}-\partial_r\log Z_r.
    \label{eq:app-fk-potential}
\end{equation}
Under decay or no-flux boundary conditions,
$\mathbb E_{q_r}[R_{r,\boldsymbol\rho}]=0$. Thus $A_{\boldsymbol\rho}$ and
$R_{r,\boldsymbol\rho}$ have the same variance across states.

\paragraph{Conditional log-weight moments.}
The continuous connection proved in
Section~\ref{sec:continuous_connection} gives
\[
    \log\widehat\omega_{\boldsymbol\rho}(Y,\xi)
    =hA_{\boldsymbol\rho}(\tau,Y)
     +hs_\tau\left[
        \xi^\top\Gamma_{\boldsymbol\rho}(Y)\xi
        -\operatorname{tr}\Gamma_{\boldsymbol\rho}(Y)
      \right]+O_{L^2}(h^{3/2}).
\]
For a fixed matrix $M$ and standard Gaussian $\xi$,
$\mathbb E[\xi^\top M\xi]=\operatorname{tr}M$ and
$\operatorname{Var}(\xi^\top M\xi)
=2\|\operatorname{sym}M\|_F^2$, where
$\operatorname{sym}M=(M+M^\top)/2$. Therefore
\begin{equation}
\begin{aligned}
    \operatorname{Var}_Y\!\left(
      \mathbb E_\xi[\log\widehat\omega_{\boldsymbol\rho}\mid Y]
    \right)
    &=h^2\operatorname{Var}_{q_\tau}
       (A_{\boldsymbol\rho}(\tau,Y))+o(h^2),\\
    \mathbb E_Y\!\left[
      \operatorname{Var}_\xi(\log\widehat\omega_{\boldsymbol\rho}\mid Y)
    \right]
    &=2h^2s_\tau^2\mathbb E_{q_\tau}
       \|\operatorname{sym}\Gamma_{\boldsymbol\rho}(Y)\|_F^2+o(h^2).
\end{aligned}
\label{eq:app-logweight-variance-components}
\end{equation}
These are the two terms in
\eqref{eq:main-logweight-total-variance}. For exact scores,
$Jd$ and $\Gamma_{\boldsymbol\rho}$ are symmetric.

\paragraph{Local forward-marginal interpretation.}
A first-order density expansion gives
\begin{equation}
    \frac{\bar q_\tau^{\,\boldsymbol\rho}}{q_\tau}
    =1-hR_{\tau,\boldsymbol\rho}+o_{L^2(q_\tau)}(h).
    \label{eq:forward-ratio-expansion}
\end{equation}
For the KL expansion, assume positive density ratios and an integrable
logarithmic Taylor remainder. Precisely, for
$\delta_h=\bar q_\tau^{\,\boldsymbol\rho}/q_\tau-1$, assume
\[
    \mathbb E_{q_\tau}\left|
      -\log(1+\delta_h)+\delta_h-\tfrac12\delta_h^2
    \right|=o(h^2).
\]
Together with \eqref{eq:forward-ratio-expansion}, this states the required
density and tail control, rather than inferring it from pointwise
smoothness alone.

\begin{proposition}[Conditional-mean log-weight variance and forward KL]
\label{prop:conditional-mean-forward-kl}
Under the assumptions above,
\begin{equation}
\begin{aligned}
    \operatorname{Var}_Y\!\left(
      \mathbb E_\xi[\log\widehat\omega_{\boldsymbol\rho}\mid Y]
    \right)
    &=h^2\operatorname{Var}_{q_\tau}(R_{\tau,\boldsymbol\rho})+o(h^2)\\
    &=2\,\mathrm{KL}
       (q_\tau\|\bar q_\tau^{\,\boldsymbol\rho})+o(h^2).
\end{aligned}
\label{eq:conditional-mean-forward-kl}
\end{equation}
\end{proposition}
\begin{proof}
The first equality follows from
\eqref{eq:app-logweight-variance-components} and
\eqref{eq:app-fk-potential}. Both marginals are normalized, so
$\mathbb E_{q_\tau}[\delta_h]=0$. The assumed Taylor control then yields
\[
    \mathrm{KL}(q_\tau\|\bar q_\tau^{\,\boldsymbol\rho})
    =\tfrac12\mathbb E_{q_\tau}[\delta_h^2]+o(h^2)
    =\frac{h^2}{2}\operatorname{Var}_{q_\tau}
       (R_{\tau,\boldsymbol\rho})+o(h^2).
\]
\end{proof}
Thus the conditional-mean component admits a local forward-matching
interpretation. The full log-weight variance also contains the
transition component in \eqref{eq:app-logweight-variance-components}.

\subsection{Connection to the Continuous Path-Space Formulation}
\label{sec:continuous_connection}

Assume positive, smooth reference densities satisfying the common forward
Fokker--Planck equation, exact gate values, and the moment conditions in
Section~\ref{prop2proof}. For each fixed candidate,
\[
    \partial_r p_r^j
    =-\nabla\cdot(\mu_r^jp_r^j)
     +\frac{s_r}{2}\nabla\cdot\nabla p_r^j,
    \qquad j\in\{A,B\},
\]
with $\mu^A=\mu^B$. Throughout this subsection the target and weights use
the unnormalized convention $\gamma_r=p_r^Ae^{\mathcal R_r}$.

\paragraph{Matched continuous process.}
The potential $A_{\boldsymbol\rho}$ defined in
\eqref{eq:app-fk-potential} follows from the product rule:
\begin{align}
    A_{\boldsymbol\rho}
    &=-\partial_r\mathcal R-(\nu^A)^\top g
      +\frac{s}{2}(\nabla\cdot g+\|g\|^2)
    \nonumber\\
    &\quad-s\left[
        \nabla\cdot u^{\boldsymbol\rho}
        +(u^{\boldsymbol\rho})^\top\nabla\log\gamma
      \right].
    \label{eq:two-field-relative-fp}
\end{align}
The time and state arguments on the right are $(r,z)$.
The matched drift satisfies
\[
    a_{\boldsymbol\rho}^\star
    =\nu^A+s(u^{\boldsymbol\rho}-g)
    =b^{\boldsymbol\rho}-s\nabla\log\gamma.
\]
Consequently,
\[
    \nabla\cdot(a_{\boldsymbol\rho}^\star\gamma)
    +\frac{s}{2}\nabla\cdot\nabla\gamma
    =-\mathcal L_{r,\boldsymbol\rho}^*\gamma,
\]
so $-\partial_r\gamma$ equals this reverse-time generator applied to
$\gamma$, plus $A_{\boldsymbol\rho}\gamma$.

In increasing reverse time $v\in[0,h]$, the corresponding process and
incremental FK weight are
\begin{equation}
\begin{aligned}
    dZ_v
    &=-a_{\boldsymbol\rho}^\star(\tau-v,Z_v)\,dv
      +\sqrt{s_{\tau-v}}\,dB_v,\qquad Z_0=y,\\
    \log\omega_{\boldsymbol\rho}^{\mathrm{cont}}
    &:=\int_0^h A_{\boldsymbol\rho}(\tau-v,Z_v)\,dv
      =hA_{\boldsymbol\rho}(\tau,y)+O_{L^2}(h^{3/2}).
\end{aligned}
\label{eq:continuous-log-weight}
\end{equation}
The last estimate uses the spatial and temporal regularity of
$A_{\boldsymbol\rho}$, with polynomial growth and sufficient moments,
and $\|Z_v-y\|_{L^p}=O(v^{1/2})$ for the required moments.
Normalizing the marginal path adds only the spatially constant potential
$\partial_r\log Z_r$.

\paragraph{Identification of the local coefficient.}
At $(\tau,y)$, the two reference log-density equations imply
\begin{equation}
    \alpha_{AB}
    =-\partial_\tau\log\frac{p_\tau^B}{p_\tau^A}
      +\frac{s_\tau}{2}\operatorname{tr}
         \nabla^2\log\frac{p_\tau^B}{p_\tau^A}.
    \label{eq:app-reference-time-identity}
\end{equation}
Substituting this into \eqref{eq:app-gate-scalar-coefficients} and
\eqref{eq:app-gate-local-expansion} gives
\begin{equation}
    \alpha_{\mathcal R}-\frac{s_\tau\chi}{2}\|d\|^2
    =-\partial_\tau\mathcal R_\tau
      +\frac{s_\tau}{2}\nabla\cdot g_\tau.
    \label{eq:gate-time-trace-identity}
\end{equation}
This includes the $\dot m_\tau$ term when $m$ varies.
Meanwhile, rearranging the scalar coefficient gives
\begin{equation}
    \alpha_{\boldsymbol\rho}
    =\alpha_{\mathcal R}-(\nu_\tau^A)^\top g
     +\frac{s_\tau}{2}\|g\|^2
     -s_\tau u^\top\nabla\log\gamma_\tau+O(h),
    \label{eq:finite-alpha-rearranged}
\end{equation}
and $s_\tau\operatorname{tr}\Gamma_{\boldsymbol\rho}
=-s_\tau\nabla\cdot u-(s_\tau\chi/2)\|d\|^2$.
Combining these expressions proves
\begin{equation}
    \alpha_{\boldsymbol\rho}
    +s_\tau\operatorname{tr}\Gamma_{\boldsymbol\rho}
    =A_{\boldsymbol\rho}(\tau,y)+O(h).
    \label{eq:finite-continuous-coefficient}
\end{equation}
The $O(h)$ term accounts for retaining $\mu_t^A$ and $s_t$ in the finite
scalar coefficient; it vanishes if those coefficients are evaluated at
$\tau$ instead.

Finally,
$\Delta y^\top\Gamma_{\boldsymbol\rho}\Delta y
=hs_\tau\xi^\top\Gamma_{\boldsymbol\rho}\xi+O_{L^2}(h^{3/2})$.
Together with \eqref{eq:continuous-log-weight}, this proves
\eqref{eq:finite-continuous-link}. The quadratic correction is
conditionally centered for a fixed candidate. This establishes the
continuous conditional-mean interpretation, not pathwise equality or
unbiasedness after same-noise selection.

\subsection{Connection to Previous Variance-Minimizing Methods}

\subsubsection{Comparison with DAS}
\label{app:comparison-das}

For this comparison only, consider the ideal normalized weight
\begin{equation}
    \omega_{\boldsymbol\rho}^{\mathrm{RN}}(x,y)
    :=\frac{q_t(x)K_t^{b^{\boldsymbol\rho}}(y\mid x)}
            {q_\tau(y)L_\tau^a(x\mid y)},
    \label{eq:app-ideal-rn-weight}
\end{equation}
where $Y\sim q_\tau$ and $X\mid Y\sim L_\tau^a$. Assume the required
support conditions and finite second moments. Direct integration gives
\begin{equation}
\begin{aligned}
    \mathbb E[\omega_{\boldsymbol\rho}^{\mathrm{RN}}\mid Y=y]
    &=\frac{\bar q_\tau^{\,\boldsymbol\rho}(y)}{q_\tau(y)},\\
    \operatorname{Var}_Y\!\left(
      \mathbb E[\omega_{\boldsymbol\rho}^{\mathrm{RN}}\mid Y]
    \right)
    &=\chi^2(\bar q_\tau^{\,\boldsymbol\rho}\|q_\tau).
\end{aligned}
\label{eq:conditional-mean-chi2}
\end{equation}
For fixed $K_t^{b^{\boldsymbol\rho}}$, the variance-minimizing backward
conditional is
\begin{equation}
    L_\tau^{\mathrm{opt}}(x\mid y)
    :=\frac{q_t(x)K_t^{b^{\boldsymbol\rho}}(y\mid x)}
            {\bar q_\tau^{\,\boldsymbol\rho}(y)}.
    \label{eq:app-optimal-backward}
\end{equation}
It makes $\omega^{\mathrm{RN}}$ constant in $x$ conditional on $y$, hence
removes the conditional-variance term. The forward-marginal term in
\eqref{eq:conditional-mean-chi2} is unchanged by the choice of $L$.
This is the ideal conditional-variance criterion underlying the proposal
component of DAS~\citep{kim2025testtime}; its Gaussian proposal uses a local
reward approximation. DAS also uses tempering to modify intermediate
targets.

Our coefficients change the auxiliary forward kernel as well as the
matched backward drift. The Gaussian-surrogate comparison is therefore
through \eqref{eq:app-logweight-variance-components} and
\eqref{eq:conditional-mean-forward-kl}. It does not assert exact
conditional-variance elimination or exact RN correction for the
implemented adaptive weights.
For clarity, even for the ideal RN weight,
\begin{equation}
    \mathbb E_L[\log\omega_{\boldsymbol\rho}^{\mathrm{RN}}\mid Y=y]
    =\log\frac{\bar q_\tau^{\,\boldsymbol\rho}(y)}{q_\tau(y)}
     -\mathrm{KL}\!\left(
        L_\tau^a(\cdot\mid y)\|L_\tau^{\mathrm{opt}}(\cdot\mid y)
      \right).
    \label{eq:app-rn-logmean}
\end{equation}
Thus the exact RN identity cannot be transferred unchanged to log weights;
our forward-KL relation for $\log\widehat\omega$ is a leading-order result.

\subsubsection{Comparison with DriftLite}
\label{app:comparison-driftlite}

DriftLite-VCG~\citep{ren2026driftlite} uses the FK compensation identity
\begin{equation}
    H_r(z;\mathfrak b)
    :=\nabla\cdot\mathfrak b(z)
      +\mathfrak b(z)^\top\nabla\log q_r(z),
    \qquad
    -\nabla\cdot(\mathfrak bq_r)+q_rH_r(\cdot;\mathfrak b)=0.
    \label{eq:driftlite-compensation}
\end{equation}
Adding $\mathfrak b$ to the sampling drift and $H_r$ to its potential
preserves the marginal path. VCG selects a finite-dimensional control by
minimizing the residual-potential variance.

In our increasing reverse-time convention,
\[
    -a_{\boldsymbol\rho}^\star
    =-\nu^A+sg-su^{\boldsymbol\rho}.
\]
Relative to $\boldsymbol\rho=\boldsymbol 0$, the added sampling control is
therefore $\mathfrak b^{\boldsymbol\rho}=-su^{\boldsymbol\rho}$.
Its compensated potential, under the same unnormalized convention, is
\begin{equation}
    \Phi_{\boldsymbol\rho,r}
    =A_{\boldsymbol 0}(r,\cdot)
     -s_r\left[
       \nabla\cdot u_r^{\boldsymbol\rho}
       +(u_r^{\boldsymbol\rho})^\top\nabla\log q_r
     \right]
    =A_{\boldsymbol\rho}(r,\cdot).
    \label{eq:driftlite-two-field-potential}
\end{equation}
Thus the continuous control families agree after matching the target,
fields, and time convention. Fixing $\rho_1$ leaves an affine
one-dimensional family in $\rho_2$.

The population VCG criterion concerns
$\operatorname{Var}_{q_\tau}(A_{\boldsymbol\rho})$. Our finite Gaussian
log-weight variance has this conditional-mean component and the additional
transition component in \eqref{eq:app-logweight-variance-components}.
The two population objectives can consequently have different minimizers.
The implemented rule instead maximizes predicted cumulative ESS using
common-noise candidate evaluations.

The full reference displacement also retains selected finite-step terms.
Writing $z=\sqrt{hs_\tau}\xi$ and $\Delta y_0=z-ha_0^\star$, for symmetric
$\Gamma_{\boldsymbol\rho}$ we have
\begin{equation}
    \Delta y_0^\top\Gamma_{\boldsymbol\rho}\Delta y_0
    =z^\top\Gamma_{\boldsymbol\rho}z
     -2h z^\top\Gamma_{\boldsymbol\rho}a_0^\star
     +h^2(a_0^\star)^\top\Gamma_{\boldsymbol\rho}a_0^\star.
    \label{eq:driftlite-displacement-decomposition}
\end{equation}
The first term estimates the trace contribution; the other terms are
$O_p(h^{3/2})$ and $O(h^2)$. They are available from the same
fixed-displacement JVP, but do not by themselves establish a higher-order
scheme because other terms of those orders are omitted. The proved
accuracy remains $O_{L^2}(h^{3/2})$ for the defined Gaussian log-weight
surrogate, with the same-noise and stabilization distinctions stated above.

\section{Class-Removal Experimental Details}
\label{app:cifar-cr-details}

This appendix documents the class-removal experiments of
\S\ref{sec:cifar10:cr}. Source-of-truth reports with full provenance
(scripts, logs, per-row CSVs) are included in the supplementary materials
as \texttt{class\_removal\_baseline.md} and \texttt{class\_removal\_smc.md}.


\subsection{Models, dataset, and metrics}

\paragraph{Unconditional generator (\(p^A\)).}
A standard DDPM UNet trained from scratch on CIFAR-10
(\texttt{cifar10\_unconditional\_v2/best\_checkpoint.pt}, EMA state
loaded). Configuration: \(\mathrm{nf}=128\), channel multipliers
\((1,2,2,2)\), \(2\) residual blocks per resolution, attention at
\(16\!\times\!16\), GroupNorm and Swish, \(\sim\!35.7\)M parameters.
Sampling uses ancestral DDPM with \(T=1000\) steps and a linear
\(\beta\) schedule from \(10^{-4}\) to \(2{\times}10^{-2}\). All methods
(baselines and ours) share this backbone.

\paragraph{Forget-class generator (\(p^B\)).}
For each forget class \(c\in\{0,1\}\), a UNet of identical architecture
trained on the single-class subset of CIFAR-10 (planes for c0, cars for
c1). Checkpoints come from the public release accompanying
\citet{koulischer2024dynamic}; the network state is loaded as the raw
\texttt{net} dict.

\paragraph{Evaluation.}
For class removal tasks, we used the same metrics as \citet{koulischer2024dynamic}. For each generated batch of \(N=10{,}240\) images we compute (i) Wrong\%,
the fraction classified as the forget class by the
\texttt{aaraki/vit-base-patch16-224-in21k-finetuned-cifar10} ViT
(HuggingFace pipeline, one forward per image); (ii) KL of the empirical
nine-keep-class histogram to the uniform distribution \(p_{\mathrm{u}}=
(1/9,\ldots,1/9)\), defined as \(\sum_{i\in\text{keep}} p_{\mathrm{u}}
\log(p_{\mathrm{u}}/q_i)\) with natural log and no renormalization of
the keep mass; and (iii) FID against the CIFAR-10 train split
filtered to drop the forget class (\(\approx\!45{,}000\) reference
images), using InceptionV3 pool3 features.

\paragraph{Pretrained-model FID floor.}
Sampling each underlying UNet \emph{unconditionally} (no guidance) at
\(N=10{,}240\) and computing FID against the natural reference set yields:

\begin{center}
\begin{tabular}{lcc}
\toprule
Model & Reference set & FID \\
\midrule
\(p^A\) (unconditional) & CIFAR-10 train (50k images, 10 classes)  & \textbf{5.43} \\
\(p^B\) (planes)        & CIFAR-10 train, plane subset (5k images) & 22.41 \\
\(p^B\) (cars)          & CIFAR-10 train, car subset (5k images)   & 11.15 \\
\bottomrule
\end{tabular}
\end{center}

These FIDs use single-class reference sets and are not directly
comparable to the FID column reported in \S\ref{sec:cifar10:cr}; they
are reported as a sanity check on the underlying models. The
unconditional value \(5.43\) is, however, a lower bound for the FID of
any class-removal method, since at zero guidance every method reduces
to unconditional sampling. The asymmetric self-FID of the two
forget-class models (cars sharper than planes) is consistent with the
asymmetric difficulty observed in \S\ref{sec:cifar10:cr} between c0
and c1.

\subsection{Hyperparameters and per-method scale grids}

All methods use \(N=10{,}240\) samples per configuration and seed \(0\)
(per-batch seeds are \(\{0,1,\ldots\}\) for resumable sampling). Batch
sizes are \(256\) (baselines) and \(512\) (c0 SMC) / \(1024\) (c1 SMC).

\paragraph{Composition.}
\(\varepsilon_{\text{guided}} = \varepsilon_u - \lambda(\varepsilon_f - \varepsilon_u)\).\\
Scale grid \(\lambda\in\{0.05,0.10,0.20,0.30,0.40,0.50,0.70,1.00\}\) on
both classes.

\paragraph{Safe Latent Diffusion.}
Safe Latent Diffusion \citep{schramowski2023safe} with the
hyperparameter set used in the public DNG repository
\citep{koulischer2024dynamic}: element-wise gated correction with EMA
momentum, threshold \(=0.04\), \(s_s=100.0\), momentum decay
\(\beta_m=0.1\), momentum rescale \(s_m=0.2\). Same scale grid as
Composition.

\paragraph{Dynamic Negative Guidance (using defaults from \citet{koulischer2024dynamic}).}
Per-particle scale \(\lambda_t = \lambda \cdot p_t/(1-p_t)\) with
running posterior \(p_t\). Hyperparameters: prior \(0.01\), temperature
\(0.2\), \(p_{\min}=10^{-6}\), \(p_{\max}=0.8\), offset
\(2{\times}10^{-4}\). Scale grids: eight points for class 0 \(\lambda\in\{0.10,0.25,0.50,1.00,2.00,5.00,8.00,12.00\}\); eight points for class 1 \(\lambda\in\{0.10,0.25,0.50,1.00,3.50,5.00,8.00,12.00\}\).

\paragraph{Safe Denoiser \citep{kim2025training}.}
Kernel-fast repellency in pixel space with \(500\) unsafe reference
images (selection seed \(42\)) and \(\sigma\) auto-calibrated as the
median pairwise reference distance in kernel space. Time window
\([0.78,1.00]\cdot T\) of the diffusion schedule. Auto window-quantile
threshold calibration with \(q=0.80\), margin \(1.6\), \(40\) \(\beta\)
sample points, \(\beta\)-quantile \(0.25\). Scale grids: sixteen points
\(\{0.05,0.08,0.125,0.20,0.33,0.50,0.75,1.00,1.50,2.00,3.00,4.00,6.00,8.00,12.00,16.00\}\)
on c0 (full operating curve) and the upper eight
\(\{2.00,3.00,4.00,5.00,6.50,8.00,9.00,10.00\}\) on c1 (where SD shows
non-trivial repellency).

\paragraph{Proposed sampler.}
Sigmoid gate \(\theta = \sigma(\log c + \beta_s\hat\ell)\) with both
\(c(t)\) and \(\beta_s(t)\) on exp schedules. Shared:
\(c_1=0.1\), \(\beta_{s,1}=0.2\); \(\hat\ell\)-clip \(20\),
\(\hat\ell\)-temperature \(\tau=0.5\), \(\hat\ell\)-offset
\(2{\times}10^{-4}\); ESS resampling threshold \(0.5\); \(\theta\)-clamp
\([0.02,0.98]\). Variance-reduction (RK) inner loop with \(z\)-cap
\(0.25\), \(t_{\min}=0.10\), ramp \(0.20\), \(10\) gradient steps with
backtracking. Class-specific schedule constants:
\(c_0=1\), \(\gamma_{\text{bs}}=2\), \(\gamma_c=1\) for c0;
\(c_0=3\), \(\gamma_{\text{bs}}=5\), \(\gamma_c=5\) for c1 (a more
aggressive ramp, selected from a pilot \(N=512\) sweep to match DNG
performance on the harder c1 task; see discussion in
\S\ref{sec:cifar10:cr}). Scale (\(=\beta_{s,0}\)) grid: twelve points
\(\{0.5,0.6,0.8,1,1.5,2,3,4,5,6,7,8\}\) on c0; eight points
\(\{1,1.5,2,3,4,5,6,7\}\) on c1.

\subsection{Compute and provenance}

Generation was distributed across eight NVIDIA A6000-class GPUs in
waves; per-configuration wall-clock at \(N=10{,}240\) is roughly
\(280\)--\(560\) minutes for baselines and \(195\)--\(285\) minutes for
the proposed sampler (longer batches but only \(10\) batches at
\(\text{batch}=1024\) for c1 vs \(20\) batches at \(\text{batch}=512\)
for c0). Total compute for the class-removal experiments is
approximately \(\sim 850\) GPU-hours.

\section{Multiclass-Selection Experimental Details}
\label{app:cifar-mc-details}

\subsection{Models, dataset, and metrics}

\paragraph{Five-class generators (\(p^A\) and \(p^B\)).}
Both are DDPM UNets of the same architecture as the unconditional
generator of App.~\ref{app:cifar-cr-details}, trained from scratch on
disjoint five-class subsets of CIFAR-10 (\(25{,}000\) images each, EMA
state loaded). Forward task: \(p^A\) is
\texttt{cifar10\_classes\_01234\_v2/best\_checkpoint.pt} (classes
\(\{0,1,2,3,4\}\)) and \(p^B\) is
\texttt{cifar10\_classes\_34567\_v2/best\_checkpoint.pt} (classes
\(\{3,4,5,6,7\}\)); the two shared classes are \(\{3,4\}\) (cat,
deer) and the keep set is \(A\setminus B = \{0,1,2\}\) (airplane,
automobile, bird). Reverse task: roles are swapped and the keep set
is \(\{5,6,7\}\) (dog, frog, horse). Sampling protocol and
\(\beta\)-schedule match App.~\ref{app:cifar-cr-details}. The
DNG\(_{\text{uncond}}\) baseline additionally uses the unconditional
ten-class generator of App.~\ref{app:cifar-cr-details} as its base
denoiser; the proposed sampler does not use the unconditional model.

\paragraph{Evaluation.}
For each configuration we draw \(N=10{,}240\) samples and report:
(i) Wrong\%, the fraction classified into the unwanted-only classes
\(B\setminus A\) by the same ViT classifier as
App.~\ref{app:cifar-cr-details}; (ii) FID against the CIFAR-10 train
split filtered to the keep classes (\(15{,}000\) reference images for
each direction); (iii) Recall \citep{kynkaanniemi2019improved}
against the same keep-class reference set, computed by the standard
\(k\)-nearest-neighbour manifold construction with \(k=3\) on
InceptionV3 pool3 features; and (iv) KL of the empirical
\(|\mathcal{K}|=3\)-bin keep-class histogram to uniform, defined as
\(\sum_{i\in\mathcal{K}} p_{\mathrm{u}}\log(p_{\mathrm{u}}/q_i)\)
with \(p_{\mathrm{u}}=(1/3,1/3,1/3)\). The choice of Recall over KL
as the main-text plotting axis is discussed below.

\begin{center}
\begin{tabular}{lcc}
\toprule
Model & Reference set & FID \\
\midrule
\(p^A\) (forward, classes 0--4) & CIFAR-10 train, classes 0--4 (25k) & 6.92 \\
\(p^A\) (reverse, classes 3--7) & CIFAR-10 train, classes 3--7 (25k) & 6.76 \\
Unconditional ten-class generator & CIFAR-10 train, all 50k & 5.43 \\
\bottomrule
\end{tabular}
\end{center}

These FIDs are computed on five-class reference sets and are not
directly comparable to the FID column reported in
\S\ref{sec:cifar10:mc}, which uses three-class keep references. They
serve as a floor on the FID achievable by any specificity method that
retains the full \(A\)-modes, since at zero guidance every method
reduces to sampling \(p^A\). The mild asymmetry between forward and
reverse (\(6.92\) vs.\ \(6.76\)) is consistent with the asymmetric
difficulty observed in \S\ref{sec:cifar10:mc}.

\paragraph{Choice of metric: Recall over KL.}
Fig.~\ref{fig:cifar-mc-kl} reports the Wrong\,\%--KL Pareto for the
forward direction for completeness; the method ordering is consistent
with the Recall plot in the main text. We prefer Recall over KL in
multiclass selection for two reasons. First, KL to uniform over
\(|\mathcal{K}|=3\) keep classes has only three histogram bins, so
it measures class-level balance but is blind to intra-class
diversity: a method can attain \(\mathrm{KL}\approx 0\) while
collapsing each kept class to a single stereotypical mode (e.g.\
the full variety of airplane views to a single canonical pose), and
this mode shaving is invisible. Recall instead uses
\(15{,}000\) reference points in \(2048\)-dimensional InceptionV3
feature space with the standard \(k\)-NN manifold construction, so
each real image votes independently on whether its local neighbourhood
is covered, and boundary-adjacent modes lost under blanket guidance
(e.g.\ unusual horse poses, dark-background frogs close in score
space to the suppressed cat class) show up directly as a Recall drop.
The nine-bin histogram of class removal has enough resolution to
compensate; at \(|\mathcal{K}|=3\) it does not. Second, the
\texttt{evaluate.py} KL computation uses unnormalised keep-class
fractions (it divides by total sample count rather than renormalising
over keep classes), so the raw KL inherits a
\(\log\!\bigl(1/(1-W)\bigr)\) artefact that is monotone in
Wrong\,\%; the number can be corrected post hoc, but Recall has no
such ambiguity.

\begin{figure}[H]
\centering
\includegraphics[width=0.55\linewidth]{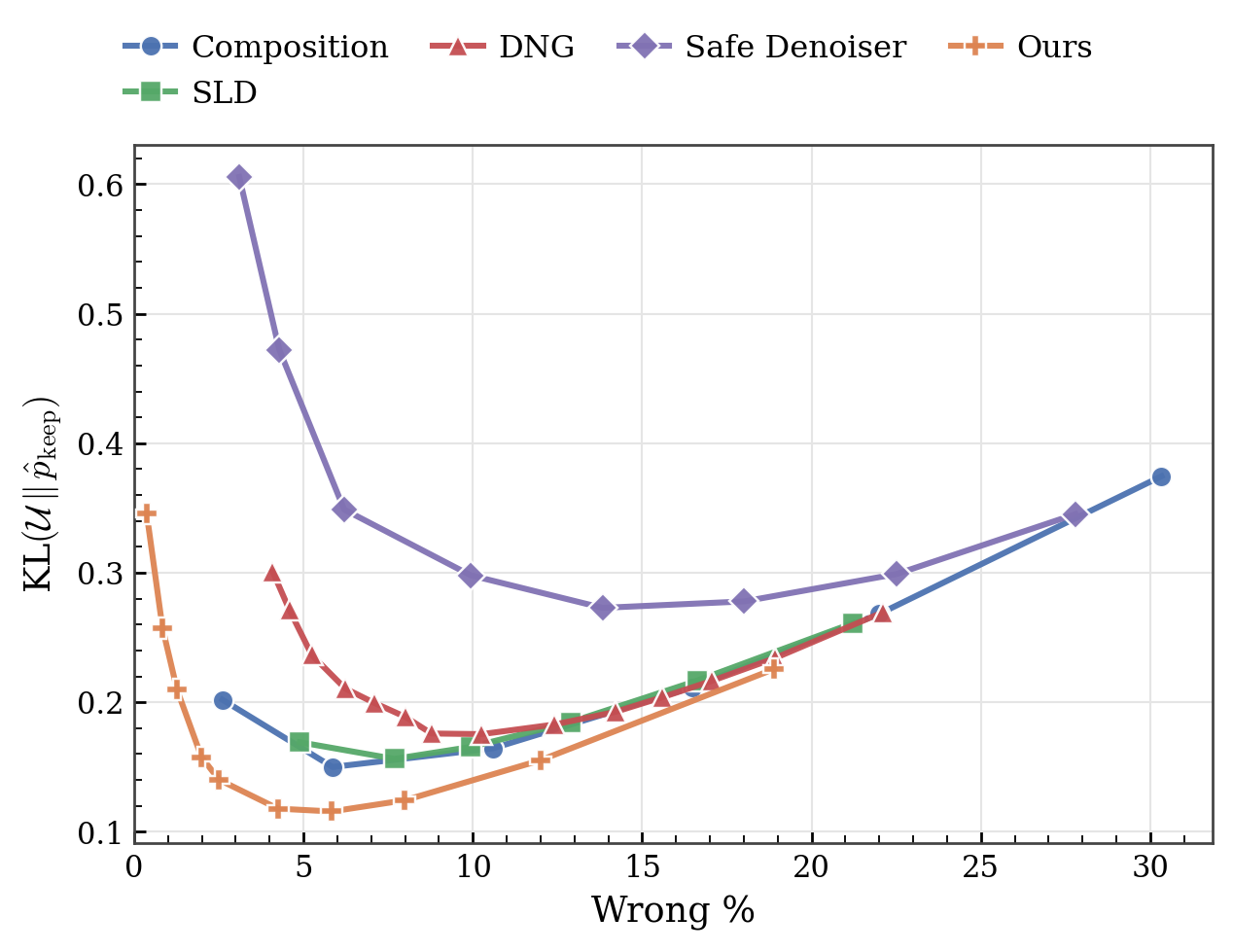}
\caption{Wrong\,\%--KL Pareto for the forward multiclass-selection
task (keep \(\{0,1,2\}\)). The method ordering is consistent with the
Recall plot in Fig.~\ref{fig:two_by_four}. The reverse direction
exhibits the same method ordering on FID and Recall but not on KL:
with only three keep classes the histogram is too coarse to
discriminate adaptive methods (DNG, SMC) from uniform-strength
baselines (Composition, SLD) on class balance alone.}
\label{fig:cifar-mc-kl}
\end{figure}

\subsection{Hyperparameters and per-method scale grids}

All methods use \(N=10{,}240\) samples per configuration and seed
\(0\) (per-batch incremental seeding). Baselines run at batch
\(256\); the proposed sampler runs at batch \(512\) (forward) and
\(1024\) (reverse).

\paragraph{Composition.}
\(\varepsilon_{\text{guided}}=\varepsilon_A-\lambda(\varepsilon_B-\varepsilon_A)\)
with \(\varepsilon_A,\varepsilon_B\) the score networks of \(p^A\)
and \(p^B\). Scale grid (both directions):
\(\lambda\in\{0.1,0.3,0.5,0.8,1.2,1.8\}\).

\paragraph{Safe Latent Diffusion.}
SLD with the hyperparameter set used in the public DNG repository
\citep{koulischer2024dynamic}: threshold \(=0.04\), \(s_s=100.0\),
momentum decay \(\beta_m=0.2\), momentum rescale \(s_m=0.1\), no
warmup. Scale grid (both directions):
\(\{0.4,0.6,0.8,1.0,1.2,1.6\}\).

\paragraph{Dynamic Negative Guidance.}
DNG: per-particle adaptive scale
\(\lambda_t=\lambda\cdot p_t/(1-p_t)\) with running posterior
\(p_t\), using the unconditional generator as the base
denoiser, as being faithful to the original paper setting. Hyperparameters: prior \(0.01\), temperature \(0.2\),
\(p_{\min}=10^{-6}\), \(p_{\max}=0.8\), offset \(2{\times}10^{-4}\).

\paragraph{Safe Denoiser.}
Kernel-fast repellency in pixel space, base \(=p^A\) (the wanted
five-class model). Reference set: \(2{,}500\) unsafe images sampled
from the unwanted-class CIFAR-10 train subset; \(\sigma=4.00\)
(fixed, not auto-calibrated); time window
\([0.70,1.00]\cdot T\); auto window-quantile threshold calibration
with margin \(2.0\), \(\beta\)-quantile \(0.25\), \(50\) sample
points, calibration window \(q=0.80\). Scale grid (forward,
\(N=10{,}240\)): \(\{1.5,2.5,3.5,4.5,6.0,8.0,11.0,15.0\}\). Safe
Denoiser is not reported on the reverse task; see
\S\ref{sec:cifar-mc-sd-failure} for the rationale.

\paragraph{Proposed sampler.}
Sigmoid gate \(\theta=\sigma(\log c+\beta_s\hat\ell)\) with both
\(c(t)\) and \(\beta_s(t)\) on exp schedules. Shared:
\(c_0=1.0\), \(c_1=0.1\), \(\beta_{s,1}=0.2\),
\(\gamma_{\text{bs}}=2.0\), \(\gamma_c=1.0\),
\(\hat\ell\)-temperature \(\tau=0.2\), \(\theta\)-clamp
\([0.02,0.98]\), \(\log w\)-clip \(80\), ESS resampling threshold
\(0.5\). Variance-reduction (RK) inner loop: \(z\)-cap \(0.25\),
\(t_{\min}=0.10\), ramp \(0.20\), \(10\) gradient steps with
backtracking, learning rate \(0.1\). $\rho_1$ fixed as 0. Direction-specific: forward uses
\(\hat\ell\)-clip \(20\) and \(\hat\ell\)-offset
\(2{\times}10^{-4}\); reverse uses \(\hat\ell\)-clip \(10\) and
\(\hat\ell\)-offset \(6{\times}10^{-5}\).

\subsection{Reverse Safe Denoiser exclusion}
\label{sec:cifar-mc-sd-failure}

We do not report Safe Denoiser on the reverse task because every
configuration we tried sits outside the plotted region of
Fig.~\ref{fig:two_by_four}. Across an \(N=512\) sweep over
\(\sigma\in\{4,20,40,80,\text{auto}\!\approx\!36\}\), scales
\(\lambda\in[0.5,20]\), and three choices of unsafe reference set
(full \(\{0,1,2,3,4\}\), overlap-only \(\{3,4\}\), and far-only
\(\{0,1,2,8,9\}\)), Wrong\(\,\%\) stayed in
\([32,41]\,\%\) and FID in \([60,121]\); the next-worst
baseline, Composition at \(\lambda=1.8\), already reaches
Wrong \(=2.9\,\%\) at FID \(54.64\).
The failure is not a gate issue: the apply rate is saturated at
\(1.0\) in every run and the auto-calibrated \(\beta\)-threshold is
essentially unchanged from the forward task.

The cause is a pixel-space discriminability issue specific to the
reverse direction. SD builds its repellent direction as
\(\widehat x_0-\widehat u(\widehat x_0)\), where
\(\widehat u(\widehat x_0)\) is a Nadaraya--Watson average of the
unsafe references weighted by an RBF kernel of bandwidth \(\sigma\)
in pixel space. For a sample on the reverse keep manifold
\(\{5,6,7\}=\{\text{dog},\text{frog},\text{horse}\}\), the nearest
unsafe references are the cat and deer exemplars (the overlap classes
\(\{3,4\}\)), not the vehicles or bird (\(\{0,1,2\}\)), because all
animal classes are visually similar at \(32\!\times\!32\) resolution.
The kernel-weighted unsafe centroid \(\widehat u(\widehat x_0)\)
therefore lies inside the keep manifold, and the residual repellent
direction damages the keep classes rather than separating them from
the unwanted modes. Small \(\sigma\) makes the cat/deer dominance
more local but does not remove it; large \(\sigma\) collapses
\(\widehat u\) to a uniform reference mean, which is also
non-discriminative. Restricting the unsafe set to the overlap classes
\(\{3,4\}\) only, or to the far-only classes \(\{0,1,2,8,9\}\), does
not break the Wrong floor; the overlap-only run with \(\sigma\) auto
and scale \(8\) collapsed to Wrong \(=56.4\,\%\) at FID \(=121.0\).
The floor itself is approximately the wanted-only prior: \(p^A\)
already places roughly \(30\%\) of its mass on the overlap classes
by construction, and SD cannot reduce that mass without distorting
\(\{5,6,7\}\).

The forward task does not exhibit this failure because the forward
keep set \(\{0,1,2\}=\{\text{plane},\text{automobile},\text{bird}\}\)
is pixel-disjoint from the animal-only unsafe set: the kernel cleanly
separates the keep manifold from the unwanted modes, and SD-forward
traces a usable Wrong/FID curve. The
class-removal task of \S\ref{sec:cifar10:cr} avoids the failure by
construction, since the single forget class is disjoint from the
nine keep classes.

\subsection{Compute and provenance}

Generation was distributed across eight NVIDIA A6000-class GPUs in
waves; per-configuration wall-clock at \(N=10{,}240\) is roughly
\(200\)--\(360\) minutes for Composition, SLD, and DNG,
\(240\)--\(360\) minutes for SD (the kernel evaluation against
\(2{,}500\) reference images adds a fixed per-step cost), and
\(140\)--\(200\) minutes for the proposed sampler (\(10\)--\(20\)
batches at batch \(512\) or \(1024\)). Total compute for the
multiclass-selection experiments is approximately \(345\) GPU-hours
(\(\approx 275\) baselines, \(\approx 70\) proposed sampler). All
configuration scripts, per-configuration logs, and aggregation code
are included in the supplementary materials.




\section{Text-to-Image Experimental Details}
\label{app:t2i-details}

This appendix describes the text-to-image experiments in
\S\ref{sec:t2i}. We use the same local SMC construction as in the main method,
together with fixed numerical stabilizers for high-dimensional latent
diffusion. Unless stated otherwise, one configuration, denoted by \(c48\), is
used for all five prompts and both negative-prompt types. Thus, the same
hyperparameters are used in all ten experimental settings.

\subsection{Experimental Setting}
\label{sec:t2i-setting}

\paragraph{Diffusion model.}
We use Stable Diffusion v1.4 \citep{rombach2022high} from the public
\texttt{CompVis/stable-diffusion-v1-4} checkpoint. Each experiment uses
\(K=100\) denoising steps and \(N=32\) particles.

\subsection{Hyperparameters}
\label{sec:t2i-target-reference}

At step \(k\), let
\begin{equation}
    \lambda_k
    :=
    \frac{k}{K-1},
    \qquad
    \lambda_k\in[0,1],
\end{equation}
denote the progress from the initial noise distribution to the generated
image. The intermediate target is
\begin{equation}
    q_k(x)
    \propto
    p_k^A(x)
    \bigl(1-\theta_k(x)\bigr)^{m_k}.
    \label{eq:t2i-powered-target}
\end{equation}
The gate power increases linearly,
\begin{equation}
    m_k
    =
    26\lambda_k,
    \label{eq:t2i-power-schedule}
\end{equation}
so that the specificity constraint is introduced gradually and the terminal
target uses \(m_{K-1}=26\).

We use
\begin{equation}
    \eta_k=0.75,
    \qquad
    \tau_\ell=0.75,
\end{equation}
where \(\tau_\ell\) scales the estimated log-density-ratio increments used
by the gate. The gate offset is scheduled through
\begin{equation}
    \log c_k
    =
    \log c_\star
    -
    7w(\lambda_k)
    -
    0.05\sin^2(\pi\lambda_k),
    \qquad
    \log c_\star = 61.65.
    \label{eq:t2i-c-schedule}
\end{equation}
Here, \(w(\lambda)\) increases smoothly from \(0\) to \(1\) over the first
quarter of the trajectory, remains equal to \(1\) over the middle half, and
returns smoothly to \(0\) over the final quarter. We use cosine transitions
at both boundaries. Although \(c_\star\) is numerically large, the gate
depends only on \(\log c_k\); its value compensates for the scale and offset
of the estimated latent-space log-density ratio.

The forward correction is
\begin{equation}
    u_k^{(\rho_1,\rho_2)}(x)
    =
    \rho_1 d_k(x)
    +
    \rho_{2,k}g_k(x),
    \label{eq:t2i-two-field-proposal}
\end{equation}
where
\begin{equation}
    d_k
    =
    \nabla\log p_k^A-\nabla\log p_k^B,
    \qquad
    g_k
    =
    \nabla\mathcal R_k
    =
    \zeta_k d_k,
\end{equation}
and \(\mathcal R_k=m_k\log(1-\theta_k)\). We fix
\begin{equation}
    \rho_1=-2
\end{equation}
throughout sampling. At each active optimization step, \(\rho_{2,k}\) is
selected from \(101\) uniformly spaced candidates in
\begin{equation}
    \rho_{2,k}\in[0,1.2].
\end{equation}
The value \(\rho_{2,k}=0\) is always included as a guard candidate. Each
step evaluates the full grid independently, without warm-starting or
fixed-point refinement. Optimization is disabled during the final denoising
step.

For every candidate, we add the stabilized increment to the cumulative log
weight carried from the preceding step and compute the normalized ESS of the
result. We select the candidate with the largest predicted full cumulative
ESS. Thus, the optimization criterion is neither the incremental ESS alone
nor an ESS computed from clipped ancestor-selection weights.

\paragraph{Mapping to Proposition~\ref{prop2}.}
The latent-space contrast estimator is rescaled by
$\tau_\ell$ before entering the gate. Equivalently, the
implementation uses the effective gate slope
\begin{equation}
    \eta_k^{\mathrm{eff}}
    :=
    \tau_\ell\eta_k .
    \label{eq:t2i-effective-eta}
\end{equation}
Accordingly, the gate-dependent coefficients in
Proposition~\ref{prop2} are evaluated as
\begin{equation}
\begin{aligned}
    \zeta_k
    &=
    m_k\eta_k^{\mathrm{eff}}\theta_k,\\
    \chi_k
    &=
    m_k(\eta_k^{\mathrm{eff}})^2
    \theta_k(1-\theta_k).
\end{aligned}
\label{eq:t2i-effective-gate-coefficients}
\end{equation}

\paragraph{Time-varying target power.}
Because $m_k$ varies over the trajectory, the one-step
gate contribution includes the discrete power-schedule term
\begin{equation}
    (m_{k+1}-m_k)\log(1-\theta_k).
    \label{eq:t2i-power-correction}
\end{equation}
We use this discrete correction once in the implemented
Proposition~\ref{prop2} increment; it is the finite-step
counterpart of the $\dot m$ term in
Appendix~\ref{prop2proof}.

\subsection{Finite-Step Stabilization and Resampling}
\label{sec:t2i-weight-stabilization}

The Proposition-2 expansion contains several particle-wise contributions
that can be heavy-tailed at finite step size. With only eight particles, a
single extreme value can dominate the ESS objective. We therefore stabilize
three groups of terms while using exactly the same stabilized expansion for
coefficient selection and propagated weights.

Let \(v=(v_1,\ldots,v_N)\) be a particle-wise contribution and define
\begin{equation}
    \widetilde v
    :=
    \operatorname{median}_i(v_i),
    \qquad
    \delta_i
    :=
    |v_i-\widetilde v|.
\end{equation}
Let \(\delta_{(r+1)}\) be the \((r+1)\)-st largest centered deviation. For
an optional absolute bound \(B\), define
\begin{equation}
    T_v^{(r,B)}
    :=
    \min\{\delta_{(r+1)},B\},
\end{equation}
where \(B=\infty\) denotes no absolute bound. We use the median-centered
winsorization operator
\begin{equation}
    \mathcal W_{r,B}(v_i)
    :=
    \widetilde v
    +
    \operatorname{sign}(v_i-\widetilde v)
    \min\left\{
        |v_i-\widetilde v|,
        T_v^{(r,B)}
    \right\}.
    \label{eq:t2i-rank-winsorization}
\end{equation}
Ties at the threshold are retained. For \(N=32\), we use \(r=1\), so the
rank rule controls only the most extreme centered deviation.

We decompose the raw expansion as
\begin{equation}
    \widehat\ell_{k,i}^{\mathrm{P2}}
    =
    G_{k,i}
    +
    H_{k,i}^{0}
    +
    J_{k,i}^{\mathrm{raw}},
\end{equation}
where \(G_{k,i}\) contains the gate-time and target-curvature terms,
\(J_{k,i}^{\mathrm{raw}}\) is the Jacobian--vector-product contribution, and
\(H_{k,i}^{0}\) contains the forward, reverse, and proposal-curvature terms.
The gate group is stabilized as
\begin{equation}
    G_{k,i}^{\mathrm{stab}}
    =
    \mathcal W_{1,0.10}(G_{k,i}).
    \label{eq:t2i-gate-clipping}
\end{equation}
For the JVP contribution, we first apply the fixed shrinkage
\begin{equation}
    J_{k,i}^{\mathrm{shrunk}}
    =
    0.75J_{k,i}^{\mathrm{raw}},
\end{equation}
and then use
\begin{equation}
    J_{k,i}^{\mathrm{stab}}
    =
    \mathcal W_{1,0.25}
    \left(J_{k,i}^{\mathrm{shrunk}}\right).
    \label{eq:t2i-jvp-clipping}
\end{equation}
Finally, the combined kernel contribution is stabilized after inserting the
JVP term,
\begin{equation}
    H_{k,i}^{\mathrm{stab}}
    =
    \mathcal W_{1,0.10}
    \left(H_{k,i}^{0}+J_{k,i}^{\mathrm{stab}}\right),
    \label{eq:t2i-kernel-clipping}
\end{equation}
and the propagated increment is
\begin{equation}
    \widehat\ell_{k,i}^{\mathrm{stab}}
    =
    G_{k,i}^{\mathrm{stab}}
    +
    H_{k,i}^{\mathrm{stab}}.
    \label{eq:t2i-stabilized-weight}
\end{equation}
The same operations, in the same order, are applied to every candidate in
the \(\rho_{2,k}\) grid.

\paragraph{Resampling.}
Let \(L_{k,i}\) denote the full cumulative log weight after adding
\eqref{eq:t2i-stabilized-weight}. Resampling is triggered when the normalized
ESS computed directly from \(L_k\) falls below \(0.30\). We require at least
\(10\) denoising steps between successive events and permit resampling over
the entire trajectory. Hence, a \(100\)-step trajectory can contain at most
ten resampling events.

The resampling trigger always uses the unclipped full cumulative weights. To
avoid a single particle dominating ancestor selection, we draw ancestors
using the median-centered weights
\begin{equation}
    \widetilde L_{k,i}
    =
    \operatorname{median}_j(L_{k,j})
    +
    \operatorname{clip}
    \left(
        L_{k,i}-\operatorname{median}_j(L_{k,j}),
        -1.5,
        1.5
    \right).
    \label{eq:t2i-ancestor-clipping}
\end{equation}
The residual \(L_{k,i}-\widetilde L_{k,i}\) is transferred to every selected
child and carried into subsequent steps. Thus, clipping changes only the
ancestor proposal probabilities; it does not discard the corresponding
importance-weight correction.

These finite-step operations leave the analytic target in
\eqref{eq:t2i-powered-target} unchanged. The propagated weight is the
stabilized retained expansion in Proposition~\ref{prop2}, including its
scheduled-power correction. It is not the exact finite-step Gaussian
transition-density ratio.

\subsection{Prompts}
\label{sec:t2i-baselines-prompts}

We use the five scenes from \citet{koulischer2024dynamic}. Each scene is
paired with one semantically related and one unrelated negative prompt,
resulting in ten experimental settings.

\begin{table}[t]
    \centering
    \caption{Text-to-image prompts.}
    \label{tab:t2i-prompts}
    \small
    \resizebox{\linewidth}{!}{
    \begin{tabular}{c p{0.13\linewidth} p{0.45\linewidth}
        p{0.16\linewidth} p{0.18\linewidth}}
        \toprule
        No. & Scene & Positive prompt & Related negative & Unrelated negative \\
        \midrule
        1
        & Medieval feast
        & A grand medieval feast set in a great hall, filled with long tables
        covered in bountiful platters of food, goblets, and flickering light,
        with knights and nobles enjoying the lavish spread.
        & Chalices, candles
        & Kids playing football \\
        \midrule
        2
        & English breakfast
        & A classic British breakfast, featuring a diverse selection of cooked
        delights, perfect for a filling and flavorful start to the day.
        & Egg, sausage
        & The view of a skyline \\
        \midrule
        3
        & Dinner table
        & A beautifully set dinner table, with elegant arrangements, a variety
        of enticing dishes, and delicate decor that speaks to the
        sophistication of the meal.
        & Flowers, wine glasses
        & A person wearing sunglasses \\
        \midrule
        4
        & Art workshop
        & An inspiring art workshop filled with creativity, featuring a wide
        range of tools and materials used by participants working on their
        individual projects.
        & Paintbrush, canvases
        & A bicycle in the rain \\
        \midrule
        5
        & Antique store
        & A charming antique store with shelves and tables filled with unique
        and historical treasures, offering a glimpse into the past.
        & Lamps, old books
        & An electric toothbrush \\
        \bottomrule
    \end{tabular}
    }
\end{table}

\subsection{Evaluation Metrics}
\label{sec:t2i-metrics}

For each method and condition, we generate \(32\) images. We compute CLIP
similarity using the \texttt{openai/clip-vit-large-patch14} model
\citep{radford2021learning}. Let
\(\mathrm{CLIP}_{+}\) denote similarity to the positive prompt and
\(\mathrm{CLIP}_{-}\) similarity to the negative prompt used during
generation. We report both quantities and the specificity margin
\begin{equation}
    \mathrm{Margin}
    =
    \mathrm{CLIP}_{+}
    -
    \mathrm{CLIP}_{-}.
    \label{eq:t2i-clip-margin}
\end{equation}

To measure particle diversity after resampling, let \(f_j\) be the fraction
of final particles descended from initial particle \(j\). We report the
number of unique surviving roots, the largest family fraction
\begin{equation}
    F_{\max}
    =
    \max_j f_j,
\end{equation}
and the effective number of roots
\begin{equation}
    R_{\mathrm{eff}}
    =
    \frac{1}{\sum_j f_j^2}.
    \label{eq:t2i-effective-roots}
\end{equation}
A larger \(R_{\mathrm{eff}}\), a larger number of unique roots, and a smaller
\(F_{\max}\) indicate less genealogical collapse.

\subsection{Ablations}
\subsubsection{JVP Contribution}
\label{sec:t2i-full-displacement-ablation}

We evaluate the JVP contribution on Prompt~1 with its related negative prompt.  On the same selected trajectory, we reconstruct the
raw Proposition~2 increment before winsorization both with and without the JVP
term and compare each surrogate with the exact finite one-step Gaussian
log-weight.  All metrics are computed after centering across particles within
each non-degenerate step, since a common log-weight offset does not affect the
normalized particle weights.  The exact finite Gaussian quantity requires an
additional selected-endpoint denoiser evaluation and is used only as an
evaluation diagnostic; both propagation and proposal optimization continue to
use the stabilized Proposition~2 expansion.

Table~\ref{tab:t2i-full-displacement-ablation} shows that including the JVP
term increases the Pearson correlation from $0.689$ to $0.804$, reduces
centered MAE by $22.0\%$, and reduces centered relative RMSE by $13.0\%$.
It also reduces the mean absolute normalized-ESS error by $7.4\%$.  The
no-JVP result is obtained by removing the JVP term from the same raw expansion,
so the comparison holds the states, noise realizations, and all remaining
terms fixed.

\begin{table}[t]
    \centering
    \caption{Agreement of the raw Proposition~2 surrogate with the exact
    finite one-step Gaussian log-weight on seed~$10{,}000$.  Metrics are
    computed after particle-wise centering within each non-degenerate step.}
    \label{tab:t2i-full-displacement-ablation}
    \small
    \resizebox{\linewidth}{!}{
    \begin{tabular}{lcccc}
        \toprule
        \textbf{Proposition~2 surrogate}
        & \textbf{Pearson}$\uparrow$
        & \textbf{Centered MAE}$\downarrow$
        & \textbf{Centered rel.\ RMSE}$\downarrow$
        & \textbf{Normalized-ESS MAE}$\downarrow$ \\
        \midrule
        Without JVP
        & $0.689$ & $2.094$ & $0.783$ & $0.115$ \\
        With JVP
        & $\mathbf{0.804}$ & $\mathbf{1.633}$
        & $\mathbf{0.681}$ & $\mathbf{0.107}$ \\
        \bottomrule
    \end{tabular}
    }
\end{table}

\subsubsection{Variance Proposal Optimization}
\label{sec:t2i-rho-ablation}

We evaluate variance-proposal optimization on Prompt~1 with its related
negative prompt using five paired seeds.  Each run uses eight particles and $100$ diffusion
steps.  We fix $\rho_1=-2$ and compare a proposal selected by maximizing the
full cumulative ESS over $101$ values of $\rho_2\in[0,1.2]$ with the fixed
proposal $\rho_2=0$.  The two conditions use identical random seeds and share
all other hyperparameters, including the stabilized Proposition~2 weights,
the ancestor-weight cap, and the minimum separation of ten steps between
resampling events.

\begin{table}[t]
    \centering
    \caption{Effect of full-ESS optimization of $\rho_2$ on CLIP specificity
    and particle genealogy.  Values are mean $\pm$ standard deviation over
    five paired seeds when 8 images generated.}
    \label{tab:t2i-rho-genealogy-ablation}
    \small
    \resizebox{\linewidth}{!}{
    \begin{tabular}{lccccc}
        \toprule
        \textbf{Proposal}
        & $\boldsymbol{\mathrm{Margin}}\uparrow$
        & $\boldsymbol{R_{\mathrm{eff}}}\uparrow$
        & \textbf{Unique roots}$\uparrow$
        & $\boldsymbol{F_{\max}}\downarrow$
        & \textbf{Resamples}$\downarrow$ \\
        \midrule
        Fixed $\rho_2=0$
        & $0.1294\pm0.0130$
        & $3.49\pm1.41$
        & $4.20\pm1.64$
        & $0.425\pm0.143$
        & $1.80\pm0.84$ \\
        Full-ESS optimized $\rho_2$
        & $\mathbf{0.1363\pm0.0080}$
        & $\mathbf{5.07\pm1.74}$
        & $\mathbf{5.80\pm1.30}$
        & $\mathbf{0.300\pm0.112}$
        & $\mathbf{0.80\pm0.45}$ \\
        \bottomrule
    \end{tabular}
    }
\end{table}

As shown in Table~\ref{tab:t2i-rho-genealogy-ablation}, optimizing $\rho_2$
increases the effective and unique numbers of roots by $45.0\%$ and $38.1\%$,
respectively.  It reduces the largest family fraction by $29.4\%$ and the
number of resampling events by $55.6\%$.  The mean full cumulative ESS also
increases from $0.579\pm0.079$ to $0.685\pm0.024$, with an improvement in all
five paired seeds.  The mean CLIP margin increases by $0.0069$, from $0.1294$
to $0.1363$.  Its minimum, mean, median, and maximum are
$0.1256/0.1363/0.1368/0.1477$, compared with
$0.1135/0.1294/0.1307/0.1456$ for $\rho_2=0$.

The margin improves in three of the five paired seeds, whereas the unique and
effective root counts improve in four.  These results therefore support a
positive average effect on both specificity and genealogy, together with a
consistent improvement in full ESS, rather than uniform per-seed dominance.

\subsubsection{Compute}
\label{sec:t2i-diagnostics}

All experiments use NVIDIA A100 40GB GPUs. At each active step, two
forward-mode JVPs are computed, one for each conditional branch. Related and
unrelated conditions are run concurrently on separate four-GPU groups, while
their particle weights, resampling operations, and genealogies remain
independent.

\subsection{Additional Visualizations}
\label{sec:t2i-additional-visualizations}

\begin{figure}[H]
    \centering
    \includegraphics[width=0.7\linewidth]
    {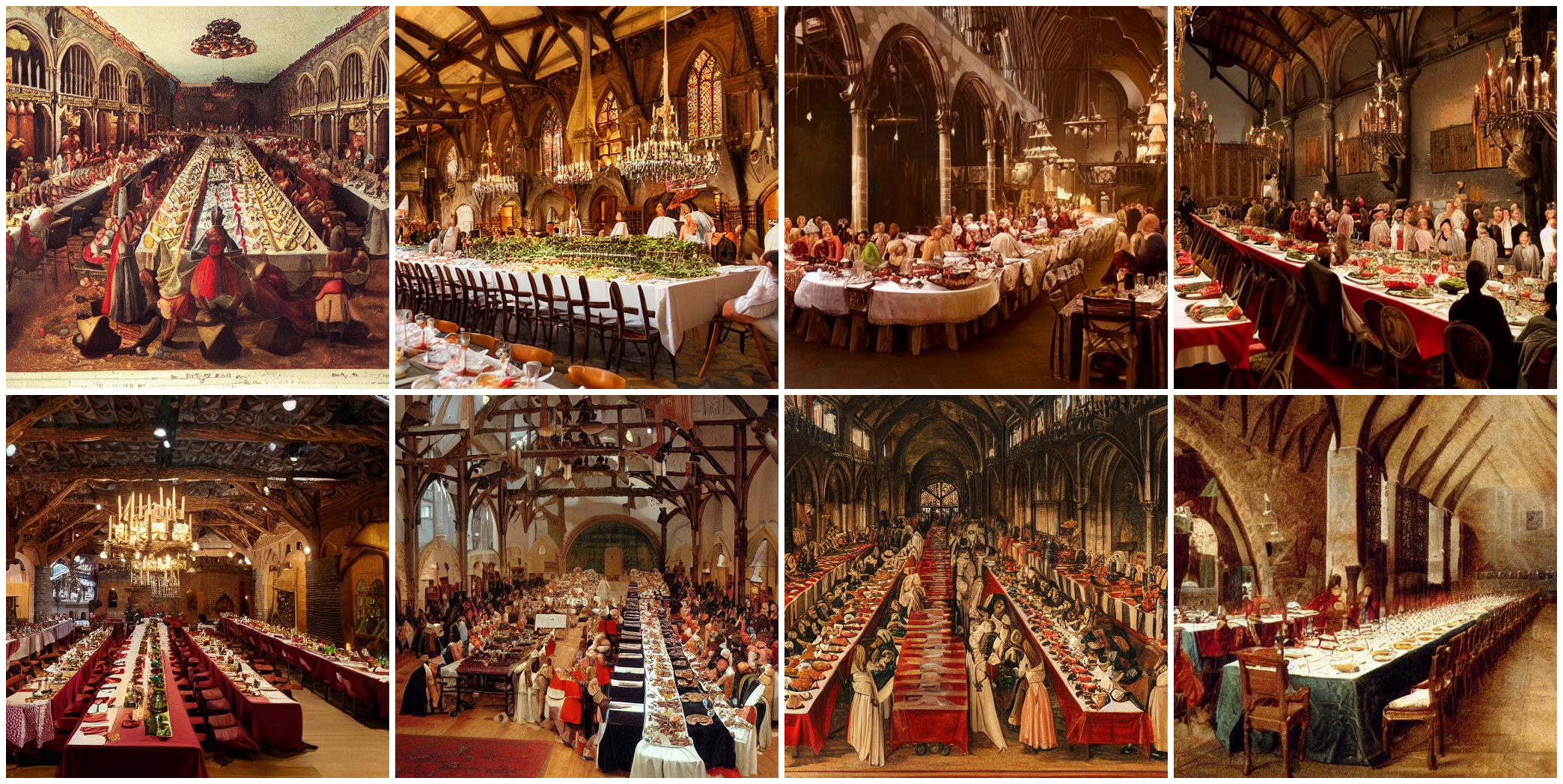}
    \caption{Qualitative samples from our method for the medieval-feast prompt
    with the related negative prompt.}
    \label{fig:t2i-medieval-related}
\end{figure}

\begin{figure}[H]
    \centering
    \includegraphics[width=0.7\linewidth]
    {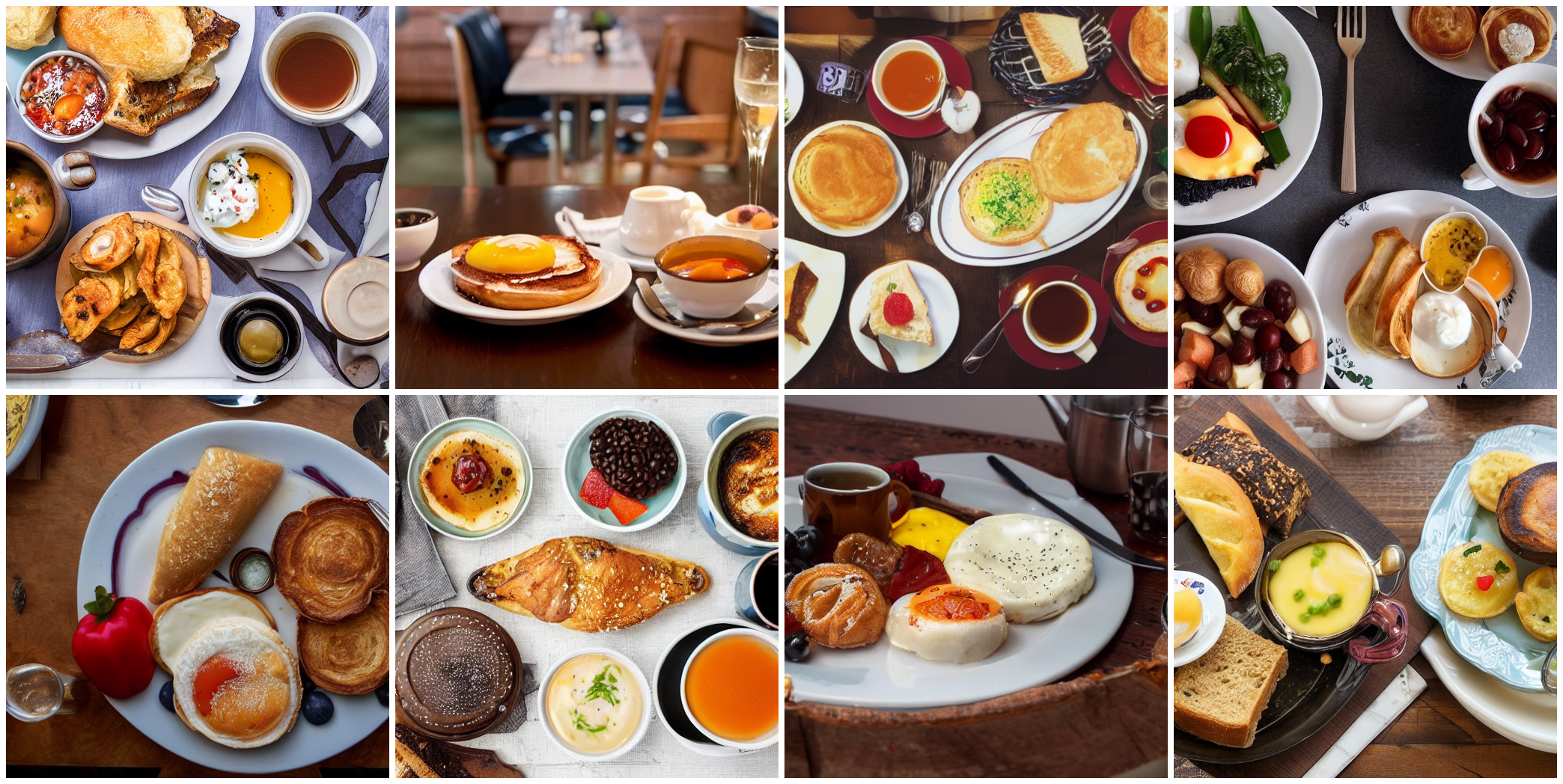}
    \caption{Qualitative samples from our method for the English-breakfast
    prompt with the related negative prompt.}
    \label{fig:t2i-breakfast-related}
\end{figure}

\begin{figure}[H]
    \centering
    \includegraphics[width=0.7\linewidth]
    {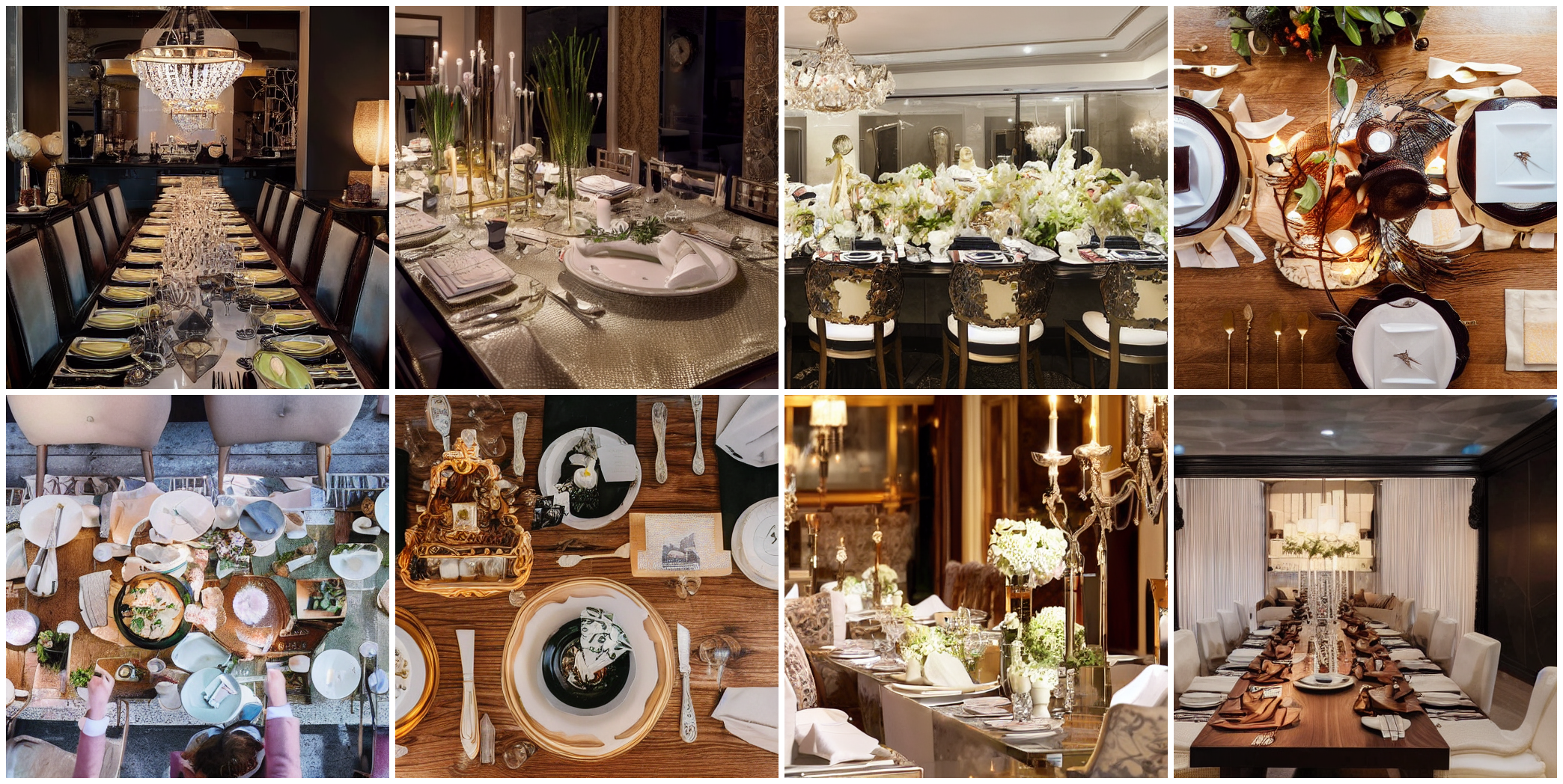}
    \caption{Qualitative samples from our method for the dinner-table prompt
    with the related negative prompt.}
    \label{fig:t2i-dinner-related}
\end{figure}

\begin{figure}[H]
    \centering
    \includegraphics[width=0.7\linewidth]
    {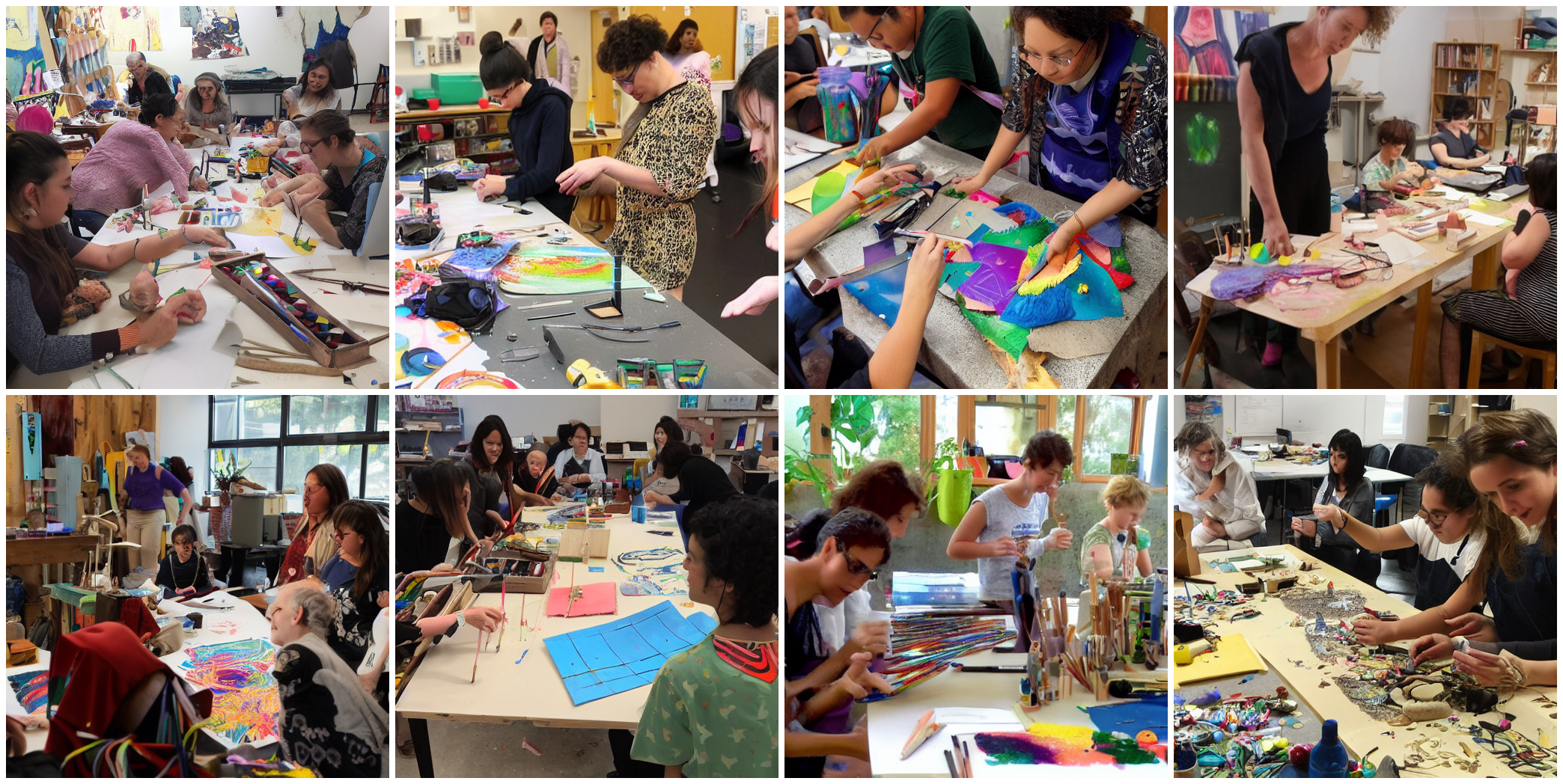}
    \caption{Qualitative samples from our method for the art-workshop prompt
    with the related negative prompt.}
    \label{fig:t2i-art-related}
\end{figure}

\begin{figure}[H]
    \centering
    \includegraphics[width=0.7\linewidth]
    {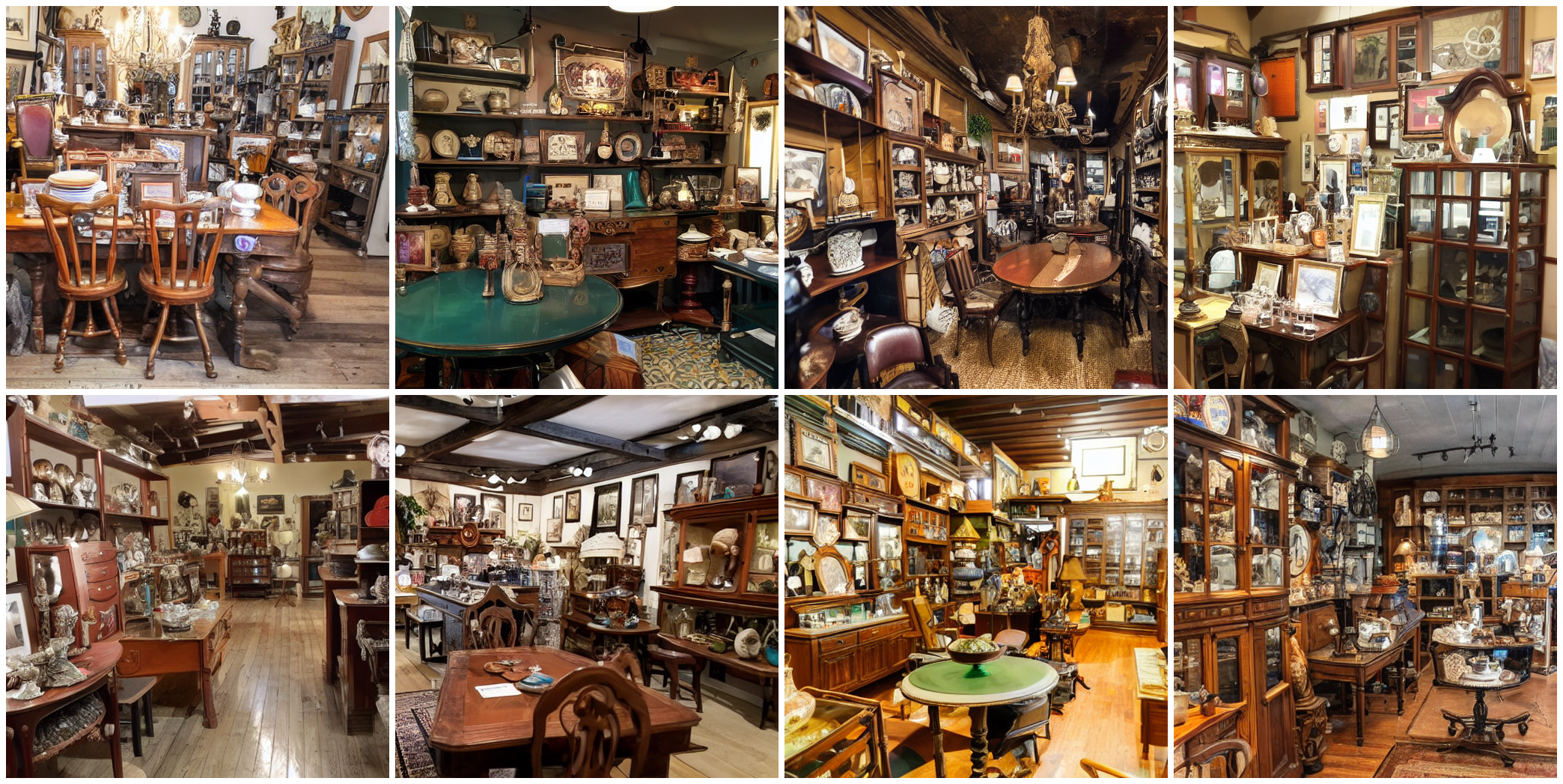}
    \caption{Qualitative samples from our method for the antique-store prompt
    with the related negative prompt.}
    \label{fig:t2i-antique-related}
\end{figure}

\section{Specificity-Aware pMHC Binder Design}
\label{app:binder-details}

\subsection{Task Setting}

We study specificity-aware protein-binder design for two closely related
peptide--MHC class I complexes. The desired target is
HLA-A*02:01--\texttt{NLVPMVATV}, whereas the undesired target is
HLA-A*02:01--\texttt{NLVFMVATV}. The two peptides differ only by a
Pro-to-Phe substitution at peptide position~4. Their normalized T-cell
activation values in the BATCAVE mutational-scan dataset are $1.000$ and
$0.031$, respectively. These measurements are used only to select a
biologically relevant specificity pair and are not treated as affinity
labels for the generated binders.

Both targets contain the same HLA-A*02:01 heavy chain and
$\beta_2$-microglobulin. The generated object is an $80$-residue protein
binder, denoted by chain $X$. The conditioning target contains the HLA heavy
chain, $\beta_2$-microglobulin, and the nine-residue peptide; no T-cell
receptor is included. Peptide residues 3--6, which include the mutated
position, define the intended binding region, while
$\beta_2$-microglobulin is marked as non-binding.

\subsection{Construction of the Paired Targets}

The desired and undesired pMHC structures are initially predicted with
Boltz-2. The undesired complex is aligned to the desired complex using a
Kabsch fit on the HLA platform, defined by the C$\alpha$ atoms of HLA
residues 1--180. We then construct an exact shared molecular context by
copying the desired HLA and $\beta_2$-microglobulin coordinates into both
conditions, while retaining the corresponding desired or mutated peptide.

Consequently, the two inputs share the same presentation scaffold and differ
only in the peptide sequence and peptide conformation. During generation, a
Kabsch fit on the same HLA platform removes the small numerical frame drift
between the two conditional denoiser evaluations. Across the reported
Proposition~2 generation, the median fitted rotation over the final
low-noise quarter of sampling is $0.029^\circ$, and the maximum across
backbone particles is $0.179^\circ$. Thus, the contrastive field is not
dominated by an arbitrary rigid-body mismatch.

\subsection{Binder-Only Contrastive Guidance}

At each reverse-diffusion step, the same noisy molecular state is evaluated
under the desired condition $A$ and the undesired condition $B$. Let $v_A$
and $v_B$ denote the corresponding denoising velocities. After rotating
$v_B$ into the frame of condition $A$, the target binder update is
\begin{equation}
v_{\mathrm{guided},X}
=
v_{A,X}
+
s_t
\left(
v_{A,X}
-
v_{B,X}R_{B\rightarrow A}
\right),
\label{eq:binder_guidance}
\end{equation}
where $R_{B\rightarrow A}$ is the HLA-derived frame correction and $s_t$ is
the likelihood-ratio-dependent guidance coefficient. Only the generated
binder is modified. The HLA, $\beta_2$-microglobulin, and peptide atoms retain
the desired-condition update.

The two conditional evaluations use the same noisy coordinates, feature
seed, noise realization, sampling step, and rigid-body augmentation.
Therefore, the difference in \eqref{eq:binder_guidance} isolates the
conditional contrast between the desired and undesired targets. No
unconditional model is evaluated; this is two-condition contrastive guidance
rather than conventional classifier-free guidance.

Let $d_t$ denote the one-unit binder-only contrastive field and let
$g_t=s_t d_t$ denote the target guided field. The proposal family is
\begin{equation}
u_t(\rho_2)
=
\rho_1 d_t + \rho_2 g_t,
\qquad
q_t(\rho_2)
=
u_t(\rho_2)-g_t.
\label{eq:binder_two_field_proposal}
\end{equation}
We fix $\rho_1=-1$ and choose $\rho_2$ adaptively from $25$ uniformly spaced
values in $[0,1.2]$. At each active step, the selected value maximizes the
full cumulative ESS after adding the candidate Proposition~2 increment:
\begin{equation}
\rho_{2,t}
=
\arg\max_{\rho_2\in\mathcal{R}}
\operatorname{ESS}
\left(
\log W_{t-1}
+
\widehat{\Delta\log w}^{\,\mathrm{P2}}_t(\rho_2)
\right).
\label{eq:binder_rho2_selection}
\end{equation}
The same retained Proposition~2 expansion is used for both selecting
$\rho_2$ and propagating the particle weights. It includes the forward- and
reverse-$\alpha$ terms, the target- and proposal-curvature terms, and the
fixed-displacement finite-difference JVP term. It is not replaced by an exact
finite Gaussian transition ratio.

\subsection{Compared Methods}

We compare Ours with the following methods:
\begin{enumerate}
    \item \textbf{BoltzGen (target A only):} the original BoltzGen sampler
    conditioned only on the desired pMHC, with the undesired target absent
    during generation and no negative guidance;
    \item \textbf{Fixed guidance:} two-condition contrastive guidance with
    coefficient $\lambda=0.25$;
    \item \textbf{DNG:} Dynamic Negative Guidance with base gain $18$; and
    \item \textbf{Ours:} likelihood-ratio gating with the adaptive two-field
    proposal and retained Proposition~2 weights.
\end{enumerate}

For Ours, let $\widehat{\ell}_t$ denote the accumulated local
reverse-kernel log-likelihood contrast. We use
\begin{equation}
\theta_t
=
\sigma\left(
\log c+\eta\widehat{\ell}_t
\right),
\qquad
s_t
=
m\eta\tau\theta_t .
\label{eq:binder_pa_gate}
\end{equation}
The reported configuration is
\begin{equation}
c=6,
\qquad
\eta=0.005,
\qquad
\tau=0.75,
\qquad
m=2022.22,
\qquad
\rho_1=-1.
\end{equation}
Since $\widehat{\ell}_0=0$, the initial gate and guidance coefficient are
\begin{equation}
\theta_0
=
\frac{c}{1+c}
=
\frac{6}{7},
\qquad
s_0
=
m\eta\tau\frac{c}{1+c}
=
6.5.
\end{equation}
We clip $\widehat{\ell}_t$ to $[-100,100]$. The proposal variance scale is
$0.10$. The finite-difference JVP uses step size $0.01$ and shrinkage
coefficient $0.75$.

For numerical robustness, the gate, JVP, and remaining kernel contributions
to the Proposition~2 log-weight are median-centered and winsorized separately,
with maximum centered magnitudes $0.10$, $0.25$, and $0.10$, respectively.
No additional clip is applied to their summed increment. Resampling is
permitted when the full cumulative ESS fraction falls below $0.30$, subject
to a ten-step cooldown, at most ten resampling events, and no resampling in
the final ten steps. If resampling occurs, ancestor-selection log-weights are
median-centered and capped at $1.5$, and the residual weight correction is
carried forward. These operations robustify the Proposition~2 estimator
without changing the underlying target expansion. We do not apply a separate
cap to the contrastive displacement.

\subsection{Generation and Structural Verification}

We use eight matched seed blocks. Each method generates eight binder
backbones per block with $200$ reverse-diffusion steps and retains one
inverse-folded sequence per backbone. The reported evaluation therefore
contains
\begin{equation}
8\ \text{seed blocks}
\times
8\ \text{backbones per block}
=
64\ \text{binders per method}.
\end{equation}
For Ours, the eight backbones within each block form an interacting particle
batch. The other methods use the same seed blocks and output budget.

Although the sampler permits SMC resampling, no resampling event occurs in
any of the eight particle batches for Ours. All eight initial roots survive
within every block, and the ancestor-weight clip is never activated.

Each binder sequence is independently modeled with both the desired and
undesired pMHC using Boltz-2. We generate three predictions per target
condition using $200$ sampling steps and three recycling steps:
\begin{equation}
64\ \text{binders}
\times
2\ \text{conditions}
\times
3\ \text{predictions}
=
384\ \text{verifier structures per method}.
\end{equation}

All four methods use the same eight seed blocks, binder count,
inverse-folding budget, target pair, verifier settings, and specificity
metric. Thus, differences in Table~\ref{tab:pmhc_binder_specificity} are not
caused by unequal sequence or verification budgets.

\subsection{Specificity Metric}

Boltz-2 is used as a structural verifier rather than as a calibrated
protein--protein affinity predictor. Let $M^{(i,c,k)}$ denote the Boltz-2
\texttt{pair\_chains\_iptm} matrix for binder $i$, condition $c$, and verifier
prediction $k$. For each prediction, we symmetrize the binder--peptide
interface score:
\begin{equation}
a_{i,c,k}
=
\frac{1}{2}
\left(
M^{(i,c,k)}_{X,P}
+
M^{(i,c,k)}_{P,X}
\right).
\end{equation}
The condition-level score is the mean over the $K$ verifier predictions:
\begin{equation}
A_{i,c}
=
\frac{1}{K}
\sum_{k=1}^{K}
a_{i,c,k}.
\end{equation}
The specificity margin of binder $i$ is
\begin{equation}
\Delta_i
=
A_{i,\mathrm{desired}}
-
A_{i,\mathrm{undesired}}.
\label{eq:binder_specificity_margin}
\end{equation}
A positive margin indicates greater predicted binder--peptide interface
confidence for the desired pMHC. We report the mean, median, and $75$th
percentile (Q75) of the margin, together with the fraction of binders for
which $\Delta_i>0$.
No whole-complex ipTM, contact term, HLA-interface term, or manually weighted
composite score enters this metric.

\subsection{Results}

\begin{table}[t]
\centering
\small
\caption{
Specificity on the HLA-A*02:01--\texttt{NLVPMVATV} versus
HLA-A*02:01--\texttt{NLVFMVATV} binder-design task.
Here, $N$ is the number of evaluated binders and $K$ is the number of
Boltz-2 predictions per binder and target condition. All methods use $K=3$;
higher values are better.
}
\label{tab:pmhc_binder_specificity}
\begin{tabular}{lccccc}
\toprule
\textbf{Method}
& $\boldsymbol{N}$
& \textbf{Mean}
& \textbf{Median}
& \textbf{Q75}
& $\boldsymbol{\Delta>0}$ \\
\midrule
BoltzGen (target A only)
& 64
& 0.1786
& 0.1587
& 0.2862
& 87.5\% \\
Fixed guidance
& 64
& 0.1654
& 0.1453
& 0.2833
& \textbf{93.8\%} \\
DNG
& 64
& 0.1483
& 0.1581
& 0.2623
& 89.1\% \\
\textbf{Ours (Proposition~2)}
& 64
& \textbf{0.1856}
& \textbf{0.1822}
& \textbf{0.3337}
& 84.4\% \\
\bottomrule
\end{tabular}
\end{table}

Ours achieves the largest mean, median, and Q75 specificity margins. Relative
to target-A-only BoltzGen, the strongest baseline on these statistics, Ours
improves the mean by $3.9\%$, the median by $14.8\%$, and Q75 by $16.6\%$.
Fixed guidance produces the largest fraction of positive margins. Thus, Ours
improves the center and upper quartile of the specificity distribution,
although it does not maximize the fraction of positive designs.

\paragraph{Representative binder.}

The binder shown in Figure~\ref{fig:pmhc-design} has the sequence
\begin{center}
\ttfamily\small
VTLTVEFTYD PSKVGAVIHS PLVLTEDGYK VKVTLTISED\\
GETITVKVRL YDSDNRLISD YKAEVTVEGD TITVTVTFDP
\end{center}
where spaces are inserted only for readability. Across the three verifier
predictions per condition, its mean desired and undesired binder--peptide
iPTM values are $0.8673$ and $0.5174$, respectively, giving a specificity
margin of $0.3499$. The binder contacts the desired peptide in two of the
three predictions and the undesired peptide in none. All six verifier
structures are complete, and no binder--target heavy-atom clashes below
$1.5$~\AA{} are detected.

\paragraph{Sensitivity.}

The reported configuration was selected in preliminary hyperparameter
screens. Table~\ref{tab:pmhc_binder_specificity} reports the complete matched
$64$-binder evaluation rather than a smaller tuning subset. Its mean, median,
and Q75 margins are $0.1856$, $0.1822$, and $0.3337$, respectively.

\paragraph{Interpretation.}

These results show that likelihood-ratio gating with an adaptive two-field
proposal and Proposition~2 weight propagation can shift the center of the
generated-binder distribution toward stronger predicted recognition of the
desired peptide relative to a single-residue off-target. The reported ipTM
margin is a structural interface-confidence surrogate and should not be
interpreted as a calibrated binding affinity. Experimental binding and
off-target measurements remain necessary to validate the generated
specificity.

\newpage

\end{document}